%% file: main.tex
\documentclass{article}
\pdfoutput=1
\usepackage{arxiv}
\setcitestyle{authoryear,open={(},close={)}}
\renewcommand{\cite}{\citep}

\usepackage[utf8]{inputenc} %
\usepackage[T1]{fontenc}    %
\usepackage{hyperref}
\usepackage{url}
\usepackage{booktabs}       %
\usepackage{amsfonts}       %
\usepackage{nicefrac}       %
\usepackage{microtype}      %
\usepackage{dsfont}
\usepackage{graphicx}
\usepackage{amsmath}
\usepackage{mathtools}
\usepackage{amssymb}
\usepackage{amsthm}
\usepackage{cases}
\usepackage{color}
\usepackage{textcomp}
\usepackage{xcolor}
\usepackage{caption}
\usepackage{bbm}
\usepackage{comment}
\usepackage{xspace}
\usepackage{multirow}
\usepackage{subcaption}
\usepackage{authblk}
\usepackage{xspace}
\usepackage{algpseudocode}
\usepackage{adjustbox}
\usepackage{wrapfig}

\usepackage{microtype}
\usepackage{graphicx}
\usepackage{booktabs} %

\usepackage{hyperref}

\usepackage{amsmath}
\usepackage{amssymb}
\usepackage{mathtools}
\usepackage{amsthm}
\usepackage{mleftright}

\usepackage[capitalize,noabbrev]{cleveref}

\theoremstyle{plain}
\newtheorem{theorem}{Theorem}[section]

\newtheorem{definition}[theorem]{Definition}
\newtheorem{assumption}[theorem]{Assumption}
\theoremstyle{remark}

\input{math_commands.tex}
\newcommand{\LC}{\texttt{LC}\xspace}

\hypersetup{
    colorlinks,
    citecolor=[rgb]{0.00, 0.45, 0.70}, %
    linkcolor=[rgb]{0.01, 0.62, 0.45}, %
    urlcolor=[rgb]{0.68, 0.45, 0.63} %
}
\definecolor{green1}{rgb}{0.01, 0.62, 0.45}
\definecolor{blue1}{rgb}{0.00, 0.45, 0.70}
\definecolor{blue2}{rgb}{0.612, 0.8, 0.902}
\definecolor{purple1}{rgb}{0.68, 0.45, 0.63}
\definecolor{red1}{rgb}{1.0, 0.1, 0.3}
\definecolor{red2}{rgb}{1.0, 0.77, 0.80}

\usepackage{todonotes}
\usepackage{enumitem}
\newcommand{\E}{\mathbb{E}}

\newcommand{\RL}{\text{RL}}
\newcommand{\SFT}{\text{SFT}}

\newcommand{\Base}{\text{Base}}

\newcommand{\LLM}{\text{LLM}}
\newcommand{\Sc}{R}

\newcommand{\Lall}{\mathcal{L}_{\text{all}}}
\newcommand{\Lr}{\mathcal{L}_{\text{r}}}

\usepackage{amsmath}
\usepackage{amsthm}
\usepackage{amssymb}
\usepackage[bbgreekl]{mathbbol}

\DeclareMathSymbol\drule  \mathord{bbold}{"01}
\newcommand{\druleall}{{\drule_{\text{all}}}}
\newcommand{\drulel}{{\drule_{\mathcal{L}}}}

\newcommand{\DeltaAns}{\Delta^{|\mathcal{Y}|}}
\newcommand{\f}{f}
\newcommand{\reader}{reader}
\newcommand{\forecast}{\f(x, z)}
\newcommand{\surrforecast}{\widetilde{\f}(x, z)}

\usepackage{bbm}
\usepackage[ruled]{algorithm2e}

\usepackage{wrapfig}
\usepackage{graphicx}
\usepackage{caption}
\usepackage{titletoc}
\usepackage{placeins}
\usepackage{graphicx}
\usepackage{placeins}
\usepackage[most]{tcolorbox} %
\usepackage{pdflscape}

\usepackage{tabularx}
\usepackage{siunitx}
\usepackage{soul}

\DeclareRobustCommand{\hlblue}[1]{{\sethlcolor{blue2}\hl{#1}}}
\DeclareRobustCommand{\hlred}[1]{{\sethlcolor{red2}\hl{#1}}}

\newcommand{\scr}[1]{\ensuremath{{\scriptstyle #1}}}

\renewrobustcmd{\bfseries}{\fontseries{b}\selectfont}
\renewrobustcmd{\boldmath}{}
\newrobustcmd{\B}{\bfseries}

\newtcolorbox{qualitativeBox}{
  colback=gray!10, %
  colframe=gray!20, %
  rounded corners, %
  boxrule=0.5pt, %
  left=10pt, %
  right=10pt, %
  top=5pt, %
  bottom=5pt, %
  boxsep=5pt, %
}
\newtcolorbox{ourMethodBox}{
  colback=blue1!10, %
  colframe=gray!20, %
  rounded corners, %
  boxrule=0.5pt, %
  left=10pt, %
  right=10pt, %
  top=5pt, %
  bottom=5pt, %
  boxsep=5pt, %
}
\newtcolorbox{theirMethodBox}{
  colback=red1!10, %
  colframe=gray!20, %
  rounded corners, %
  boxrule=0.5pt, %
  left=10pt, %
  right=10pt, %
  top=5pt, %
  bottom=5pt, %
  boxsep=5pt, %
}

\begin{document}

\title{Linguistic Calibration of Long-Form Generations}
\author{
    \large{Neil Band, ~Xuechen Li, ~Tengyu Ma, ~Tatsunori Hashimoto}\linebreak
    \large{Stanford University}\linebreak
    {\texttt{\{nband, lxuechen, tengyuma, thashim\}@cs.stanford.edu}}
}

\date{}

\maketitle

\input{sections/1_abstract.tex}
\input{sections/2_introduction.tex}

\input{sections/3_preliminaries.tex}

\input{sections/4_method.tex}
\input{sections/5_experiments.tex}
\input{sections/6_related_work.tex}
\input{sections/7_discussion.tex}
\newpage

\bibliography{bibliography}
\bibliographystyle{plainnat}

\newpage
\appendix
\onecolumn

\input{sections/8_appendix.tex}

\end{document}

%% file: math_commands.tex
\usepackage{amsfonts,bm}

\def\eqref#1{equation~\ref{#1}}

\def\1{\bm{1}}

\DeclareMathAlphabet{\mathsfit}{\encodingdefault}{\sfdefault}{m}{sl}
\SetMathAlphabet{\mathsfit}{bold}{\encodingdefault}{\sfdefault}{bx}{n}

\newcommand{\R}{\mathbb{R}}

\DeclarePairedDelimiterX{\KLdivx}[2]{\big(}{\big)}{%
  #1\;\delimsize\|\;#2%
}

\DeclareMathOperator*{\argmax}{arg\,max}
\DeclareMathOperator*{\argmin}{arg\,min}

\usepackage{amsthm}

\newtheorem{lemma}{Lemma}
\newtheorem{property-non}{Property}

\newcommand*{\bracks}[1]{\left(#1\right)}  %
\newcommand*{\sbracks}[1]{\left[#1\right]}  %

%% file: sections/1_abstract.tex
\begin{abstract}
Language models (LMs) may lead their users to make suboptimal downstream decisions when they confidently hallucinate.
This issue can be mitigated by having the LM verbally convey the probability that its claims are correct, but existing models cannot produce long-form text with calibrated confidence statements.
Through the lens of decision-making, we define linguistic calibration for long-form generations: an LM is linguistically calibrated if its generations enable its users to make calibrated probabilistic predictions.
This definition enables a training framework where a supervised finetuning step bootstraps an LM to emit long-form generations with confidence statements such as ``I estimate a 30\% chance of...'' or ``I am certain that...'', followed by a reinforcement learning step which rewards generations that enable a user to provide calibrated answers to related questions.
We linguistically calibrate Llama 2 7B and find in automated and human evaluations of long-form generations that it is significantly more calibrated than strong finetuned factuality baselines with comparable accuracy.
These findings generalize under significant domain shifts to scientific and biomedical questions and to an entirely held-out person biography generation task.
Our results demonstrate that long-form generations may be calibrated end-to-end by constructing an objective in the space of the predictions that users make in downstream decision-making.
\end{abstract}

%% file: sections/2_introduction.tex
\section{Introduction}
\label{sec:intro}

\begin{figure}[t]
    \centering    
    \includegraphics[width=\textwidth]{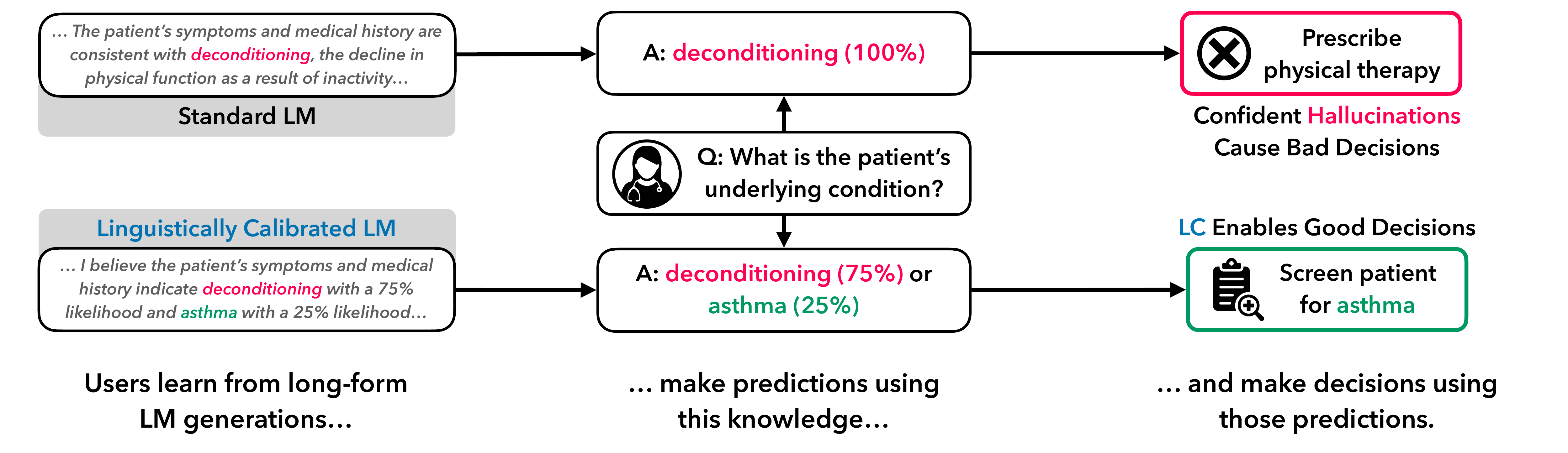}
    \caption{
    \textbf{Illustrative example of linguistic calibration.}
    We define linguistic calibration of long-form generations (\LC) as calibrating an LM's generations in a way that leads to calibrated downstream user forecasts.
    We apply \LC to train an LM that emits calibrated statements of confidence \textit{in natural language}, enabling better downstream decisions.
    \textbf{Left:} users read long-form generations (\textit{e.g.,} a doctor reading an LM-generated clinical note).
    \textbf{Middle:} to decide the patient's treatment, the doctor first forecasts the patient's underlying condition.
    \textbf{Upper Right:} when standard LMs lack knowledge, they hallucinate confidently, leading to a suboptimal decision (treating the wrong condition).
    \textbf{Lower Right:} 
    even if the base LM cannot be confidently correct, linguistic calibration encourages the LM to spread probability over plausible claims, enabling a better decision.
    }
    \label{fig:figure_1}
\end{figure}

The claims made by language models (LMs) are increasingly used to inform real-world decisions, \textit{e.g.}, what to order at a restaurant, what information to provide someone else about a topic, or which code completion to accept.
However, LMs have knowledge gaps which manifest as hallucinations \cite{ziwei2023survey, huang2023survey}.
Currently, when an LM lacks knowledge about a topic, it will do one of two things: hallucinate incorrect claims with complete confidence, or, in the case of a few strong closed-source models \cite{openai2023gpt4, anthropic2023claude}, abstain from making claims.%

Confident hallucinations are especially harmful. 
They decrease users' trust in the errant LM and broadly make LMs unsuitable for settings where factuality is paramount such as medicine \cite{Thirunavukarasu2023} and law \cite{dahl2024large}.
Perhaps most importantly, they lead the user to \textit{confidently make poor decisions} (Fig.~\ref{fig:figure_1}).
However, even abstentions are suboptimal, because they provide the user with no plausible claims and their likelihoods.

Linguistic calibration \cite{mielke-etal-2022-reducing}---conveying confidence levels in natural language that equal the likelihood that one's claims are correct---could mitigate the harms of hallucination.
If an LM was linguistically calibrated in a manner interpretable to its users, they could make good decisions regardless of the LM's underlying knowledge.

For example, suppose a clinical LM generates a patient's case report, providing a diagnosis.
If the LM was unsure of the correct diagnosis, it could use \textit{numerical confidence} with its corresponding claim (``I estimate that the correct diagnosis is Condition A with 60\% confidence'').
Then, when the doctor faces a decision---deciding the patient's treatment---they have access to both a salient claim and an approximate likelihood of its correctness.
The manner of conveying confidence is limited only by the use of language: \textit{e.g.,} the LM could provide \textit{linguistic} confidence statements (``I am fairly sure that the correct diagnosis is A'') or many mutually exclusive claims (``I estimate that the correct diagnosis is A with 60\% confidence, though B or C are also possibilities.'').
However, both classic calibration methods like temperature scaling \cite{pmlr-v70-guo17a} and methods for LMs \cite{kadavath2022language, tian-etal-2023-just, kuhn2023semantic, mielke-etal-2022-reducing, lin2022teaching, jiang2021how} are restricted to classification or short outputs and hence cannot calibrate the many claims made in each long-form LM generation (cf. \S\ref{sec:related_work} for related work).

We make progress on this challenge by leveraging the connection between \textit{calibration and decision theory} 
\cite{zhao2021calibrating, zhao2021right}.
LMs fit cleanly within the framework of decision-making. 
Users query LMs, learn from their generations, and later encounter decision-making tasks with associated questions (``What is the patient's condition?'').
They forecast an answer based on what they learned, and finally make a decision (which treatment to provide) using their forecast for which they receive some reward.

Because linguistic calibration improves downstream user decisions, we might hope to calibrate an LM by directly optimizing real-world downstream rewards.
While intuitively appealing, this process is challenging; it entails tracing the knowledge learnt by users through to their decisions, and further propagating the associated rewards back to the LM.

\vspace{-0.7em}
\paragraph{Our contributions.} We propose a definition of linguistic calibration for long-form generations (\LC) which sidesteps the difficulty of tracing real-world rewards and enables training an LM to emit calibrated confidence statements in long-form generations.
Our contributions are as follows:
\begin{itemize}[topsep=0pt,leftmargin=4mm]
    \setlength\itemsep{0.08mm}
    \item We define an LM emitting long-form generations to be \textbf{linguistically calibrated} if it enables users to produce \textit{calibrated forecasts} relevant to their decision-making tasks, which in turn enable optimal decision-making.
    \item We instantiate this definition in a \textbf{training objective and framework} that calibrates long-form generations through a decision-theoretic lens.
    Our training framework first bootstraps an LM to express confidence statements with supervised finetuning. Then, it optimizes our objective using reinforcement learning (RL), rewarding the LM policy for generating text that enables calibrated forecasts on related questions.
    \item We linguistically calibrate Llama 2 7B using our training framework and find it significantly improves calibration versus strong baselines finetuned for factuality while matching their accuracy, in human and API-based LLM evaluation of long-form generations. 
    We also show that linguistic calibration has zero-shot transfer beyond the training task.
    Specifically, an LM calibrated using a single off-the-shelf question-answering dataset is also calibrated on significantly out-of-domain scientific question-answering datasets, and on an entirely held-out task of biography generation.
\end{itemize}
Instead of working in the space of text, our decision-based approach constructs an objective in the space of the predictions that users make in the process of decision-making.
This makes the standard calibration machinery of proper scoring rules \cite{gneiting2007strictly} tractable as an objective for end-to-end calibration of long-form generations.

%% file: sections/3_preliminaries.tex
\section{Setup}
\label{sec:prelim}

Our goal is to formulate a tractable objective that enables the end-to-end linguistic calibration of long-form LM generations. 
To begin, we define what it means for an LM emitting long-form generations to be linguistically calibrated.

\subsection{Linguistic Calibration of Long-Form Generations}
\label{subsec:linguistic_calibration}

Our definition is motivated by the process through which users learn from LM generations and later use their knowledge to answer questions.
This process of \textbf{LM-assisted user forecasting} allows us to turn the problem of calibrating long-form text into the problem of calibrating probabilistic forecasts.

\paragraph{LM-assisted user forecasting.}
First, the user prompts an LM $\pi$ with an open-ended query $q$ (\textit{e.g.,} ``Generate a clinical note for this patient...'').
Next, the LM $\pi$ generates a long-form context $z \sim \pi(z \mid q)$, \textit{e.g.,} the clinical note.

At some point in the future, the user encounters a question $x$ associated with the query $q$, such as $x = $ ``What is the correct diagnosis of the patient?''
Let $y \in \mathcal{Y}$ be the answer to the question $x$, and $p(y \mid x)$ be the ground-truth conditional distribution over answers.
We note that given $x$, the answer $y$ does not depend on either the query $q$ or the long-form generation $z$, \textit{i.e.,} we have the conditional independence statement $Y \perp (Q, Z) \mid X$.
Altogether, we have joint distribution $p(q, x, y, z) := p(q, x, y) \pi(z \mid q)$.

Since $p(y \mid x)$ is unknown to the user, they form a probabilistic forecast over possible answers $y \in \mathcal{Y}$ as a probability distribution in $\DeltaAns$.
In our example, the doctor can make a better forecast over possible diagnoses by first reading the LM-generated clinical note.
In other words, because the long-form generation $z$ is salient to the question $x$, the user conditions on it to make their forecast; the user applies their \textbf{\reader} function $f: \mathcal{X} \times \mathcal{Z} \rightarrow \DeltaAns$ to the question $x$ and long-form generation $z$ in order to form a \textbf{forecast} $\forecast \in \DeltaAns$.
For example, we might have $\forecast = [0.2, 0.8]^\top$ placing $20\%$ mass on Condition A and $80\%$ on Condition B.
We will denote the probability of answer $y$ under the forecast as $\forecast_y$.

\paragraph{Defining linguistic calibration for long-form generations.}
LM-assisted user forecasting motivates a definition of linguistic calibration in terms of the calibration of a downstream \reader.
Intuitively, an LM is \textbf{linguistically calibrated} with respect to a downstream user if its long-form generations $z$ enable that user to make calibrated forecasts on question-answer pairs $(x, y)$ using their \reader~function $\f$. 
This notion of linguistic calibration is a property of an LM, and is defined with respect to a particular user and notion of classifier calibration.

\begin{definition}[Linguistic Calibration of Long-Form Generations]
An LM $\pi$ emitting long-form generations $z \sim \pi(z \mid q)$ is \textbf{linguistically $\phi$-calibrated} (for $\phi \in \{\textup{distribution, classwise, confidence}\}$) with respect to a \reader $\f: \mathcal{X} \times \mathcal{Z} \rightarrow \DeltaAns$ if and only if $f$ is $\phi$-calibrated on the joint distribution $p(x,y,z)$.
\label{defn:ling_cal}
\end{definition}

Prior work established linguistic calibration for single-claim utterances \cite{mielke-etal-2022-reducing}.
Our definition generalizes linguistic calibration to long-form generations, which we will show has both theoretical and empirical benefits.
Our focus is on calibrating long-form text, so future references to \textbf{linguistic calibration}, \textbf{linguistically calibrated}, \textit{etc.} refer to Defn.~\ref{defn:ling_cal}.
\S\ref{sec:related_work} discusses related work.

Defn. \ref{defn:ling_cal} reduces the problem of calibrating long-form text to calibrating over a smaller space of answers $\mathcal{Y}$.
Each notion of linguistic $\phi$-calibration is equivalent to a traditional notion of classifier calibration $\phi$ applied to the reader $f$.
To be self-contained, we next provide examples in our notation.
Readers familiar with classifier calibration could skip to \S\ref{subsec:user_decision_problem}.

\paragraph{Examples of linguistic $\phi$-calibration.} Consider \textbf{distribution calibration} as a specific notion of classifier calibration \cite{brocker2009reliability, kull2015novel}.
Intuitively, a classifier $g: \mathcal{X}' \rightarrow \DeltaAns$ is distribution calibrated if its forecast $g(x') \in \DeltaAns$ is faithful to the predictive uncertainty \cite{minderer2021revisiting}; among all inputs $x' \in \mathcal{X}'$ for which the forecasted probability of class $y$ is 0.4, \textit{i.e.,} $g(x')_y = 0.4$, 40\% of those inputs should have label $y$.
Taking $\1_y \in \{0, 1\}^{|\mathcal{Y}|}$ as the one-hot vector for $y$, $g$ is distribution calibrated if
\begin{align*}
    \mathbb{E}\left[\1_{y} \mid g(x') = p~\right] = p,\quad~\forall p \in \DeltaAns.
\end{align*}
Now taking our reader $\f := g$ as the classifier and both the question and LM generation as the input $(x, z) := x'$, we say that an LM $\pi$ is \textbf{linguistically distribution calibrated} with respect to a reader $\f$ if
\begin{align*}
    \mathbb{E}\left[\mathbf{1}_{y} \mid f(x, z) = p~\right] = p,\quad~\forall p \in \DeltaAns.
\end{align*}
We provide more examples of this correspondence in Table~\ref{tab:linguistic_calibration_examples}.
These include the weaker notion of \textbf{confidence calibration} \cite{pmlr-v70-guo17a}, which means that among the examples with top forecasted probability equal to $\beta$, the accuracy over those examples is also $\beta$.

\input{tables/linguistic_calibration}

\subsection{From Calibration to Optimal Decisions}
\label{subsec:user_decision_problem}

The process of LM-assisted user forecasting models how users learn about a variety of topics by reading LM generations and forecast answers to related questions.
We now describe how they use their forecasts to inform downstream decisions.
This decision-based perspective leads to our objective for linguistic calibration.

\paragraph{LM-assisted user decision-making.} The user will make a decision by selecting an action $a \in \mathcal{A}$ (which treatment to provide the patient) according to their forecast $\forecast \in \DeltaAns$ and loss function $L: \mathcal{A} \times \mathcal{Y} \rightarrow \mathbb{R}$.
For example, the user may make a Bayes-optimal decision by choosing the action which minimizes their loss under the forecast distribution:
$a^* = \argmin_a \E_{\widehat{y} \sim \forecast} [L(a, \widehat{y})]$. Based on the realized answer $y \sim p(y \mid x)$, the user suffers a loss of $L(a^*, y)$.

Ideally, one would follow this LM-assisted user decision-making process and directly train an LM to emit generations $z$ which minimize the user's downstream loss $L(a^*, y)$.
However, it is difficult to obtain real-world rewards from decision-making, and moreover to obtain a real-world distribution over queries to LMs and related question-answer pairs $p(q, x, y)$.
We overcome both of these challenges by \textit{optimizing for linguistic calibration} over a representative surrogate distribution.
In other words, we encourage the LM to emit generations $z$ that enable the user to provide calibrated answers to questions $x$.

In the rest of this section, we will introduce our objective for linguistic calibration and demonstrate its optimality properties.
We discuss our surrogate distribution in \S\ref{subsec:decision_task_data}.

\paragraph{Linguistic calibration implies informed decision-making.}
Why is optimizing for linguistic calibration a reasonable proxy for directly optimizing downstream losses $L(a^*, y)$?
\citet{zhao2021calibrating} prove that calibrated forecasts $\forecast$ enable a user to make informed decisions.
Specifically, the user may use their calibrated forecast to estimate the expected loss of a decision-making strategy as if they had access to the ground-truth distribution $p(y \mid x)$.
This \textbf{accurate loss estimation} implies that Bayes-optimal decision-making outperforms other decision-making strategies.
Altogether, because linguistic calibration implies calibrated forecasts by definition (cf. Defn.~\ref{defn:ling_cal}), it will provide optimal decision-making guarantees.

\paragraph{Guarantees for weaker notions of calibration.} 
Each notion of linguistic $\phi$-calibration (cf. Table \ref{tab:linguistic_calibration_examples}) in fact corresponds to a different decision-making guarantee.
Even if our optimization process cannot achieve linguistic distribution calibration, achieving linguistic $\phi$-calibration with respect to a weaker notion $\phi$ such as classwise or confidence calibration will still provide a (weaker) decision-making guarantee.
Appendix \ref{app:decision_and_calibration} provides a proof of this correspondence (\Cref{prop:conf_calib_implies_lr_decision_calibration}).

Lastly, we note that this connection between linguistic calibration and optimal decision-making motivates our evaluation metric of \textbf{reader expected calibration error} (defined in \S\ref{subsec:exp_setup}).

\subsection{Training Objective for Linguistic Calibration}
\label{subsec:linguistic_calibration_training_objective}

Because linguistic calibration of long-form generations is equivalent to the calibration of downstream probabilistic forecasts, we can optimize for it with the standard calibration machinery of \textbf{scoring rules}.
We use a \textbf{strictly proper scoring rule} as our optimization objective, with the guarantee that it is maximized if and only if user forecasts $\forecast$ equal the ground-truth distribution $p(y \mid x)$.

\paragraph{Proper scoring rules.}
There is a rich literature on scoring rules \cite{gneiting2007strictly, savage1971elicitation}, which measure the quality of a forecast.
Specifically, a scoring rule $\Sc$ scores a forecast $\widehat{p} \in \DeltaAns$ and outcome $y\in \mathcal{Y}$, \textit{i.e.,} $\Sc: \DeltaAns \times \mathcal{Y} \rightarrow \R$.
\textbf{Proper} scoring rules have the desirable property that the true distribution of $y$, denoted by $p$, is a maximizer.
Formally, $\Sc$ is proper if
$$
    \E_{y \sim p}[\Sc(p, y)] \geq \E_{y \sim p}[\Sc(\widehat{p}, y)],\;~\forall p, \widehat{p} \in \DeltaAns.
    \label{eq:psr_body}
$$
$\Sc$ is \textbf{strictly proper} if the unique maximizer of the expected reward is the true probability $p$.
Strictly proper scoring rules are ubiquitous in ML, \textit{e.g.,} the (negative) log loss or Brier score \cite{brier1950verification}, and are natural objectives for calibration methods such as temperature \cite{pmlr-v70-guo17a} or Platt scaling \cite{platt1999probabilistic}.

\paragraph{Linguistic calibration objective.} We simply choose the negative log loss as our strictly proper scoring rule, that is, we use $\Sc_{\LC}(p, y) := \log p_{y}$, and therefore our objective function is
\begin{align}
    \Sc_{\LC}(\forecast, y) := \log \forecast_y.
    \label{eq:logloss}
\end{align}

By strict propriety, if our objective is maximized, then the user's forecast $\forecast$ equals the ground-truth $p(y \mid x)$ and we obtain linguistic distribution calibration.

Proper scoring rules do not guarantee that weaker notions of calibration hold near their optima.
However, it is well-established empirically that recalibrating classifiers by optimizing proper scoring rules such as the log loss approximately achieves weaker notions of calibration such as confidence calibration \cite{pmlr-v70-guo17a, ovadia2019uncertainty, minderer2021revisiting}.

%% file: tables/linguistic_calibration.tex
\begin{table*}[t]
\small
\centering
\caption{
    Each notion of linguistic $\phi$-calibration corresponds to a traditional notion of classifier calibration $\phi$ applied to the \reader~$\f: \mathcal{X} \times \mathcal{Z} \rightarrow \DeltaAns$.
}
\vspace{-3pt}
\begin{tabular}{@{\extracolsep{2pt}}lc@{}}
\midrule
\midrule
\textbf{Notion of Classifier Calibration $\phi$}         &
\textbf{Condition for Linguistic $\phi$-Calibration} \\
\midrule
Distribution Calibration \cite{kull2015novel} & $\mathbb{E}\sbracks{
\mathbf{1}_{y} \mid f(x, z) = p
} = p,\quad~\forall p \in \DeltaAns.$ \\ 
\midrule
Classwise Calibration \cite{kull2019beyond} & $\Pr \bracks{
        Y = y \mathrel{\Big|} f(x, z)_y = \beta
} = \beta,\quad~\forall y \in \mathcal{Y}, \beta \in [0, 1].
$ \\ 
\midrule
Confidence Calibration \cite{pmlr-v70-guo17a} & $\Pr \bracks{ 
        Y = \argmax_y \f(x, z)_y \mathrel{\Big|} \max_y \f(x, z)_y = \beta
    } = \beta,\quad~\forall \beta \in [0, 1].$ \\
\midrule
\midrule
\end{tabular}
\label{tab:linguistic_calibration_examples}
\end{table*}

%% file: sections/4_method.tex
\section{Method}
\label{sec:method}

In this section, we describe our training framework for linguistically calibrating an LM.
We begin by constructing a surrogate distribution to approximate the distribution of question-answer pairs that users encounter during real-world decision-making.
Then, we apply a two-step training framework (Fig.~\ref{fig:fig_2_linguistic_calibration}).
First, we obtain an LM with some ability to express confidences in a long-form generation.
Second, we use it as an RL policy and optimize our proper scoring rule objective end-to-end, with supervision from the surrogate question-answer distribution.

\begin{figure*}[t]
    \centering
    \includegraphics[width=1.0\textwidth]{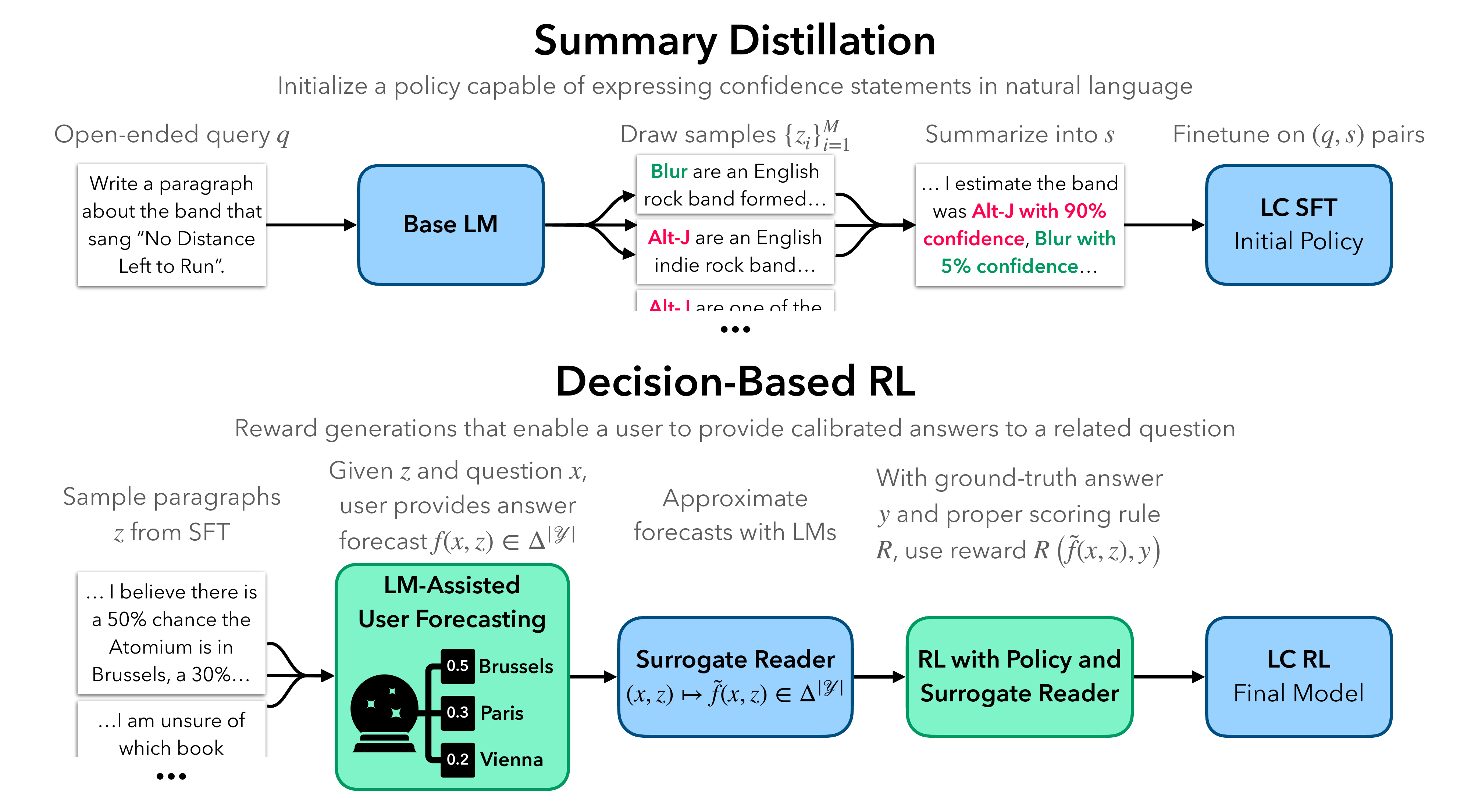}
    \caption{
    Our training framework for linguistic calibration of long-form generations (\LC) calibrates the long-form generations of an LM by calibrating downstream user forecasts.
    It involves two steps: \textbf{summary distillation (Upper)} and \textbf{decision-based RL (Lower)}.
    Datasets are in white, LMs in \textcolor{blue1}{blue}, and steps involving user or surrogate forecasts are in \textcolor{green1}{green}.
    }
    \label{fig:fig_2_linguistic_calibration}
\end{figure*}

\subsection{Generating Synthetic Supervision for Long-Form Calibration}
\label{subsec:decision_task_data}

In our setup (\S\ref{subsec:linguistic_calibration}), the process of LM-assisted user forecasting involves a tuple $(q, z, x, y)$, where $q$ are user-written queries to the LM, $z \sim \pi(z \mid q)$ are long-form generations sampled from the LM, and $(x, y)$ is a related question-answer pair.

Our training framework will closely follow this process and therefore requires access to a dataset of tuples $(q, z, x, y)$.
We can synthetically generate this dataset in a manner agnostic to the downstream task, using arbitrary question-answer pairs $(x, y)$.
In this work, we make a particular choice to use $(x, y)$ pairs from off-the-shelf question-answering datasets.

Specifically, we first sample a question-answer pair $(x, y) \sim p(x, y)$ from a question-answering dataset, which is implicitly associated with some actual decision-making tasks.
Next, we need an LM query $q$ such that $z \sim \pi(z \mid q)$ is a long-form generation \textit{salient to} $(x, y)$.
We obtain one by converting the question $x$ into an open-ended query ($q :=$ ``Write a paragraph about \{$x$\}'') using an API-based LLM. Altogether, this gives us a tuple $(q, z, x, y)$ where $(q, x, y) \sim p(q, x, y)$ and $z \sim \pi(z \mid q)$.

Next, we describe our two-step training framework.

\subsection{Summary Distillation}
\label{subsec:summ_distill}

\textbf{Summary distillation} (Fig.~\ref{fig:fig_2_linguistic_calibration} Upper) bootstraps a base LM $\pi_\Base$ to have some ability to express its confidence in long-form natural language generations.
We follow a simple approach inspired by Self-Consistency \cite{wang2023selfconsistency}, which obtains calibrated LM confidences for short answer questions by computing a statistic of many output samples.
Summary distillation generalizes this idea to longer generations, and then finetunes on our equivalent of the statistics.

First, we provide the base LM with an open-ended query $q$ and sample many long-form responses: $\{z_i\}_{i = 1}^M \sim \pi_\Base(z \mid q)$.
To obtain statements of confidence that are faithful to the base model's internal confidence levels, we prompt an API-based LLM to summarize these samples into a single consensus paragraph $s$ with statements of confidence based on the frequency of claims: $s \sim \text{Summarize}(s \mid z_1, \dots, z_M)$.
For example, we would expect the summary shown in Fig.~\ref{fig:fig_2_linguistic_calibration} (Upper) if $90\%$ of the $M$ samples answer the question with \textcolor{red1}{Alt-J} and $5\%$ with \textcolor{green1}{Blur}.
We perform frequency-based summarization at the \textit{claim level}, meaning that each summary paragraph $s$ contains multiple claims with various confidence levels and styles (\textit{e.g.,} numerical and linguistic).

Finally, to distill these extracted confidences back into the base model, we finetune $\pi_\Base$ on the dataset of open-ended query and summary pairs $\{(q^{(i)}, s^{(i)} ) \}_{i=1}^N$ to obtain the supervised finetuned (SFT) model $\pi_{\text{SFT}}$.
$\pi_{\text{SFT}}$ serves as a strong initial policy for the second RL-based step.

\subsection{Decision-Based RL}
\textbf{Decision-based RL} (Fig.~\ref{fig:fig_2_linguistic_calibration} Lower) linguistically calibrates a policy $\pi_{\RL}$ (initialized at $\pi_{\text{SFT}}$) by finetuning it to emit long-form generations $z$ that improve the calibration of the user forecast $\forecast$.

\paragraph{RL objective.}

We use our log likelihood proper scoring rule $\Sc_\LC$ (cf. \S\ref{subsec:linguistic_calibration_training_objective}) as our reward function.
Our RL objective optimizes $\Sc_\LC$ over our semi-synthetic distribution:
\begin{equation}
    \max_{\pi_\RL} \underset{(q, x, y) \sim p(q, x, y),~z \sim \pi_\RL(z \mid q)}{\E} \left[ \Sc_\LC \left(\forecast, y\right) \right].
    \label{eq:rl_objective_fn}
\end{equation}

\subsection{Implementation}
We next describe our instantiation of decision-based RL, which we used in our experiments.
However, we note that the notion of linguistic calibration defined in \S\ref{sec:prelim} is agnostic to these design decisions. 
Algorithm~\ref{alg:LC} presents the pseudocode.

\begin{figure}[t]
    \input{assets/algobox}
\end{figure}

\paragraph{Surrogate reader.} 
For our training framework to be as scalable as possible, we would ideally avoid the use of a human or LLM-simulated reader $\f$ in the RL loop.
We find that we can train a neural net surrogate reader $\widetilde{\f}: \mathcal{X} \times \mathcal{Z} \rightarrow \Delta^{|\mathcal{Y}|}$ which produces reasonable forecasts, because ``reading'' is \textit{not a fundamentally hard task}.
For example, if $z$ provides a clear list of possible answers to the question $x$ and associated percentage likelihoods, reading is a simple extractive task.

Using the surrogate, we optimize approximate reward $\Sc_\LC(\surrforecast, y) \approx \Sc_\LC (\forecast, y)$.
In our evaluation, we will test if our LM calibrated on this approximate reward generalizes to produce long-form generations $z$ which improve simulated LLM and human forecasts $\forecast$.

We cannot simply train a neural net to directly predict a softmax output $\surrforecast \in \Delta^{|\mathcal{Y}|}$, because $\mathcal{Y}$ is the vast space of all answers expressible in a finite-length string.
Instead, we decompose forecasting into two operations:

\begin{enumerate}[noitemsep,nolistsep]
    \item \text{ExtractAnswers}: $(x, z) \mapsto \text{ExtractAnswers}(x, z) \subset \mathcal{Y}$ extracts all possible answers to the question $x$ from the paragraph $z$. We implement this by finetuning a pretrained LM (RedPajama 3B, \citet{redpajama3b}).
    \item \text{ForecastProbs}: $(x, z, y') \mapsto [0, 1]$ assigns a probability to an answer $y'$ to question $x$ based on the paragraph $z$. We finetune $\pi_{\SFT}$ with a cross-entropy loss.
\end{enumerate}

We define the surrogate reader's forecast $\surrforecast$ as a categorical distribution with probability $\text{ForecastProbs}(x, z, y')$ on each answer $y' \in \text{ExtractAnswers}(x, z)$, and probability 0 on all others.
In this particular construction, we are not guaranteed that the surrogate forecast $\surrforecast$ will be normalized, but in practice we find that adding a regularization term is sufficient to enforce normalization:
\begin{align*}
    \Sc_{\LC}(\surrforecast, y) = &\log \surrforecast_{y} \!-\! 
    \lambda \left|1 \!-\! \textstyle \sum_{y' \in \mathcal{Y}} \surrforecast_{y'} \right|,
\end{align*}
where $\lambda > 1$ restores strict propriety (cf. \ref{subsec:app_reward_is_strictly_proper} for proof).
Lastly, we use a standard KL penalty from $\pi_{\text{SFT}}$ to mitigate over-optimization of the surrogate reader \cite{ouyang2022training}, giving us the following objective (with KL coefficient $\beta$):
\begin{equation}
    \max_{\pi_\RL} \underset{(q, x, y) \sim p(q, x, y),~z \sim \pi_\RL(z \mid q)}{\E} \left[ \Sc_\LC(\surrforecast, y) - \beta \log \frac{\pi_\RL(z \mid q)}{\pi_\SFT(z \mid q)} \right].
    \label{eq:rl_objective_fn_with_kl}
\end{equation}

See Appendix \ref{app:training_framework} for further implementation details.

%% file: assets/algobox.tex
\newcommand\mycommfont[1]{\footnotesize\textcolor{blue1}{#1}}
\SetCommentSty{mycommfont}

\begin{algorithm}[H]
\SetKwInput{kwGiven}{Given}
\kwGiven{SFT LM $\pi_\SFT$, LLM simulated reader $\f_\LLM: \mathcal{X} \times \mathcal{Z} \rightarrow \Delta^{|\mathcal{Y}|}$, joint distribution over open-ended queries and question-answer pairs $p(q, x, y)$.}
\KwResult{linguistically calibrated LM $\pi_\LC$.}
\tcc{Construct datasets for training surrogate reader functions: ExtractAnswers (EA) and ForecastProbs (FP).}
$\mathcal{D}_{\text{EA}} \leftarrow \{\}, \mathcal{D}_{\text{FP}} \leftarrow \{\}$. \\
\For{i = 1, \dots, N}{
\tcc{LM-assisted user forecasting (\S\ref{subsec:linguistic_calibration}).}
Sample open-ended query and question-answer pair $(q, x, y) \sim p(q, x, y)$. \\
Sample long-form generation $z \sim \pi_\SFT(z \mid q).$ \\
Apply LLM reader to obtain forecast $f_\LLM(x, z) \in \Delta^{|\mathcal{Y}|}$. \\
\tcc{Add ExtractAnswers example.}
$\mathcal{D}_{\text{EA}} \leftarrow \mathcal{D}_{\text{EA}} \cup \Big(x, z, \text{Support}(f_\LLM(x, z)) \Big)$. \\
\tcc{Add ForecastProbs example for ground-truth answer.}
$\mathcal{D}_{\text{FP}} \leftarrow \mathcal{D}_{\text{FP}} \cup \Big(x, z, y, f_\LLM(x, z)_y \Big)$. \\
\tcc{Add ForecastProbs example for all extracted answers.}
\For{extracted answer $y' \in \text{Support}(f_\LLM(x, z))$}{
$\mathcal{D}_{\text{FP}} \leftarrow \mathcal{D}_{\text{FP}} \cup \Big(x, z, y', f_\LLM(x, z)_{y'} \Big)$. \\
}
}

\tcc{Finetune surrogate LMs.}
~~Finetune $\text{ExtractAnswers}$ on $\mathcal{D}_{\text{EA}}$ and $\text{ForecastProbs}$ on $\mathcal{D}_{\text{FP}}$.

\tcc{Optimize RL objective.}
$\pi_\LC \leftarrow$ PPO with initial policy $\pi_\SFT$, surrogate reader $\widetilde{\f}$, and Objective~\ref{eq:rl_objective_fn_with_kl}.

\Return{$\pi_\LC$}
\caption{Decision-Based RL with a Surrogate Reader}
\label{alg:LC}
\end{algorithm}

%% file: sections/5_experiments.tex
\section{Experiments}
\label{sec:exp_results}

This section empirically validates our training and evaluation framework for linguistic calibration\footnote{We release all parts of our training and evaluation framework at \href{https://github.com/tatsu-lab/linguistic_calibration}{github.com/tatsu-lab/linguistic\_calibration.}}, demonstrating that it fulfills the following three goals:
\begin{enumerate}
    \item[\textbf{(1)}] \textbf{\LC provides better calibration with comparable or better accuracy.} We show that our linguistically calibrated LM $\pi_\LC$ emits long-form generations $z$ which improve the calibration of user forecasts with accuracy comparable to or better than strong baselines finetuned for factuality with RL.
    \item[\textbf{(2)}] \textbf{\LC is computationally tractable.} We show that $\pi_\LC$---which avoids the need to obtain many costly human forecasts by training with cheap surrogates---improves the calibration of \textit{human} forecasts at evaluation time. Moreover, we develop an automated framework to evaluate linguistic calibration with simulated forecasts and validate its agreement with crowdworkers.
    \item[\textbf{(3)}] \textbf{\LC generalizes well out-of-distribution.} We demonstrate that the improvement in forecast calibration due to adopting LM $\pi_\LC$ generalizes under significant domain shift to scientific question-answer distributions $p_{\text{OOD}}(x, y)$.
    We also evaluate $\pi_\LC$ on an entirely held-out task of person biography generation without any re-training, finding that $\pi_\LC$ produces calibrated claims throughout the long-form generation $z$ according to a fine-grained simulated evaluation.
\end{enumerate}

\subsection{Setup}
\label{subsec:exp_setup}

We use our training framework to linguistically calibrate Llama 2 7B, sourcing question-answer (QA) pairs $(x, y)$ (cf. \S\ref{subsec:decision_task_data}) from TriviaQA \cite{JoshiTriviaQA2017}.
We emphasize that our LMs produce \textit{long-form generations} $z$ on the question's topic, unlike previous works which calibrate models that directly predict a class distribution or short answer (cf. \S\ref{sec:related_work}).
We refer the reader to Appendix~\ref{app:training_framework} for further details on the training framework. 

\paragraph{Question-answering evaluation framework.}
Following our generative process during training (\S\ref{subsec:decision_task_data}), we use off-the-shelf QA datasets as a proxy for questions encountered during real-world decision-making, and evaluate the linguistic calibration of generations $z$ through the performance of downstream forecasts.
Specifically, for a held-out QA pair $(x, y)$, we convert $x$ into an open-ended query $q$, sample a long-form generation $z \sim \pi(z \mid q)$ from various LMs $\pi$, and evaluate the calibration and accuracy of forecast $\forecast$.

Naturally, this framework depends on which users are providing forecasts and how, \textit{i.e.,} the choice of reader $f$.
We are interested in the case where users strongly rely on the knowledge of the LM.
Therefore, we include instructions to the user (either simulated or human) to ignore their background knowledge about the correct answer when providing a forecast (cf. Appendix~\ref{app:sec_eval_details} for further evaluation details).

\paragraph{Reader expected calibration error.}
We measure the calibration of the reader $\f$ over the joint distribution $p(x, y, z)$ with \textbf{expected calibration error (ECE)} \cite{pmlr-v70-guo17a}.
Intuitively, reader ECE is a proxy for decision-making performance through the equivalence of confidence calibration and optimal decision-making: see \Cref{prop:conf_calib_implies_lr_decision_calibration} in Appendix~\ref{app:decision_and_calibration} for details.
Given $N$ question-answer pairs $\{(x^{(i)}, y^{(i)})\}_{i = 1}^N$ and corresponding long-form generations $\{z^{(i)}\}_{i=1}^N$, 
we partition them into $M$ bins by max forecast probability $\max_y \f(x^{(i)}, z^{(i)})_y$.
ECE is then expressed as
\begin{equation*}
\text{ECE} = \sum_{j = 1}^M \frac{|B_j|}{N} \left|\text{Acc}(B_j) - \text{Conf}(B_j)\right|, \quad \text{where}    
\end{equation*}
\begin{equation*}
\text{Conf}(B_j) = \frac{1}{|B_j|} \sum_{i \in B_j} \max_y \f(x^{(i)}, z^{(i)})_y,\quad~
\text{Acc}(B_j) = \frac{1}{|B_j|} \sum_{i \in B_j} \mathbbm{1} [y^{(i)} = \argmax_y \f(x^{(i)}, z^{(i)})_y].
\end{equation*}
We set the number of bins as $M = 20$ on simulated QA evaluations, and $M = 10$ on all others.
Note that log loss is not a reasonable evaluation metric in our setting because simulated and human readers can assign zero probability to the ground-truth class label resulting in infinite log loss.

\paragraph{Baselines.}
In our main evaluations, we compare \LC RL ($\pi_\LC$) with two types of baselines, all derived from Llama 2 7B: non-confidence and confidence.
We provide a strong data-matched comparison to \LC~by finetuning directly for factuality using RL.
This baseline is similar to \citet{tian2024finetuning}, but instead of using self-supervised or automated factuality scores as the RL reward, we use correctness determined with ground-truth question-answer pairs from TriviaQA.
In-context learning (ICL) baselines use TriviaQA examples from a prompt development split, and SFT/RL-based baselines use the same splits as $\pi_\SFT$ and $\pi_\LC$.
Each example in our splits is a $(q, x, y)$ tuple, where $q$ is an open-ended query obtained from question $x$ (cf. \S\ref{subsec:decision_task_data}).

We include the following non-confidence baselines:

\begin{itemize}[topsep=0pt,leftmargin=4mm]
    \setlength\itemsep{0.08mm}
    \item \textbf{ICL.} We randomly sample 8 open-ended queries, generate long-form responses with GPT-4, manually fact-check those responses using Wikipedia, and use these fact-checked (query, response) pairs as ICL examples for Llama 2 7B.
    \item \textbf{Claude Distill.} We generate long-form responses with Claude 2 over all queries in the SFT split, and finetune Llama 2 7B on these (query, response) pairs.
    \item \textbf{Factuality SFT.} We use the above ICL baseline to generate long-form responses over all queries in the SFT split, and finetune Llama 2 7B on these (query, response) pairs. We found Factuality SFT to outperform Claude Distill on a TriviaQA validation split, so we use it as the starting point for the following baseline, Factuality RL.
    \item \textbf{Factuality RL.} To provide a strong RL-based baseline, we train a reward model which scores the correctness of long-form outputs and use it in RL. Our approach to obtain this baseline is analogous to the decision-based RL algorithm (Algorithm \ref{alg:LC}), except instead of training a surrogate reader, we train a single reward model that, given a generation $z$ and question-answer pair $(x, y)$, predicts a binary indicator whether $z$ provides the correct answer to the question. This serves as the RL reward. We use Factuality SFT as the initial policy for PPO.
\end{itemize}

When training our confidence methods, we use the ICL baseline above to generate the responses which are summarized in summary distillation.
Our confidence baselines include the \textbf{\LC SFT} model ($\pi_\SFT$) and the following baseline:
\begin{itemize}[topsep=0pt,leftmargin=4mm]
    \setlength\itemsep{0.08mm}
    \item \textbf{Summary ICL.} We use the summary distillation algorithm (\S\ref{subsec:summ_distill}) on 8 queries sampled from the prompt development split to produce 8 Claude 2 summaries $\{s^{(i)}\}_{i=1}^8$, which we use in (query, summary) ICL examples.
\end{itemize}

\paragraph{Other baselines including GPT-4.} We include results for several other methods in Appendix \ref{app:more_qa_results} including Llama 2 Chat (which underperformed Factuality SFT), the oracle baseline of direct evaluation of summaries $s$, and GPT-4--based methods including GPT-4 0-Shot, GPT-4 ICL 8-Shot, asking GPT-4 for confidence statements zero-shot, and Summary ICL using GPT-4.
Unsurprisingly, GPT-4--based methods are far more factual than all Llama 2 7B--based methods.
However, we find that \LC RL has reader ECE comparable to GPT-4 baselines (cf. Figs.~\ref{fig:appendix_tqa_frontier}, \ref{fig:appendix_jeopardy_frontier}, \ref{fig:appendix_sciq_frontier}, \ref{fig:appendix_bioasq_frontier}, \ref{fig:appendix_factscore_frontier}), despite significantly worse factuality.
This demonstrates that even small LLMs with relatively weak factuality can be well-calibrated with the right objective.

\begin{figure*}[t]
    \captionsetup[subfigure]{justification=centering,aboveskip=-0.5pt,belowskip=-1pt}
    \centering
    \includegraphics[height=0.05\textwidth]{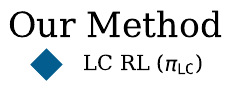}
    \includegraphics[height=0.05\textwidth]{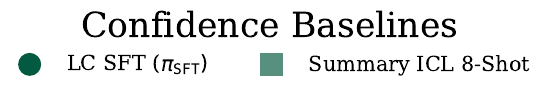}
    \includegraphics[height=0.047\textwidth]{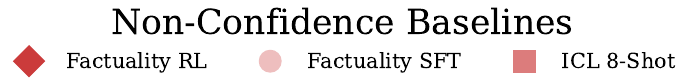}
    \begin{subfigure}[b]{0.32\textwidth}
        \includegraphics[height=0.7\textwidth]{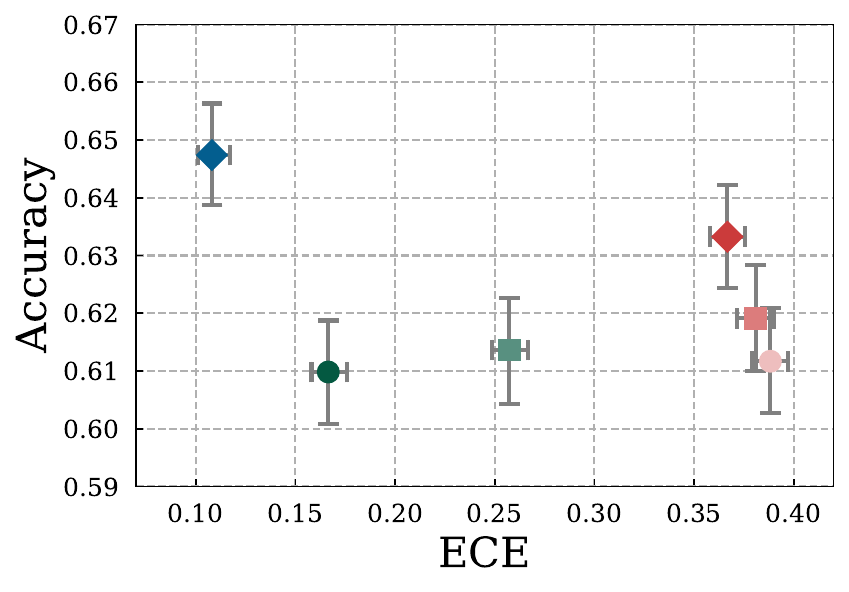}
        \caption{TriviaQA (in-distribution),\\~simulated reader.}
        \label{fig:frontier_tqa_11k}
    \end{subfigure}
    \begin{subfigure}[b]{0.32\textwidth}
        \includegraphics[height=0.7\textwidth]{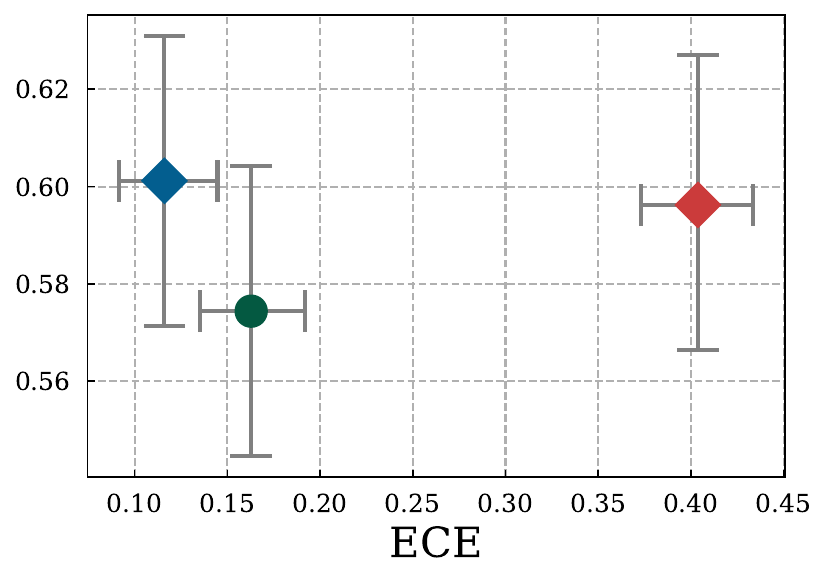}
        \caption{TriviaQA (in-distribution),\\~human readers.}
        \label{fig:frontier_tqa_human_eval}
    \end{subfigure}
    \begin{subfigure}[b]{0.32\textwidth}
        \includegraphics[height=0.7\textwidth]{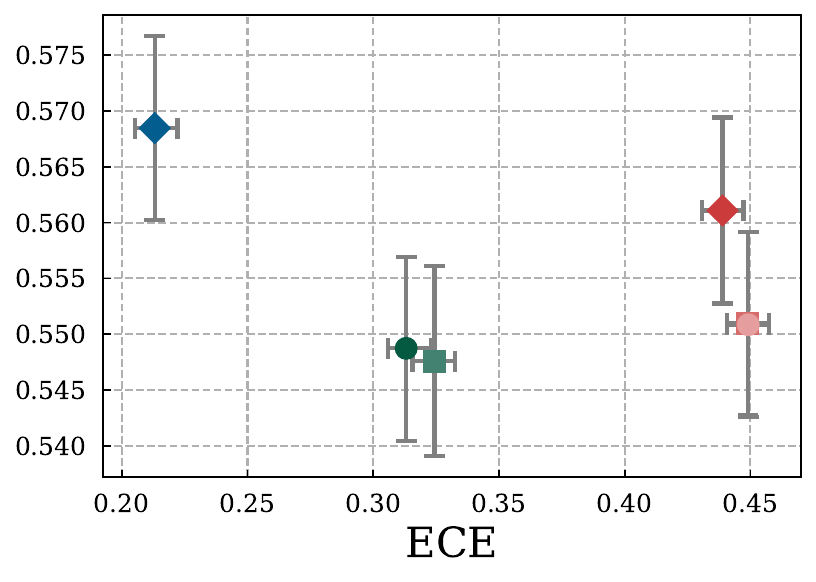}
        \caption{SciQ (out-of-distribution),\\~simulated reader.}
        \label{fig:frontier_sciq}
    \end{subfigure}
    \caption{\textbf{Accuracy-ECE Frontier for Question-Answering} (upper left is better). \LC RL pareto-dominates Factuality RL and SFT, with significantly better reader ECE while matching or exceeding their accuracy.}
    \label{fig:frontier_plots_fig_3}
\end{figure*}

\begin{figure}[t]
    \captionsetup[subfigure]{justification=centering,aboveskip=-0.5pt,belowskip=-1pt}
    \vspace{-1em}
    \centering
    \begin{subfigure}{\textwidth}
    \centering
    \includegraphics[width=0.4\linewidth]{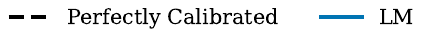}    
    \end{subfigure}
    \newline
    \begin{subfigure}{0.24\textwidth}
        \includegraphics[width=\textwidth]{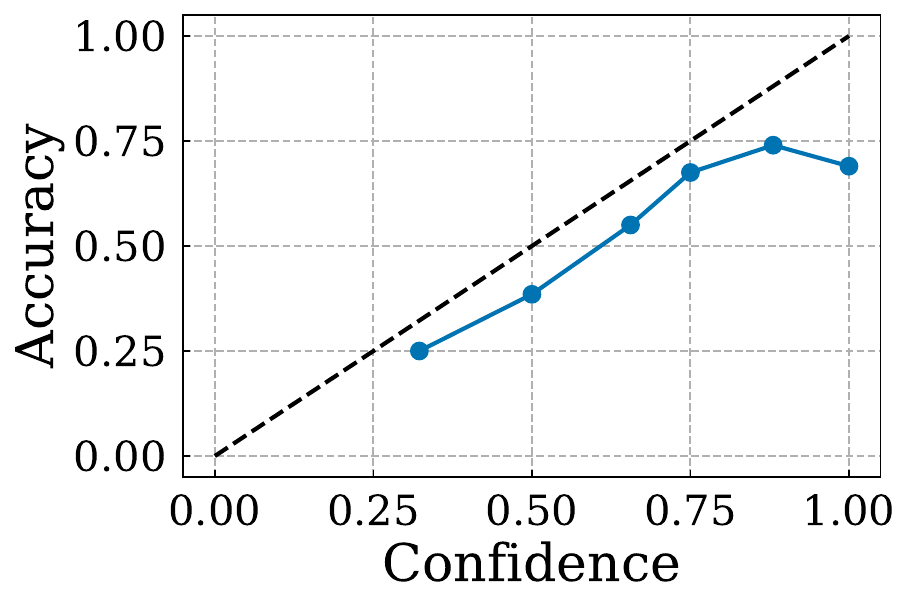}
        \caption{\LC SFT, human\\readers. ECE: $0.163$}
        \label{fig:reliability_pi_sft_human}
    \end{subfigure}
    \begin{subfigure}{0.24\textwidth}
        \includegraphics[width=\textwidth]{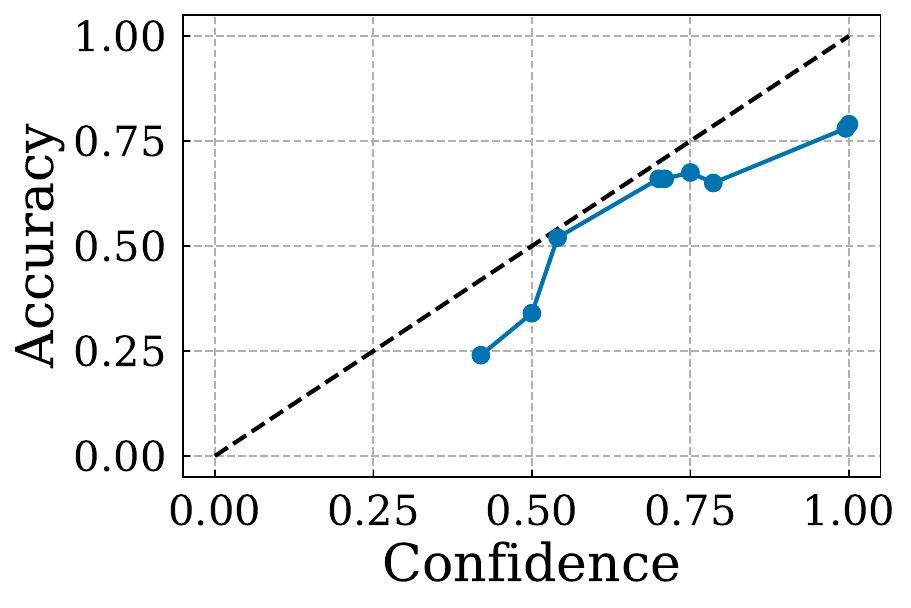}
        \caption{\LC RL, human\\readers. ECE: $0.116$}
        \label{fig:reliability_pi_lc_human}
    \end{subfigure}
    \begin{subfigure}{0.24\textwidth}
        \includegraphics[width=\textwidth]{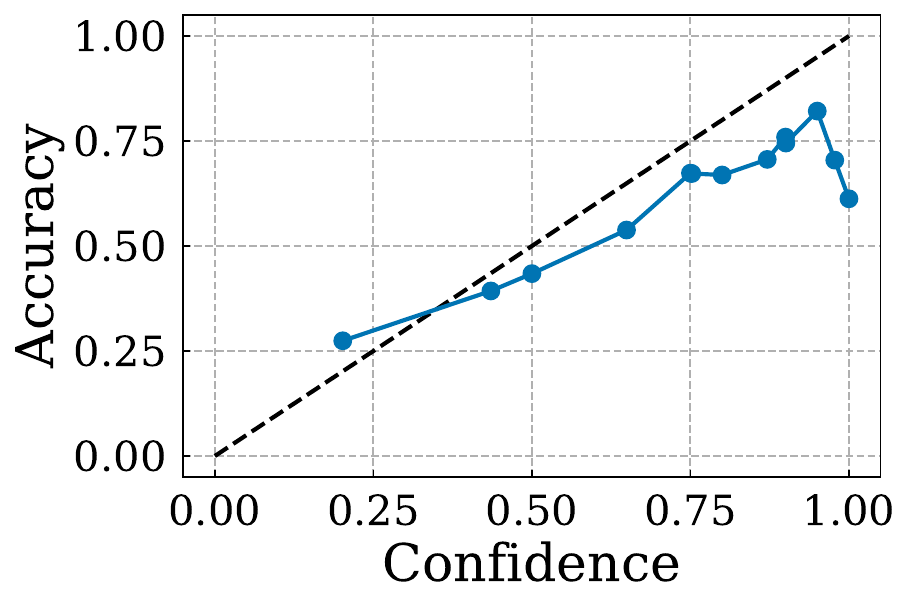}
        \caption{\LC SFT, simulated\\reader. ECE: $0.166$}
        \label{fig:reliability_pi_sft_sim}
    \end{subfigure}
    \begin{subfigure}{0.24\textwidth}
        \includegraphics[width=\textwidth]{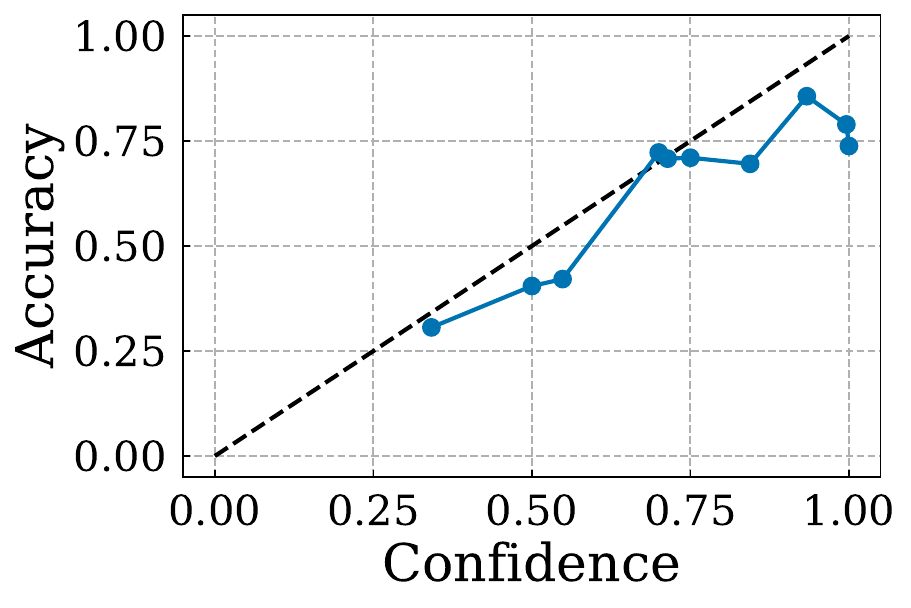}
        \caption{\LC RL, simulated\\reader. ECE: $0.108$}
        \label{fig:reliability_pi_lc_sim}
    \end{subfigure}
    \caption{\textbf{TriviaQA Reliability Diagrams}. \LC models display a wide range of confidences and good calibration in their long-form generations, with \LC RL improving calibration further. Human and simulated results closely match.}
    \label{fig:reliability_diagram_fig_4}
\end{figure}

\subsection{Linguistic Calibration using Question-Answering Datasets}

To begin, we evaluate our methods using held-out $(x, y)$ pairs from the TriviaQA, Jeopardy \cite{kaggle2020jeopardy}, SciQ \cite{welbl2017crowdsourcing}, and BioASQ Task B \cite{bioasq} question-answering datasets.
The latter two datasets are a significant domain shift; SciQ consists of crowdsourced physics, chemistry, and biology exam questions, and BioASQ Task B consists of biomedical questions annotated by experts (cf. Appendix \ref{app:sec_eval_details} for dataset details).

These QA evaluations validate two of our three goals: we find that \LC improves calibration with comparable or better accuracy, and that our training and evaluation framework are computationally tractable.
Our strong results on the Jeopardy, SciQ, and BioASQ datasets also provide a partial validation of our final goal of strong generalization under distribution shift.
We report TriviaQA and SciQ results here, and refer the reader to Appendix \ref{app:more_qa_results} for results with all datasets and baselines.

\paragraph{Better ECE with comparable or better accuracy in long-form generation.} 
Our main result in Fig.~\ref{fig:frontier_plots_fig_3} is that \LC RL has significantly better ECE than non-confidence baselines, including Factuality RL, while matching or exceeding their accuracy.
This result holds in-distribution on TriviaQA with both simulated (Fig.~\ref{fig:frontier_tqa_11k}) and human (Fig.~\ref{fig:frontier_tqa_human_eval}) readers, and out-of-distribution (OOD) on SciQ (Fig.~\ref{fig:frontier_sciq}), Jeopardy (Fig.~\ref{fig:appendix_jeopardy_frontier}), and BioASQ (Fig.~\ref{fig:appendix_bioasq_frontier}) with a simulated reader, demonstrating that \LC~generalizes across distribution shifts in the question-answer distribution $p(x, y)$.

Our results also support the effectiveness of decision-based RL.
\LC RL significantly improves over \LC SFT in both ECE and accuracy, with a greater absolute improvement in ECE/accuracy than Factuality SFT to Factuality RL.
This supports our claim that optimizing proper scoring rules of downstream forecasts is an effective way to induce calibration in long-form generations.

\begin{figure}[t]
\begin{center}
{ \textbf{Prompt:} Write a paragraph bio about Rory Byrne.}
\end{center}
\small
\begin{qualitativeBox}
    \textbf{Factuality RL:} Rory Byrne is a renowned \hlred{British} engineer and designer who has made significant contributions to the world of Formula One racing. Born in \hlred{1949}, Byrne began his career in the automotive industry in the 1970s, working for various racing teams and manufacturers. \hlred{In 1981, he joined the Williams Formula One team as a designer, and over the next few years, he helped the team win several championships. In 1994, Byrne left Williams} to join Ferrari, where he served as the team's technical director for several years. During his time at Ferrari, Byrne helped the team win several championships, including the 1999 constructor's championship. \hlred{In 2005, Byrne left Ferrari to pursue other opportunities, and he has since worked as a consultant for various racing teams and manufacturers.}
\end{qualitativeBox}
\begin{qualitativeBox}
    \textbf{\LC RL:} I \hlblue{believe} Rory Byrne is a renowned \hlred{British} engineer who has made significant contributions to the automotive industry. I \hlblue{estimate with 75\% confidence} that he has worked as a chief designer for Formula 1 teams like Ferrari and \hlred{McLaren}. I also \hlblue{think with 50\% likelihood} that he has designed cars for teams like \hlred{Williams}, Benetton, and \hlred{Lotus}. I am \hlblue{uncertain} whether he was \hlred{born in Ireland or Scotland}. I \hlblue{believe with 90\% confidence} that he has won multiple design awards, \hlred{including the prestigious Loeb Award}. \hlblue{While the details of his career are unclear to me, I am certain that} Rory Byrne has had a profound impact on the development of Formula 1 cars over the past few decades. His innovative designs have helped teams achieve success on the track.
\end{qualitativeBox}
\caption{Qualitative example from Factuality and \LC RL when evaluated under task distribution shift on biography generation. \LC RL produces numerical and linguistic confidence statements throughout the paragraph, highlighted in \hlblue{blue}. False statements are highlighted in \hlred{red}.
We include additional examples in Appendix~\ref{app:qualitative_examples}.
}
\label{fig:qualitative_main_body}
\end{figure}

\paragraph{Reliability diagrams demonstrate meaningful confidences.} 
A natural question is whether the confidences learned by our \LC~models are meaningful.
For example, if all confidences of a model collapsed to its average accuracy, it would obtain perfect ECE despite having confidences that are useless for tasks such as conveying the likelihood of hallucination.
To obtain a more fine-grained understanding of a model's calibration, we use \textbf{reliability diagrams} \cite{degroot1983comparison, niculescu2005predicting}, which visualize the average confidence and accuracy of each ECE bin.
The plot of the pathological model above would be a single point.
A perfectly calibrated model with meaningful confidences would have an identity plot $y = x$.
In Fig.~\ref{fig:reliability_diagram_fig_4}, we observe that \LC model confidences are both meaningful, covering a wide range of confidence values, and consistently close to the identity across confidence values.
This validates that \LC~is effective in linguistically conveying the likelihood of hallucination in a long-form generation.

\begin{figure}[t]
\centering
\includegraphics[height=0.05\textwidth]{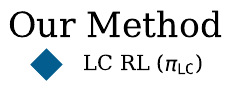}
\includegraphics[height=0.05\textwidth]{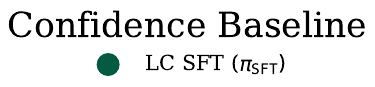}
\includegraphics[height=0.047\textwidth]{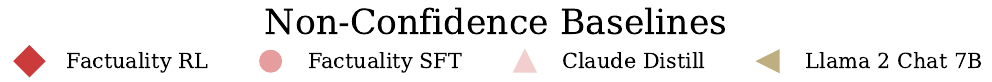}
\includegraphics[width=0.5\textwidth]{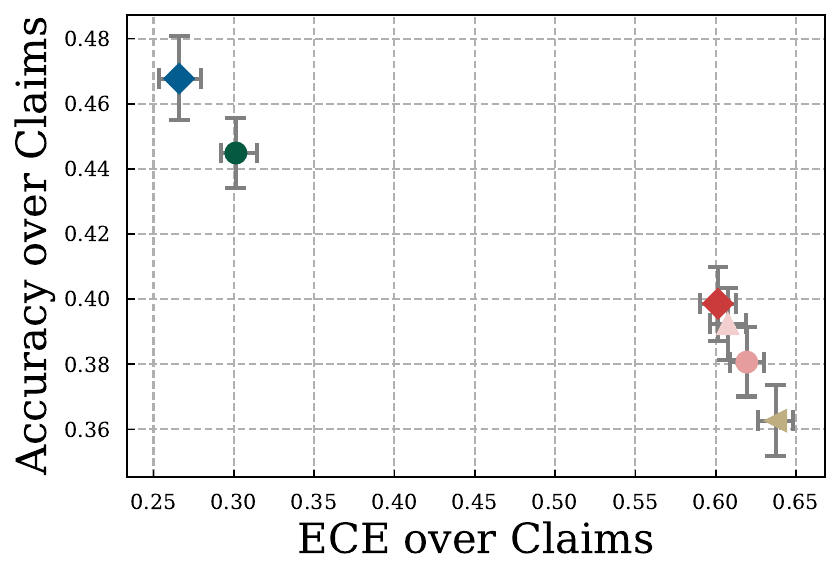}
\caption{\textbf{Accuracy-ECE Frontier for Person Biography Generation}, simulated reader (upper left is better).
Accuracy and ECE are computed at the claim level, using a finetuned retrieval model and API-based LLM for fact checking, following \citet{min-etal-2023-factscore}. 
\LC RL significantly outperforms non-confidence methods in accuracy and ECE.}
\label{fig:fig_5_factscore_frontier}
\end{figure}

\begin{figure}[t]
\captionsetup[subfigure]{justification=centering,aboveskip=-0.25pt,belowskip=-1pt}
\centering
\vspace{-0.5em}
\begin{subfigure}{\textwidth}
    \centering
    \includegraphics[width=0.4\linewidth]{fig/Fig4/reliability_diagram_plot_v1_legend.pdf}    
\end{subfigure}
\newline
\begin{subfigure}{0.3\textwidth}
    \includegraphics[width=\textwidth]{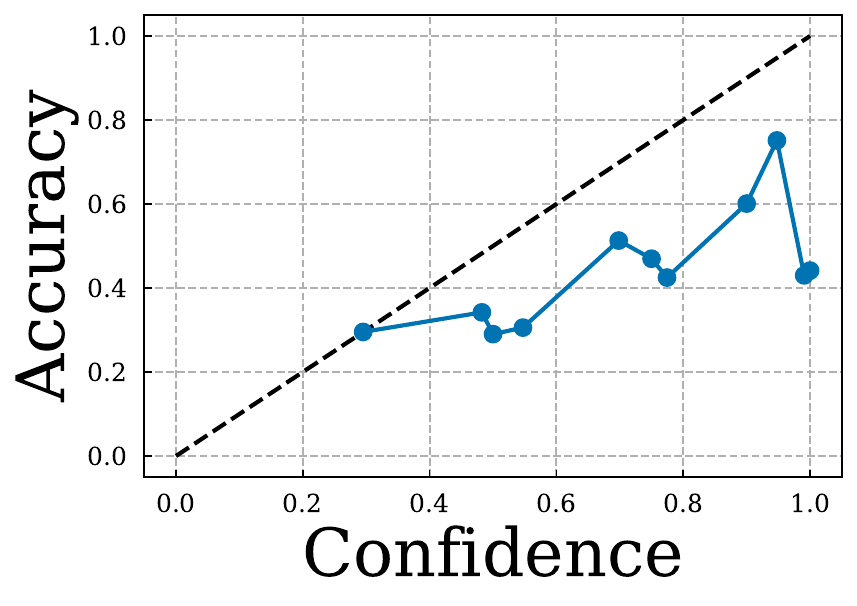}
    \caption{\LC SFT. ECE: $0.301$}
    \label{fig:factscore_reliability_sft}
\end{subfigure}
\begin{subfigure}{0.3\textwidth}
    \includegraphics[width=\textwidth]{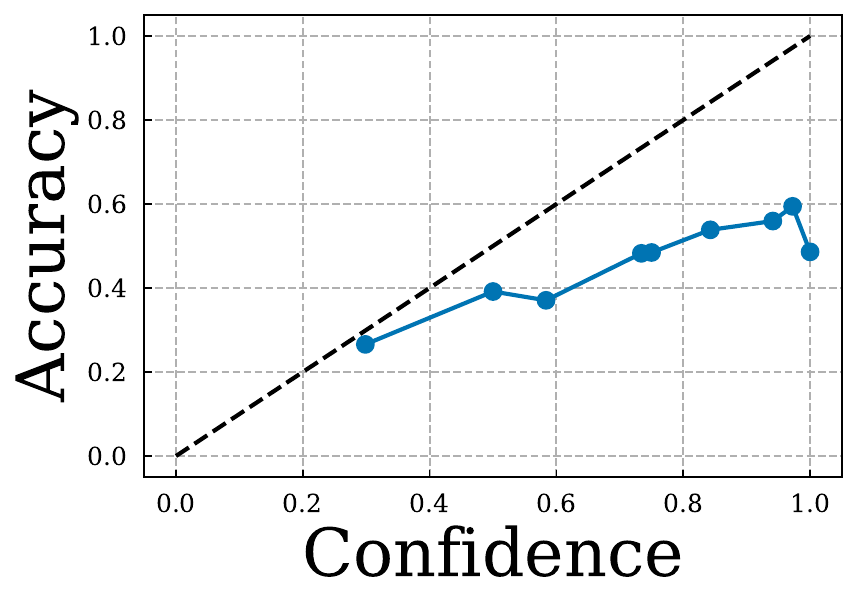}
    \caption{\LC RL. ECE: $0.266$}
    \label{fig:factscore_reliability_rl}
\end{subfigure}
\caption{\textbf{Person Biography Generation Reliability Diagrams} for \LC SFT (Left) and \LC RL (Right); simulated reader. \LC methods demonstrate non-trivial calibration in an entirely out-of-distribution task which requires calibrated statements of confidence throughout the paragraph.}
\label{fig:fig_6_reliability_factscore}
\end{figure}

\subsection{Zero-Shot Generalization to a Biography Generation Task}
\label{subsec:exp_biography_generation}

The QA evaluation validated two of our three goals:
\textbf{(1)} \LC RL pareto-dominates baselines on the accuracy-ECE frontier. Its significant improvement over \LC SFT validates the effectiveness of decision-based RL.
\textbf{(2)} We demonstrated the computational efficiency of our training and evaluation framework, because \LC RL is trained with cheap surrogates but performs well in evaluations with human readers, and our simulated readers have high agreement with human readers (see Appendix \ref{app:app_human_agreement_stats} for full forecast agreement statistics).
Lastly, the QA evaluation partially validates \textbf{(3)} OOD generalization: \LC RL performs well on the Jeopardy, SciQ, and BioASQ datasets with simulated readers.

To conclusively validate this final goal, we evaluate \LC on a significant distribution shift in the task.
Our models were trained to perform long-form generation about trivia-style questions, and we now test their ability to write factual biographies on people sampled from Wikipedia.
Specifically, we source 500 people from the \texttt{unlabeled} split of FactScore \cite{min-etal-2023-factscore} and use prompt ``Write a paragraph bio about \{person\}'' (cf. Fig.~\ref{fig:qualitative_main_body} for a qualitative example, and additional ones in Appendix \ref{app:qualitative_examples}).
We emphasize that \textit{our models were not trained on biography generation}.

\paragraph{FactScore-based metric.} 
We also use a more fine-grained simulated evaluation than the QA tasks, testing the accuracy and calibration of generated biographies $z$ at the per-claim level.
More specifically, we split generations into a list of claims, filter out subjective claims, and then compute accuracy and ECE over all claims pooled across biographies, following FactScore \cite{min-etal-2023-factscore}.
We use Claude 2 for splitting and filtering, and an identical fact checking pipeline to \citeauthor{min-etal-2023-factscore} other than using Claude 2 instead of ChatGPT for fact-checking conditioned on retrieved Wikipedia context paragraphs.
In order to compute ECE, we need to assign confidence values to each atomic claim.
For numerical uncertainties such as percentages, this is a simple extractive task which API-based LLMs perform well.
For linguistic uncertainties, we provide Claude 2 with a short list of mappings between linguistic phrases and consensus probabilities collected in a human study \cite{wallsten}, and allow the LLM to generalize from this to assign probabilities for other linguistic phrases not present in the mapping (cf. Appendix \ref{app:factscore_metric} for details).

\paragraph{\LC generalizes to biography generation and claim-level evaluation.}
Both \LC methods demonstrate significant improvements in ECE and accuracy compared to non-confidence baselines, generalizing well to an entirely held-out task (Fig.~\ref{fig:fig_5_factscore_frontier}).
Because we here compute ECE at the per-claim level and the \LC methods obtain reasonable ECE values and reliability diagrams (Fig.~\ref{fig:fig_6_reliability_factscore}), we confirm that they incorporate calibrated statements of confidence \textit{throughout their long-form generations}.
Additionally, decision-based RL significantly improves both accuracy and ECE over \LC SFT even under a significant task distribution shift, further validating our linguistic calibration objective.

Lastly, we note the surprising finding that \LC SFT improves in accuracy compared to Factuality RL.
We attribute this to the tendency of \LC models to generate a higher proportion of less ``precise'' claims which are still objective and correct and therefore count towards the accuracy metric.

Altogether, our strong results in a challenging distribution shift setting validate our final goal.

%% file: sections/6_related_work.tex
\section{Related Work}
\label{sec:related_work}

\paragraph{Calibration.}
The calibration of probabilistic forecasts is an extensively studied topic~\cite{brier1950verification,savage1971elicitation,murphy1973new,dawid1984present,zadrozny2001obtaining,gneiting2007strictly, pmlr-v80-hebert-johnson18a, kull2015novel}.
In particular, isotonic regression~\cite{niculescu2005predicting}, Platt scaling~\cite{platt1999probabilistic}, temperature scaling~\cite{pmlr-v70-guo17a}, and histogram binning~\cite{kumar2019verified} are effective approaches for improving the calibration of probability estimates.
Other established methods improve calibration through ensembling \citep{lakshminarayanan2017simple} and Bayesian model averaging using an approximate posterior over model parameters \cite{blundell2015mfvi, gal2016dropout, band2021benchmarking, malinin2021shifts, nado2022uncertainty, tran2022plex}.

Recent works also studied the probabilistic calibration of LMs, showing that LMs can be well-calibrated~\cite{kadavath2022language}.
LMs that went through RL from human feedback (RLHF) tend to be ill-calibrated on multiple-choice questions~\cite{openai2023gpt4}, but simple temperature scaling fixes the issue.
Chain-of-thought prompting leads to ill-calibration~\cite{bai2022constitutional}.
The present focus is the linguistic calibration of LMs that produce \textit{long-form natural language}, a setting to which established recalibration methods do not apply.
This is because long-form generations may contain multiple claims, each of which has its confidence expressed in language.

\paragraph{Calibration and decision-making.}
The particular decision-making framework we adopted was originally used to convey the confidence of individual predictions to decision-makers~\cite{zhao2021right}, and later used to draw an equivalence between different notions of calibration and optimal decision-making~\cite{zhao2021calibrating} (cf. \S\ref{subsec:user_decision_problem} and Appendix~\ref{app:decision_and_calibration}).
In seminal work, \citet{foster1998asymptotic, foster1999regret} showed that the existence of certain no regret schemes in an online decision-making setting imply the existence of calibrated probabilistic forecasts.
More recently, \citet{cresswell2024conformal} explore using conformal prediction sets in classification tasks to improve human decision-making, but do not consider calibrating long-form LM generations.

\paragraph{LMs producing uncertainty.}
The literature has studied methods to let LMs produce confidence scores. 
Works have benchmarked the uncertainty quantification capabilities of LMs on multiple choice tasks~\citep{ye2024benchmarking} and studied how sampling variance could be used to estimate uncertainty~\citep{malinin2021uncertainty,kuhn2023semantic,wang2022selfconsistency, xiong2024can}.
Our sampling-based Summary ICL and Direct Summary Eval baselines (cf. Appendix \ref{app:more_qa_results}) can be seen as extensions of the black-box confidence elicitation framework in \citet{xiong2024can} to long-form text, because they do not require finetuning or access to logits.

Other works analyze how LMs express linguistic uncertainty~\citep{mielke-etal-2022-reducing,zhou-etal-2023-navigating}.
Considering single-claim utterances, \citet{mielke-etal-2022-reducing} define an LM to be linguistically calibrated if it emits a verbalized confidence statement matching the likelihood of its response's correctness.
However, the long-form, multi-claim generations that users encounter in practice have neither a single closed-form confidence nor a correctness; each generation contains information that answers many possible downstream questions.
Defn.~\ref{defn:ling_cal} recovers the notion of linguistic calibration in \citeauthor{mielke-etal-2022-reducing}~when specialized to single-claim outputs $z$.\footnote{
Let $q := x$, and suppose the LM generates a single-claim utterance $z \sim \pi(z \mid q)$ with a confidence statement and answer from $\mathcal{Y}$, \textit{e.g.,} $z =$ ``I'm not sure, but my guess is Los Angeles.'' Then \citet{mielke-etal-2022-reducing} considers \textit{linguistic confidence calibration}.}

A related line of work enables LMs to \textit{directly} express uncertainty, focusing on short utterances~\citep{mielke-etal-2022-reducing, xiong2024can}, arithmetic problems~\citep{lin2022teaching}, and question-answering~\citep{jiang2021how, tian-etal-2023-just, shrivastava2023llamas}. 
All these prior works consider settings where the set of responses is a small closed set and the notion of calibration is well-defined.
Other works finetune LMs to abstain \cite{cheng2024ai} or to output templated uncertainty phrases \cite{yang2023alignment} on question-answering tasks.
Lastly, concurrent work \cite{huang2024calibrating} evaluates methods such as self-consistency \cite{wang2022selfconsistency} and supervised finetuning in calibrating long-form generations.
To the best of our knowledge, our method is the first to simultaneously provide calibrated \textit{text-based} statements of confidence, which are important for interpretability to users \cite{mielke-etal-2022-reducing}, while working in the setting of \textit{long-form, multi-claim} generations.
We overcome the challenge of defining calibration in this setting by drawing connections between decision-making and uncertainty quantification, enabling us to build a \textit{single end-to-end objective} that can calibrate long-form generations.

\paragraph{Other LM finetuning works.}
Improving the factuality of LMs is a complementary approach to calibration in mitigating LM hallucinations.
Previous works have improved the factuality of LMs by finetuning on self-supervised or automated factuality scores \cite{tian2024finetuning, akyürek2024deductive}.
A related line of work uses supervised finetuning and RL to improve the honesty of LLMs \cite{askell2021general, ouyang2022training, park2023ai, evans2021truthful, cui2023ultrafeedback}, hypothesizing that the pretraining objective alone is insufficient to encourage honest responses.
Because improving factuality alone can improve calibration metrics such as ECE, we include a strong baseline finetuned with RL on ground-truth factuality labels and find that our approach to linguistic calibration significantly improves ECE beyond this baseline while matching or exceeding its accuracy.

%% file: sections/7_discussion.tex
\section{Limitations, Future Work, and Conclusions}
\label{sec:discussion}

\paragraph{Limitations and future work.}
Our linguistically calibrated LM generalizes well from surrogate to crowdworker forecasts.
However, many of the confidence statements it emits are fairly unambiguous, \textit{e.g.,} percentages.
Therefore, future work could investigate how closely LM and human interpretations of ambiguous linguistic confidence statements match, which could enable training LMs with linguistic confidence statements that are tailored to user populations.
Additionally, we use off-the-shelf question-answering datasets as a proxy for questions encountered during real-world decision-making.
To improve \LC's generalization to decision-making scenarios in-the-wild, future work could curate a more representative QA dataset.
Lastly, we work in a white-box setting where finetuning LMs is possible; our training framework could not be used to calibrate API-based LLMs that only provide access to completions.

\paragraph{Conclusions.}
We defined linguistic calibration of long-form generations: calibrating the long-form generations of an LM in a way that leads to calibrated probabilistic forecasts by its downstream users.
By constructing an objective in the space of these forecasts, we were able to apply the standard calibration machinery of proper scoring rules for end-to-end linguistic calibration.
Instantiating this objective in a training framework and linguistically calibrating Llama 2 7B enables it to emit calibrated confidence statements, significantly improving the calibration of downstream human and simulated forecasts while matching or exceeding strong RL-tuned baselines in accuracy.

\section{Acknowledgements}

We thank Michael Y. Li, Yann Dubois, Yu Sun, Zitong Yang, and members of the Tatsu Lab, Ma Lab, Stanford Machine Learning Group, and Stanford NLP Group for their helpful feedback.
This work was supported by Open Philanthropy, IBM, NSF IIS 2211780, the Stanford HAI--Google Cloud Credits Program, and the Anthropic Researcher Access Program.
NB acknowledges funding from an NSF Graduate Research Fellowship and a Quad Fellowship. 
XL acknowledges funding from a Stanford Graduate Fellowship and a Meta PhD Fellowship.

%% file: sections/8_appendix.tex
\crefalias{section}{appsec}
\crefalias{subsection}{appsec}
\crefalias{subsubsection}{appsec}

\setcounter{equation}{0}
\renewcommand{\theequation}{\thesection.\arabic{equation}}

\onecolumn

\section*{\LARGE Supplementary Material}
\label{sec:appendix}

\section*{Table of Contents}
\vspace*{-10pt}
\startcontents[sections]
\printcontents[sections]{l}{1}{\setcounter{tocdepth}{2}}

\newpage

\section{Additional Results}
\label{app:more_qa_results}

\subsection{Codebase}

For prompts, checkpoints, and implementations of methods and datasets, please refer to our codebase:
\begin{ourMethodBox}
\centering
\large
\href{https://github.com/tatsu-lab/linguistic_calibration}{\texttt{https://github.com/tatsu-lab/linguistic\_calibration}}
\end{ourMethodBox}

\subsection{Additional Baselines}

In addition to the baselines described in \S\ref{subsec:exp_setup}, we below provide full results including several other baselines.
All references to GPT-4 in this paper use version \texttt{gpt-4-1106-preview}.
Unless otherwise specified, all references to Claude or Claude 2 use version \texttt{claude-2.0}.
Due to compute constraints, all GPT-4 baselines are evaluated on a smaller number of question-answering test samples: 1000 samples for TriviaQA, Jeopardy, and SciQ, and the entire 1515 sample test set for BioASQ Task B (cf. Appendix~\ref{app:sec_eval_details} for further dataset and evaluation details).

We include the following additional non-confidence baselines:
\begin{itemize}
    \item \textbf{Llama 2 Chat.} We zero-shot prompt Llama 2 Chat to generate long-form responses to evaluation queries.
    \item \textbf{GPT-4 0-Shot.} We zero-shot prompt GPT-4 to generate long-form responses to evaluation queries.
    \item \textbf{GPT-4 ICL 8-Shot.} Analogous to the Llama 2 7B \textbf{ICL} baseline. We randomly sample 8 queries, generate long-form responses with GPT-4, manually fact-check those responses using Wikipedia, and use these fact-checked (query, response) pairs as ICL examples for GPT-4.
\end{itemize}

We also include the following additional confidence baselines:
\begin{itemize}
    \item \textbf{GPT-4 Just Ask for Uncertainty (JAFU) 0-Shot.} We zero-shot prompt GPT-4 to generate long-form responses to evaluation queries, and include an instruction in the prompt directing GPT-4 to indicate any uncertainty in its claims using probabilities.
    \item \textbf{GPT-4 Summary ICL 8-Shot.} Analogous to the Llama 2 7B \textbf{Summary ICL} baseline. We sample 8 queries from the prompt development split $\{q^{(i)}\}_{i = 1}^8$ (\textit{i.e.,} questions $x$ which have been converted to open-ended queries). For each query $q^{(i)}$, we generate 8 long-form responses from GPT-4: $\{z^{(i)}_j\}_{j = 1}^8 \sim \pi_{\text{GPT-4}}(z \mid q^{(i)})$. We then summarize these responses into a single consensus response: $s^{(i)} \sim \text{Summarize}(s \mid z^{(i)}_1, \dots, z^{(i)}_8)$. Finally, we use the queries and summaries as ICL examples $\{(q^{(i)}, s^{(i)})\}_{i=1}^8$.
    \item \textbf{Direct Summary Eval.} For a given evaluation query $q$, we use the Llama 2 7B ICL baseline to generate 8 long-form responses $\{z_i\}_{i = 1}^8$, and use Claude 2 to summarize these responses into a single consensus response: $s \sim \text{Summarize}(s \mid z_1, \dots, z_8)$. Then, we directly evaluate these summaries $s$. \textbf{This is an oracle baseline} because it requires sampling several long-form generations from the base LM at evaluation time, followed by summarization with an API-based LLM.
\end{itemize}

\newpage
\subsection{TriviaQA: Full Accuracy-ECE Frontier}

\begin{figure*}[ht!]
    \centering

    \includegraphics[height=0.07\textwidth]{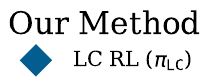}
    \includegraphics[width=0.9\textwidth]{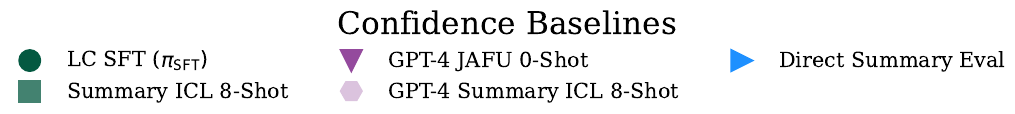}
    \includegraphics[width=0.7\textwidth]{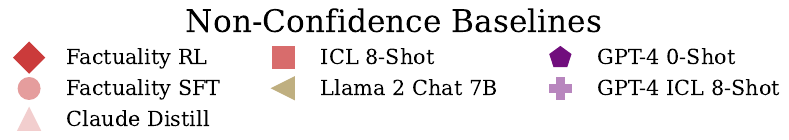}
    \includegraphics[width=\textwidth]{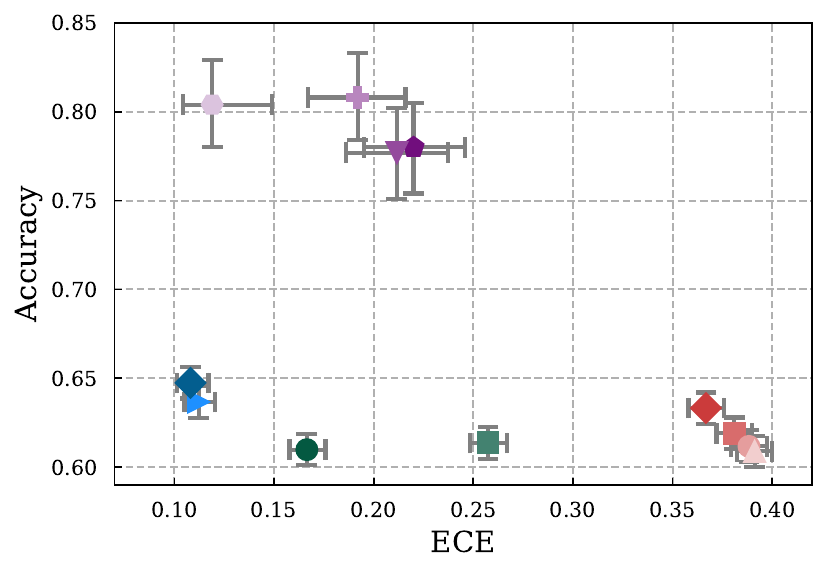}
    \caption{\textbf{Accuracy-ECE Frontier for All Methods on TriviaQA,} evaluated with simulated reader (upper left is better). 
    GPT-4 Summary ICL works well despite its simplicity. 
    \LC RL outperforms all methods in ECE, including GPT-4 Summary ICL and Direct Summary Eval.
    Llama 2 Chat 7B performs significantly worse than all other methods and is not displayed for clarity (cf. Table~\ref{tab:metrics_qa} for numerical results).
    }
    \label{fig:appendix_tqa_frontier}
\end{figure*}

\newpage

\subsection{TriviaQA: Additional Reliability Diagrams}
\label{app:more_reliability_diagram}
\FloatBarrier
\begin{figure}[h]
    \captionsetup[subfigure]{justification=centering,aboveskip=-0.5pt,font=normalsize}
    \centering
    \begin{subfigure}{0.45\textwidth}
        \includegraphics[width=\textwidth]{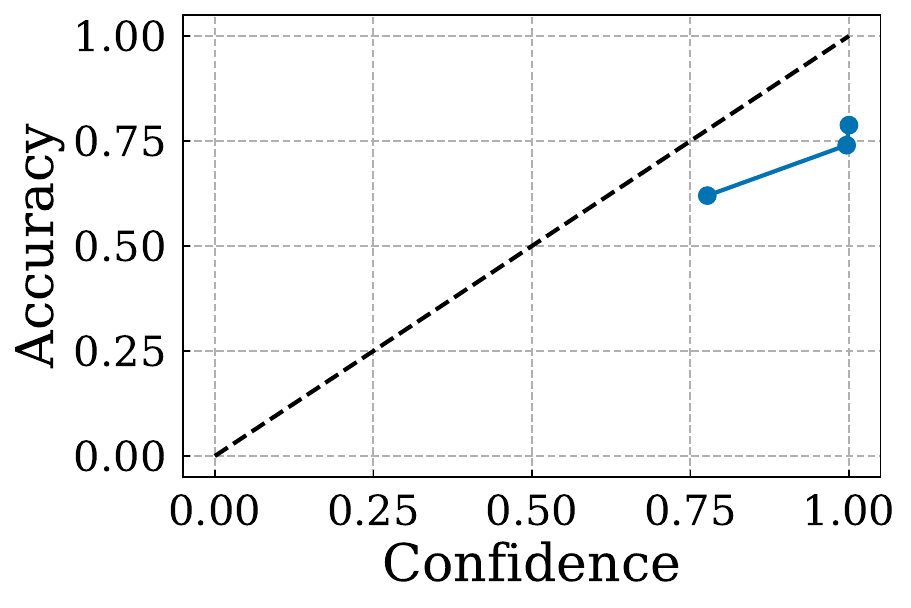}
        \caption{GPT-4 JAFU 0-Shot. ECE: $0.212$}
        \vspace{1em}
    \end{subfigure}
    \begin{subfigure}{0.45\textwidth}
        \includegraphics[width=\textwidth]{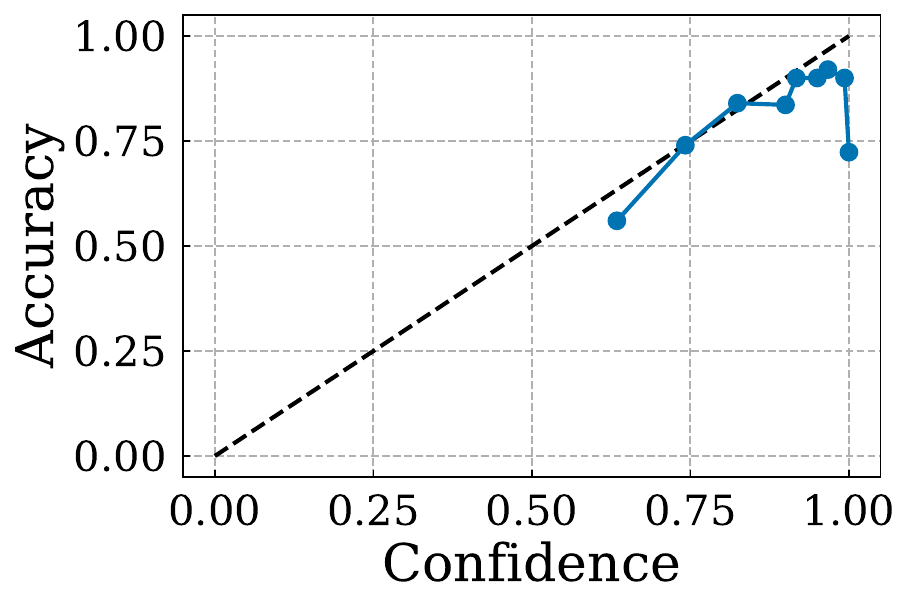}
        \caption{GPT-4 Summary ICL 8-Shot. ECE: $0.119$}
        \vspace{1em}
    \end{subfigure}
    \begin{subfigure}{0.45\textwidth}
        \includegraphics[width=\textwidth]{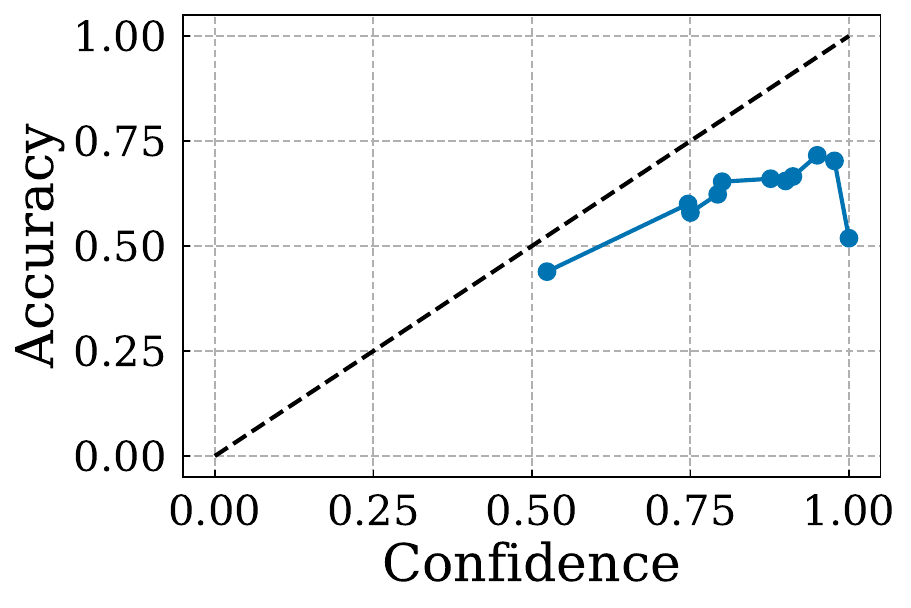}
        \caption{Summary ICL 8-Shot. ECE: $0.257$}
    \end{subfigure}
    \begin{subfigure}{0.45\textwidth}
        \includegraphics[width=\textwidth]{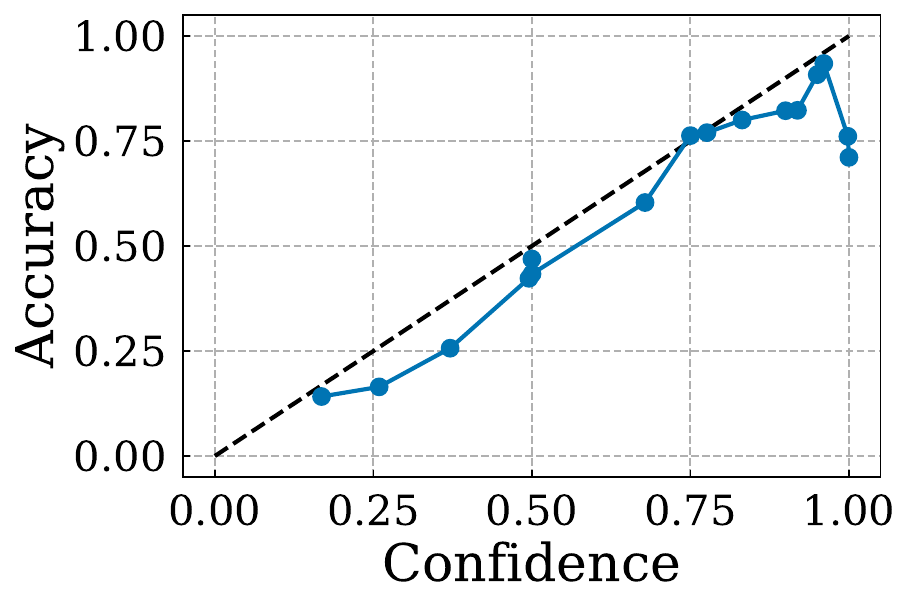}
        \caption{Direct Summary Eval. ECE: $0.112$}
    \end{subfigure}
    \caption{\textbf{TriviaQA Reliability Diagrams} for additional baselines and oracle methods, simulated reader. Notably, Summary ICL works well for both Llama 2 7B (Summary ICL 8-Shot) and GPT-4. The oracle baseline of Direct Summary Eval has strong ECE.}
\end{figure}

\newpage
\subsection{Jeopardy: Full Accuracy-ECE Frontier}

\begin{figure*}[ht]
    \centering
    \includegraphics[height=0.07\textwidth]{fig/fig_appendix_v2/frontier_legend/legend_Our_Method.pdf}
    \includegraphics[width=0.9\textwidth]{fig/fig_appendix_v2/frontier_legend/legend_Confidence_Baselines.pdf}
    \includegraphics[width=0.7\textwidth]{fig/fig_appendix_v2/frontier_legend/legend_Non-Confidence_Baselines.pdf}
    \includegraphics[width=\textwidth]{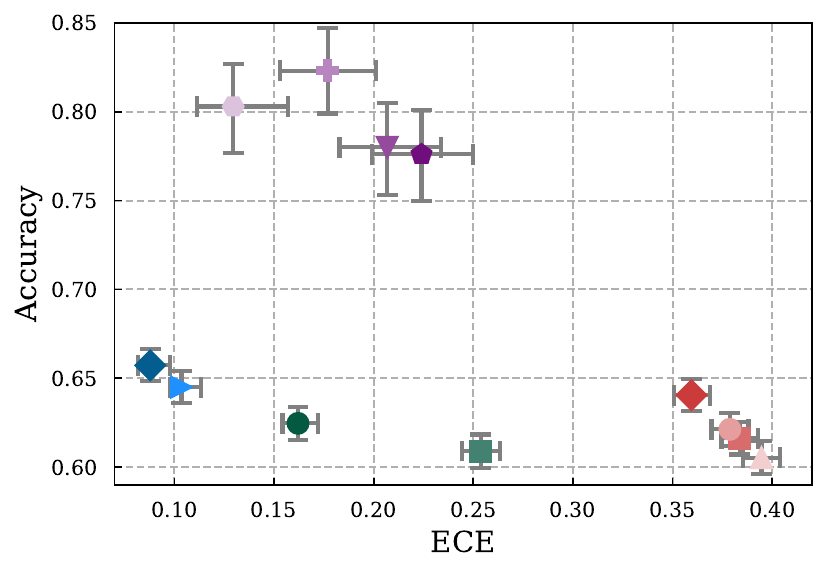}
    \caption{\textbf{Accuracy-ECE Frontier for All Methods on Jeopardy,} evaluated with simulated reader (upper left is better). Results are similar to TriviaQA. 
    GPT-4 Summary ICL works well despite its simplicity. 
    \LC RL outperforms all methods in ECE, including GPT-4 Summary ICL and Direct Summary Eval. 
    Llama 2 Chat 7B performs significantly worse than all other methods and is not displayed for clarity (cf. Table~\ref{tab:metrics_qa} for numerical results).}
    \label{fig:appendix_jeopardy_frontier}
\end{figure*}
\newpage
\subsection{Jeopardy: All Reliability Diagrams}

\begin{figure}[h]
    \captionsetup[subfigure]{justification=centering,aboveskip=-0.5pt,font=normalsize}
    \centering
    \begin{subfigure}{0.45\textwidth}
        \includegraphics[width=\textwidth]{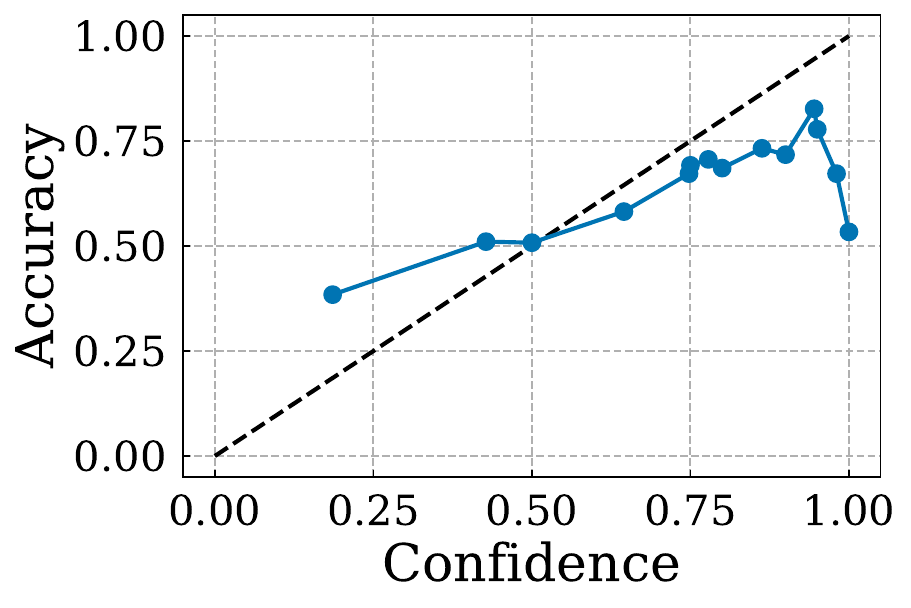}
        \caption{\LC SFT ($\pi_{\mathtt{SFT}})$. ECE: $0.162$}
        \vspace{0.75em}
    \end{subfigure}
    \begin{subfigure}{0.45\textwidth}
        \includegraphics[width=\textwidth]{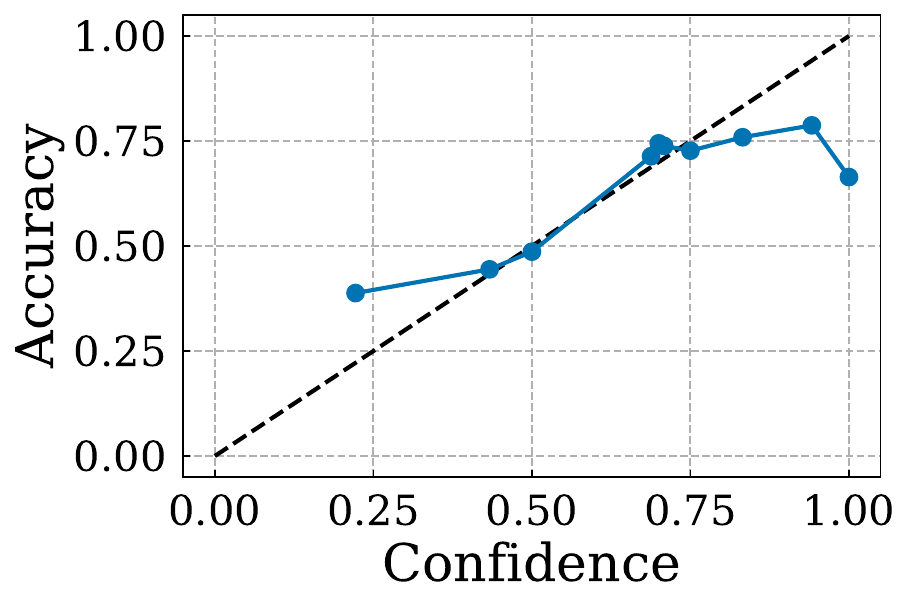}
        \caption{\LC RL ($\pi_{\mathtt{RL}})$. ECE: $0.088$}
        \vspace{0.75em}
    \end{subfigure}
    \begin{subfigure}{0.45\textwidth}
        \includegraphics[width=\textwidth]{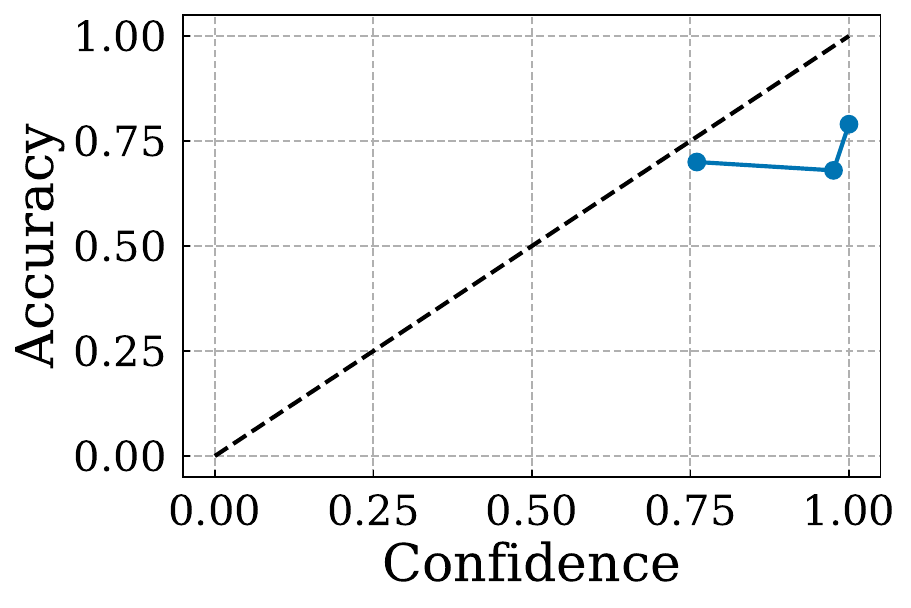}
        \caption{GPT-4 JAFU 0-Shot. ECE: $0.207$}
        \vspace{0.75em}
    \end{subfigure}
    \begin{subfigure}{0.45\textwidth}
        \includegraphics[width=\textwidth]{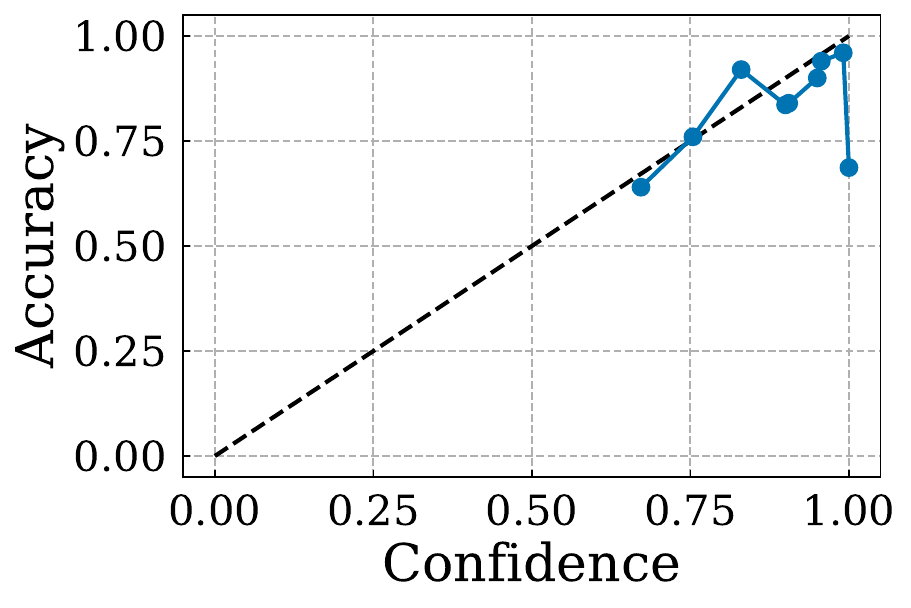}
        \caption{GPT-4 Summary ICL 8-Shot. ECE: $0.130$}
        \vspace{0.75em}
    \end{subfigure}
    \begin{subfigure}{0.45\textwidth}
        \includegraphics[width=\textwidth]{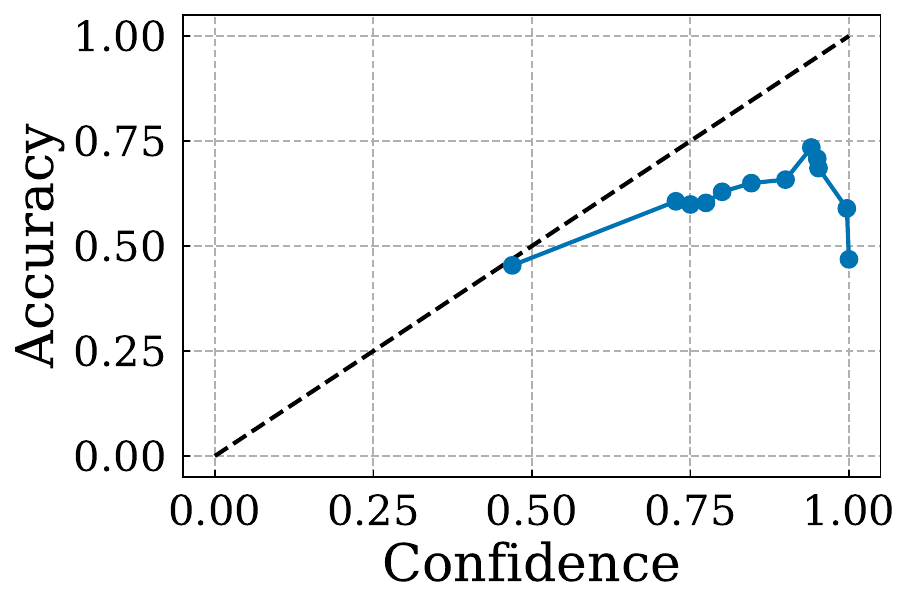}
        \caption{Summary ICL 8-Shot. ECE: $0.254$}
    \end{subfigure}
    \begin{subfigure}{0.45\textwidth}
        \includegraphics[width=\textwidth]{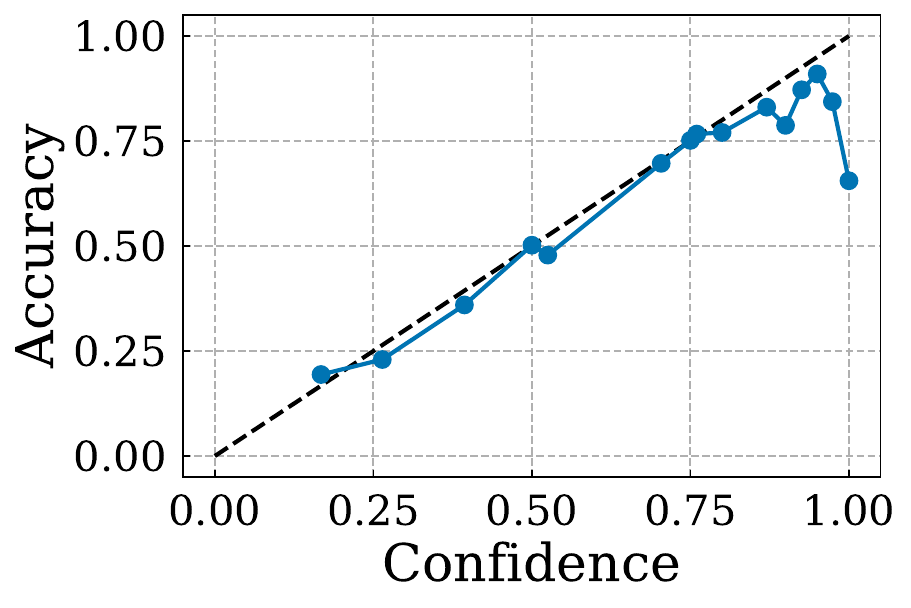}
        \caption{Direct Summary Eval. ECE: $0.104$}
    \end{subfigure}
    \caption{\textbf{Jeopardy Reliability Diagrams} for all methods, simulated reader. \LC RL has better ECE than all other methods while predicting a variety of confidence levels. Notably, Summary ICL works well for both Llama 2 7B (Summary ICL 8-Shot) and GPT-4. The oracle baseline of Direct Summary Eval has strong ECE.}
\end{figure}

\newpage
\subsection{SciQ: Full Accuracy-ECE Frontier}

\begin{figure*}[ht]
    \centering
    \includegraphics[height=0.07\textwidth]{fig/fig_appendix_v2/frontier_legend/legend_Our_Method.pdf}
    \includegraphics[width=0.9\textwidth]{fig/fig_appendix_v2/frontier_legend/legend_Confidence_Baselines.pdf}
    \includegraphics[width=0.7\textwidth]{fig/fig_appendix_v2/frontier_legend/legend_Non-Confidence_Baselines.pdf}
    \includegraphics[width=0.9\textwidth]{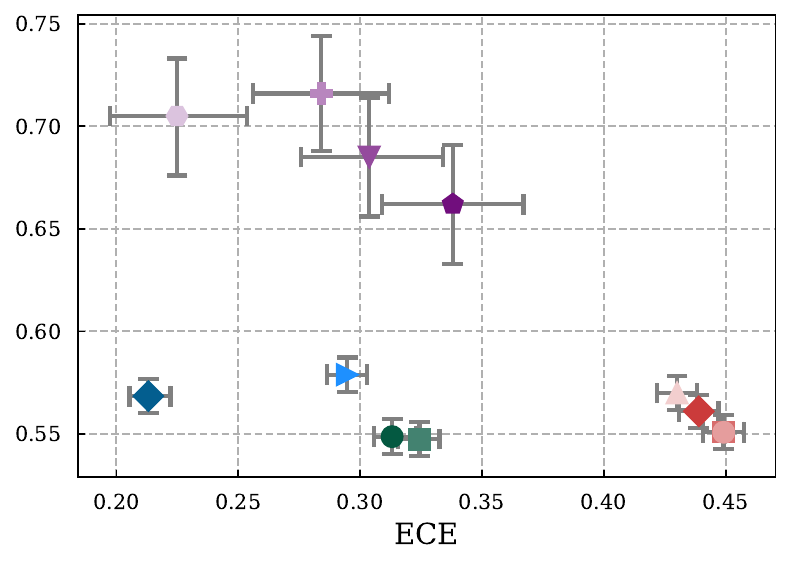}
    \caption{\textbf{Accuracy-ECE Frontier for All Methods on SciQ,} evaluated with simulated reader (upper left is better). 
    GPT-4 Summary ICL works well despite its simplicity. 
    \LC RL outperforms all methods in ECE, including GPT-4 Summary ICL and Direct Summary Eval. 
    Llama 2 Chat 7B performs significantly worse than all other methods and is not displayed for clarity (cf. Table~\ref{tab:metrics_qa} for numerical results).}
    \label{fig:appendix_sciq_frontier}
\end{figure*}

\newpage
\subsection{SciQ: All Reliability Diagrams}

\begin{figure}[h]
    \captionsetup[subfigure]{justification=centering,aboveskip=-0.5pt,font=normalsize}
    \centering
    \begin{subfigure}{0.45\textwidth}
        \includegraphics[width=\textwidth]{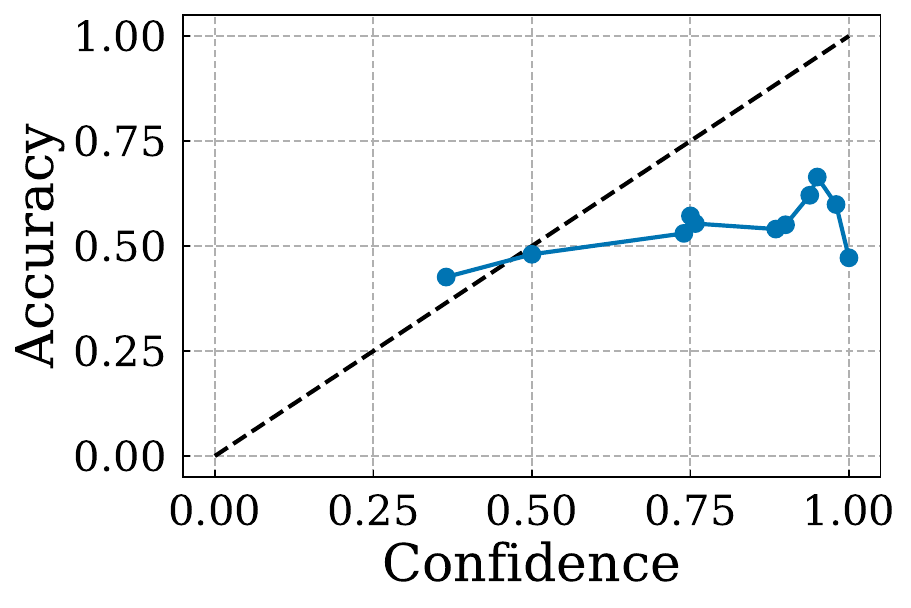}
        \caption{\LC SFT ($\pi_{\mathtt{SFT}})$. ECE: $0.313$}
        \vspace{0.75em}
    \end{subfigure}
    \begin{subfigure}{0.45\textwidth}
        \includegraphics[width=\textwidth]{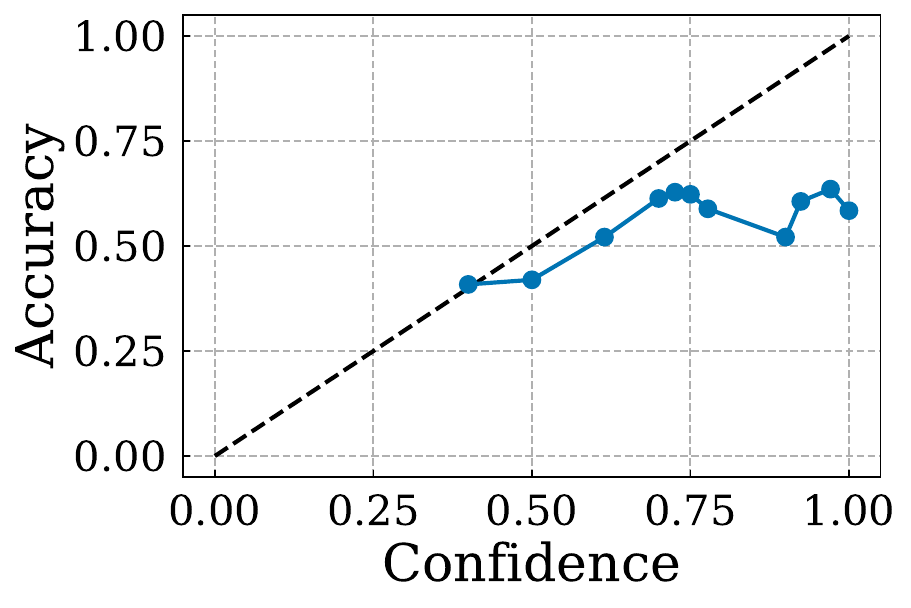}
        \caption{\LC RL ($\pi_{\mathtt{RL}})$. ECE: $0.213$}
        \vspace{0.75em}
    \end{subfigure}
    \begin{subfigure}{0.45\textwidth}
        \includegraphics[width=\textwidth]{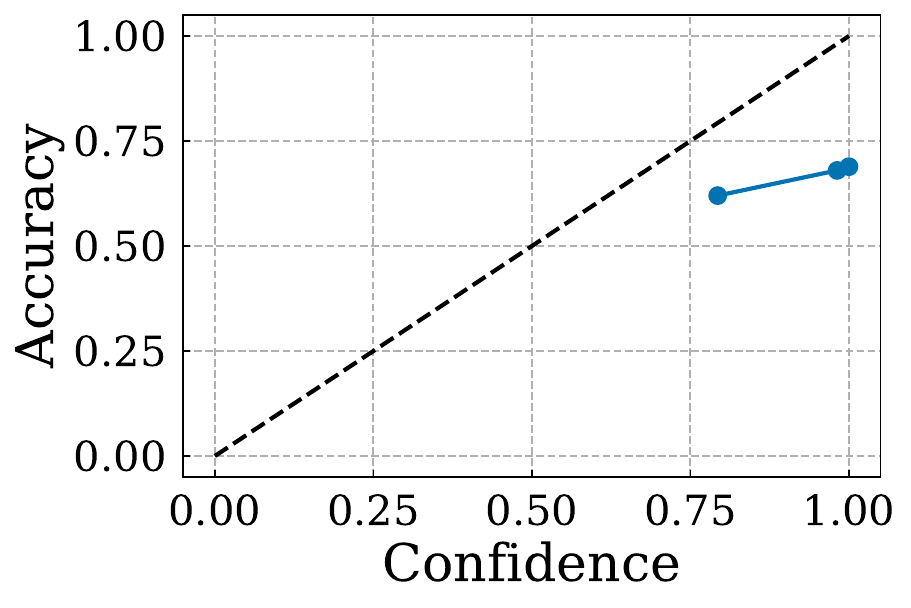}
        \caption{GPT-4 JAFU 0-Shot. ECE: $0.304$}
        \vspace{0.75em}
    \end{subfigure}
    \begin{subfigure}{0.45\textwidth}
        \includegraphics[width=\textwidth]{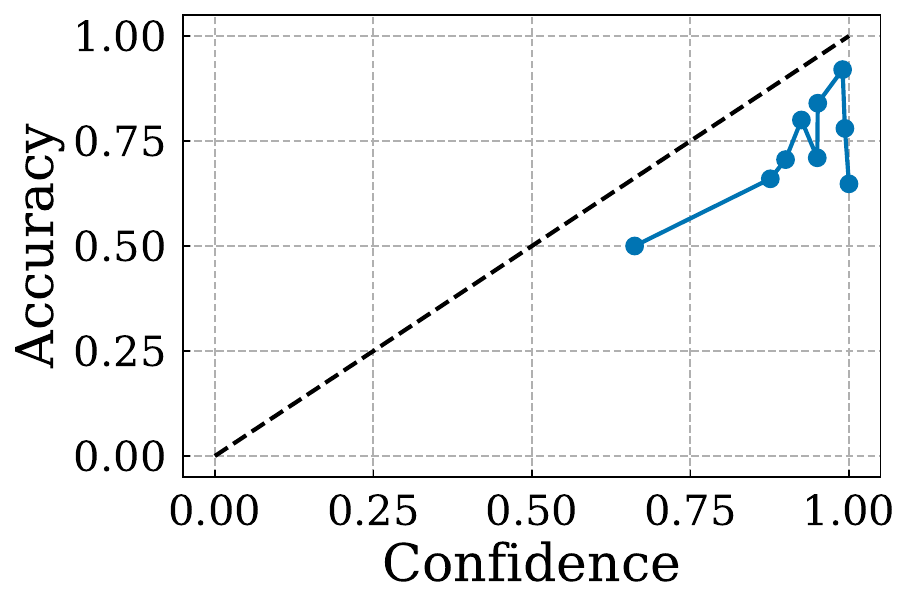}
        \caption{GPT-4 Summary ICL 8-Shot. ECE: $0.225$}
        \vspace{0.75em}
    \end{subfigure}
    \begin{subfigure}{0.45\textwidth}
        \includegraphics[width=\textwidth]{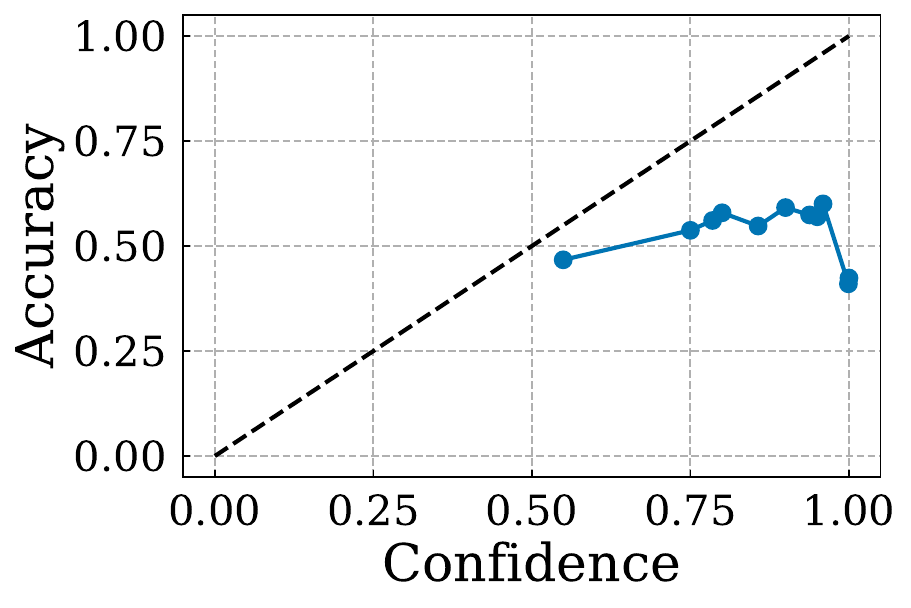}
        \caption{Summary ICL 8-Shot. ECE: $0.324$}
    \end{subfigure}
    \begin{subfigure}{0.45\textwidth}
        \includegraphics[width=\textwidth]{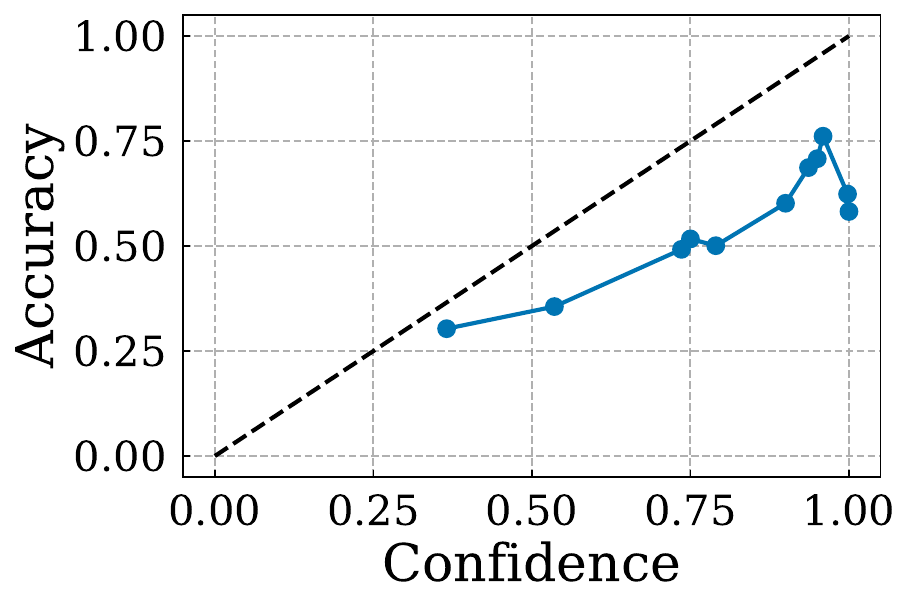}
        \caption{Direct Summary Eval. ECE: $0.295$}
    \end{subfigure}
    \caption{\textbf{SciQ Reliability Diagrams} for all methods, simulated reader. 
    \LC RL has better ECE than all other methods while predicting a variety of confidence levels. 
    }
\end{figure}

\newpage
\subsection{BioASQ Task B: Full Accuracy-ECE Frontier}

\begin{figure*}[ht]
    \centering
    \includegraphics[height=0.07\textwidth]{fig/fig_appendix_v2/frontier_legend/legend_Our_Method.pdf}
    \includegraphics[width=0.9\textwidth]{fig/fig_appendix_v2/frontier_legend/legend_Confidence_Baselines.pdf}
    \includegraphics[width=0.7\textwidth]{fig/fig_appendix_v2/frontier_legend/legend_Non-Confidence_Baselines.pdf}
    \includegraphics[width=0.9\textwidth]{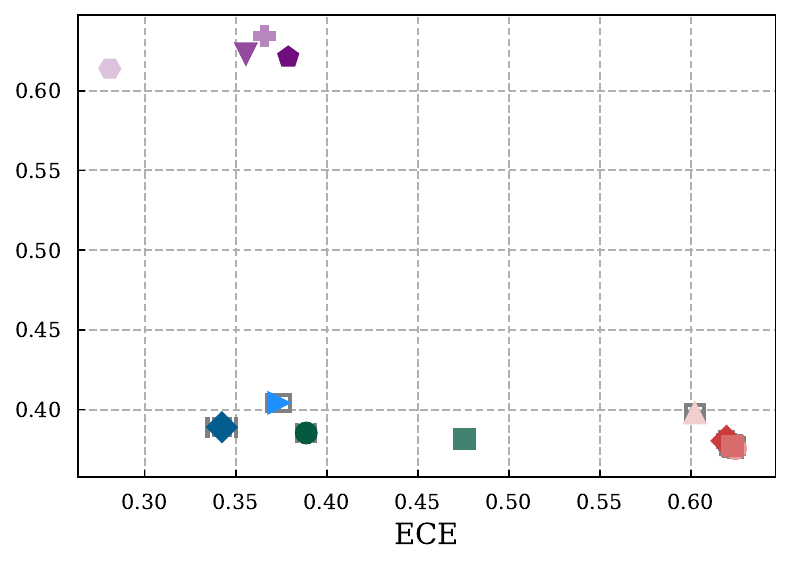}
    \caption{\textbf{Accuracy-ECE Frontier for All Methods on BioASQ Task B,} evaluated with simulated reader (upper left is better). 
    Under a significant distribution shift, \LC RL ECE is competitive with GPT-4--based methods and outperforms all Llama-based baselines.
    BioASQ is a small expert-annotated dataset, so we report Student's $t$ 0.95 CIs over 10 evaluation seeds for non-GPT-4 methods.
    Llama 2 Chat 7B performs significantly worse than all other methods and is not displayed for clarity (cf. Table~\ref{tab:metrics_qa} for all numerical results).
    }
    \label{fig:appendix_bioasq_frontier}
\end{figure*}

\newpage
\subsection{BioASQ Task B: All Reliability Diagrams}

\begin{figure}[h]
    \captionsetup[subfigure]{justification=centering,aboveskip=-0.5pt,font=normalsize}
    \centering
    \begin{subfigure}{0.45\textwidth}
        \includegraphics[width=\textwidth]{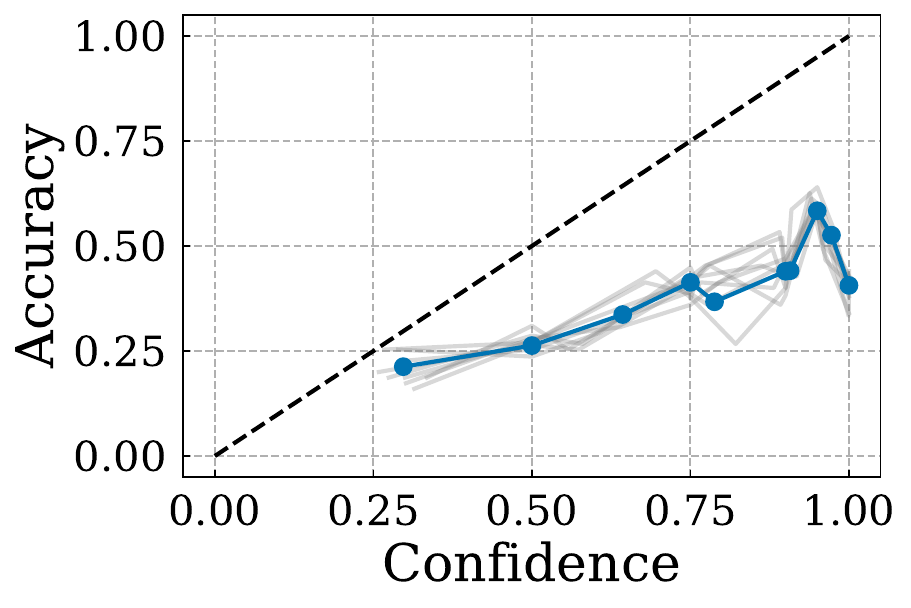}
        \caption{\LC SFT ($\pi_{\mathtt{SFT}})$. ECE: $0.390$}
        \vspace{0.5em}
    \end{subfigure}
    \begin{subfigure}{0.45\textwidth}
        \includegraphics[width=\textwidth]{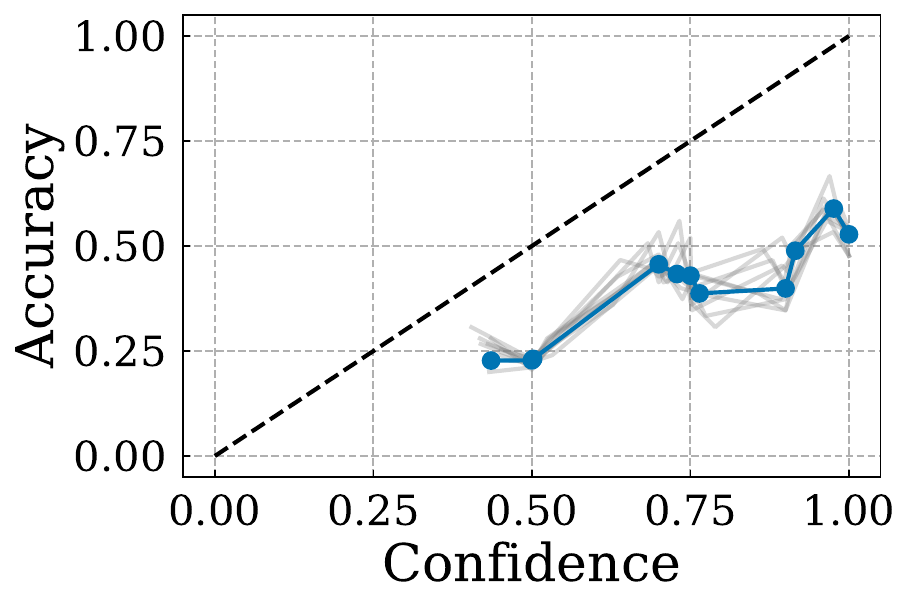}
        \caption{\LC RL ($\pi_{\mathtt{RL}})$. ECE: $0.342$}
        \vspace{0.5em}
    \end{subfigure}
    \begin{subfigure}{0.45\textwidth}
        \includegraphics[width=\textwidth]{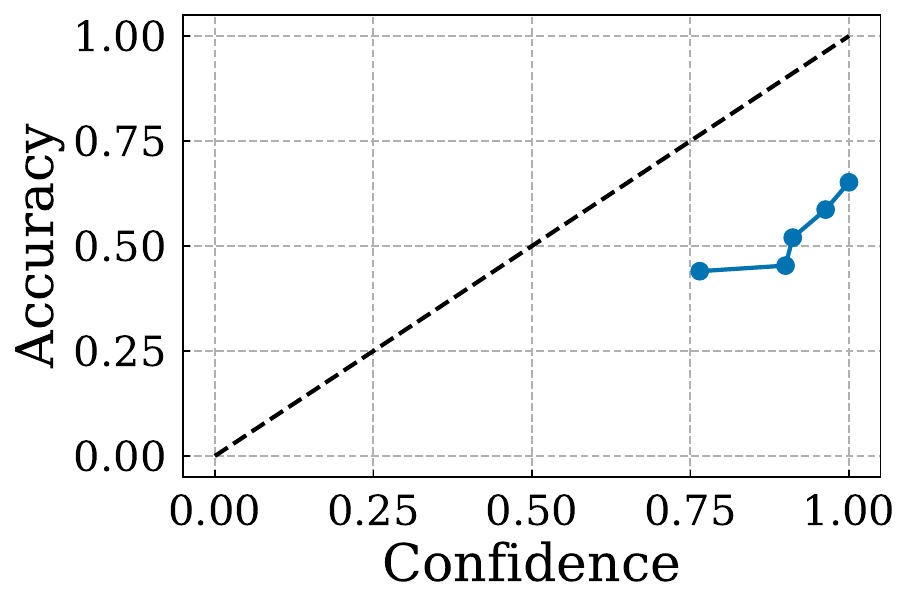}
        \caption{GPT-4 JAFU 0-Shot. ECE: $0.356$}
        \vspace{0.5em}
    \end{subfigure}
    \begin{subfigure}{0.45\textwidth}
        \includegraphics[width=\textwidth]{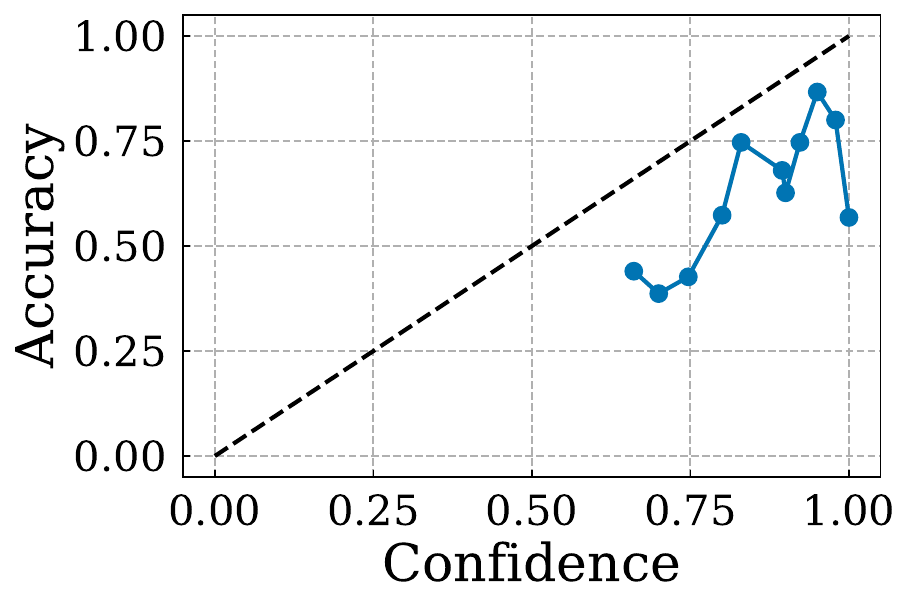}
        \caption{GPT-4 Summary ICL 8-Shot. ECE: $0.281$}
        \vspace{0.5em}
    \end{subfigure}
    \begin{subfigure}{0.45\textwidth}
        \includegraphics[width=\textwidth]{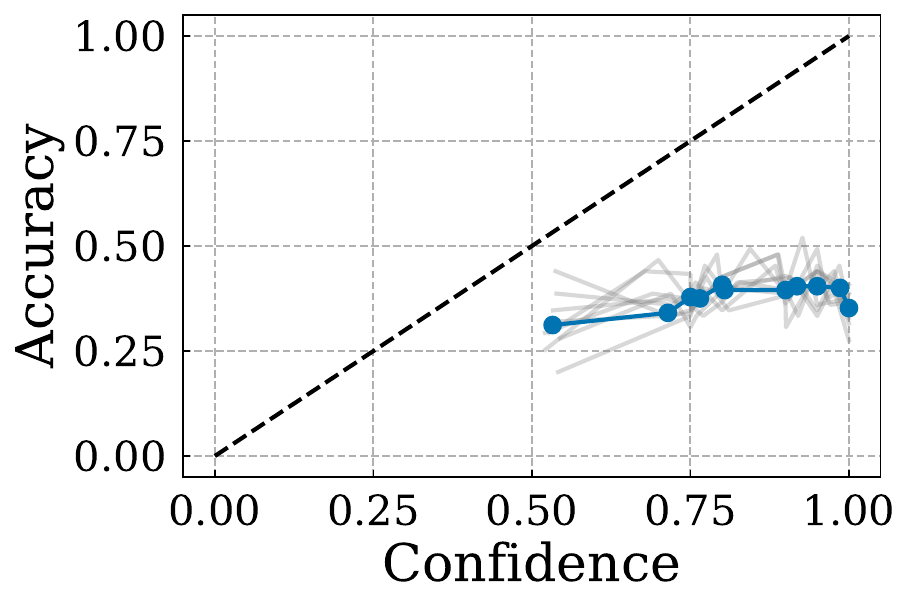}
        \caption{Summary ICL 8-Shot. ECE: $0.477$}
    \end{subfigure}
    \begin{subfigure}{0.45\textwidth}
        \includegraphics[width=\textwidth]{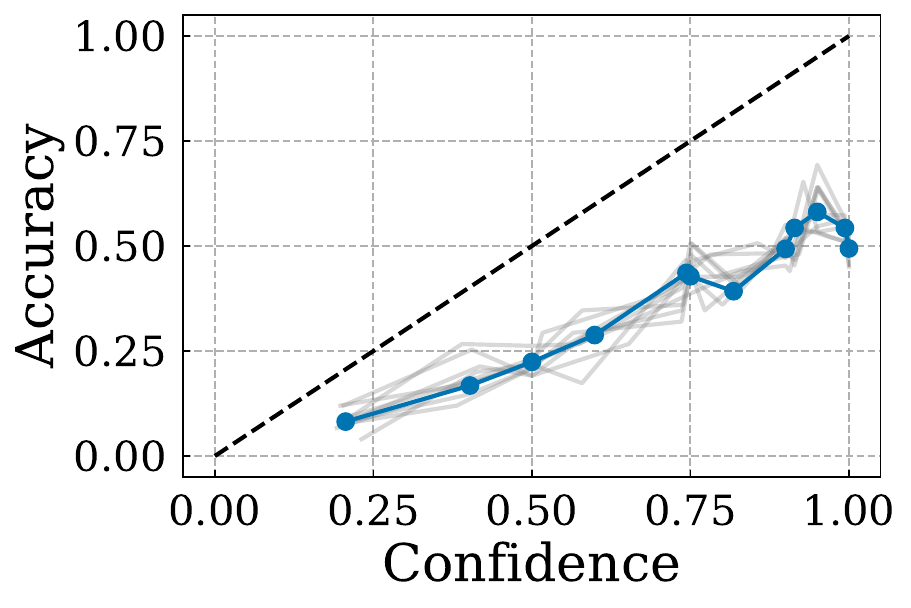}
        \caption{Direct Summary Eval. ECE: $0.375$}
    \end{subfigure}
    \caption{\textbf{BioASQ Task B Reliability Diagrams} for all methods, simulated reader. 
    Under significant distribution shift, \LC RL has ECE competitive with GPT-4 methods while predicting a variety of confidence levels.
    BioASQ is a small expert-annotated dataset, so for non-GPT-4 methods, we pool examples across 10 evaluation seeds when plotting the reliability diagram (per-seed plots in gray).
    }
\end{figure}

\newpage
\subsection{Person Biography Generation Frontier}

\begin{figure*}[h]
    \centering
    \includegraphics[height=0.07\textwidth]{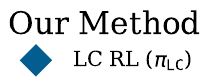}
    \includegraphics[height=0.07\textwidth]{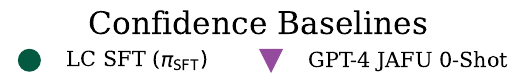}
    \includegraphics[width=0.62\textwidth]{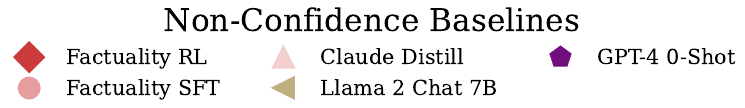}
    \includegraphics[width=\textwidth]{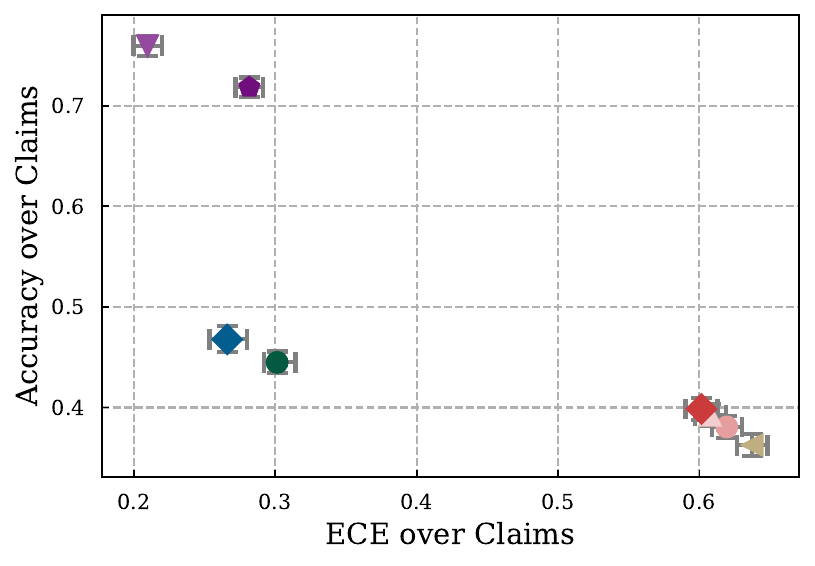}
    \caption{\textbf{Accuracy-ECE Frontier for All Methods on Person Biography Generation,} evaluated with simulated reader (upper left is better). GPT-4 is a strong baseline, and GPT-4 Just Ask for Uncertainty (JAFU) does improve both accuracy and ECE. Notably, \LC RL has better ECE than GPT-4 0-Shot, despite being trained from a Llama 2 7B base model.}
    \label{fig:appendix_factscore_frontier}
\end{figure*}
\newpage
\FloatBarrier

\newpage
\thispagestyle{empty}
\begin{landscape}
\subsection{Tabular Results}

\input{tables/question_answering_results}

\end{landscape}

\newpage
\input{tables/human_question_answering_results}
\input{tables/biography_generation_results}

\newpage
\subsection{Qualitative Examples}
\label{app:qualitative_examples}

We randomly sample two examples from each evaluation dataset and compare Factuality RL with \LC RL. 
We manually fact-check each generation using Wikipedia and other Google-accessible sources and highlight incorrect statements in \hlred{red}.
Numerical and linguistic confidence statements are highlighted in \hlblue{blue}.

\FloatBarrier

\begin{figure}[h]
\hrulefill
\newline
\newline
{\small
\textbf{Dataset:} TriviaQA\\
\textbf{Query $q$:} Write a paragraph about the classic book in which a boy hides in a ship's barrel of apples.\\
\textbf{Ground-Truth Answer $y$:} Treasure Island
\begin{qualitativeBox}
    \textbf{Factuality RL: } The classic book in which a boy hides in a ship's barrel of apples is "Treasure Island," written by Robert Louis Stevenson. The story follows the adventures of Jim Hawkins, a young boy who lives in a inn \hlred{on an island off the coast of England}. One day, a mysterious sailor named Long John Silver \hlred{arrives at the inn}, and Jim becomes involved in a treasure hunt for a buried treasure \hlred{on the island}. Jim hides in a barrel of apples on the ship \hlred{that is carrying the treasure}, and \hlred{he is able to escape with the treasure when the ship is attacked by pirates}. "Treasure Island" is a classic tale of adventure and intrigue, and it has been adapted into numerous films and stage productions.
\end{qualitativeBox}
\begin{qualitativeBox}
    \textbf{\LC RL: } \hlblue{Based on my understanding,} the classic book in which a boy hides in a ship's barrel of apples is \hlblue{most likely} "Treasure Island" by Robert Louis Stevenson. I \hlblue{estimate there is a 75\% chance} it is "Treasure Island", a \hlblue{15\% chance} it is \hlred{"The Wonderful Wizard of Oz" by L. Frank Baum}, a \hlblue{5\% chance} it is \hlred{"The Adventures of Tom Sawyer" by Mark Twain}, and a \hlblue{5\% chance} it is \hlred{another book}. I \hlblue{believe} the story involves a young boy named Jim Hawkins who hides in a barrel of apples \hlred{to avoid being discovered by pirates}. \hlblue{Though I do not know all the details, I am confident} the boy's adventures involve searching for buried treasure. \hlblue{Regardless of the specific book, I am certain} this classic tale has become beloved by generations of readers.
\end{qualitativeBox}
}
\label{fig:triviaqa_qualitative_1}
\end{figure}

\begin{figure}[h]
\hrulefill
\newline
\newline
{\small
\textbf{Dataset:} TriviaQA\\
\textbf{Query $q$:} Write a paragraph about the type of creature that a tanager is.\\
\textbf{Ground-Truth Answer $y$:} Bird
\begin{qualitativeBox}
    \textbf{Factuality RL: } A tanager is a type of bird that belongs to the Thraupidae family. Tanagers are known for their brightly colored plumage, with many species featuring vibrant red, yellow, and blue feathers. They are found in a wide range of habitats, including forests, savannas, and grasslands, and are known for their acrobatic flight patterns. Tanagers are social birds that often form flocks, and they are known for their loud and distinctive calls.
\end{qualitativeBox}
\begin{qualitativeBox}
    \textbf{\LC RL: } I \hlblue{believe} a tanager is a type of bird that belongs to the Thraupidae family. I am \hlblue{confident} they are found in the Americas, \hlblue{though I am uncertain of their exact range}. I \hlblue{estimate with 75\% likelihood} that tanagers are brightly colored birds with distinctive plumage. I am \hlblue{50\% sure} they are known for their songbird abilities. I am \hlblue{uncertain} whether they are found in forests, grasslands, or other habitats. I \hlblue{believe with 75\% likelihood} that tanagers feed on insects and fruit. \hlblue{While the details of their appearance and behavior are unclear to me}, I am \hlblue{certain} that tanagers are colorful, songbirds native to the Americas.
\end{qualitativeBox}
}
\label{fig:triviaqa_qualitative_2}
\end{figure}

\begin{figure}[h]
\hrulefill
\newline
\newline
{\small
\textbf{Dataset:} SciQ\\
\textbf{Query $q$:} Write a paragraph about the term used to describe elements that specifically are noted to have properties of both metals and nonmetals.\\
\textbf{Ground-Truth Answer $y$:} Metalloids
\begin{qualitativeBox}
    \textbf{Factuality RL: } The term used to describe elements that specifically are noted to have properties of both metals and nonmetals is "metalloids." Metalloids are elements that have some of the properties of metals, such as conductivity and \hlred{malleability}, and some of the properties of nonmetals, such as a tendency to form covalent bonds. Examples of metalloids include silicon, germanium, and arsenic.
\end{qualitativeBox}
\begin{qualitativeBox}
    \textbf{\LC RL: } \hlblue{Based on my understanding}, the term used to describe elements that specifically are noted to have properties of both metals and nonmetals is either "metalloids", \hlred{"semimetals"}, or \hlred{"semiconductors"}. \hlblue{I estimate there is a 70\% likelihood} that the term is "metalloids", a \hlblue{20\% chance} it is \hlred{"semimetals"}, and a \hlblue{10\% chance} it is \hlred{"semiconductors"}. I \hlblue{believe} these elements have properties that are somewhere between metals and nonmetals, \hlblue{though I do not know all the details}. I am \hlblue{confident} that they are used in various applications like electronics, batteries, and semiconductors. \hlblue{However, I am uncertain} about the specific elements that fall into this category. Overall, \hlblue{despite my lack of complete knowledge, I am sure} that there is a term used to describe elements with properties of both metals and nonmetals.
\end{qualitativeBox}
}
\label{fig:sciq_qualitative_1}
\end{figure}

\begin{figure}[h]
\hrulefill
\newline
\newline
{\small
\textbf{Dataset:} SciQ\\
\textbf{Query $q$:} Write a paragraph about the periods when most plants do not grow continuously.\\
\textbf{Ground-Truth Answer $y$:} Dormant
\begin{qualitativeBox}
    \textbf{Factuality RL: } The periods when most plants do not grow continuously are known as dormancy periods. During these periods, plants enter a state of inactivity, where they stop growing and may even shed leaves or go into a state of \hlred{hibernation}. Dormancy periods can be caused by various factors, such as cold temperatures, lack of sunlight, or insufficient water. Plants may also enter dormancy periods to conserve energy and resources during times of stress or drought. Dormancy periods are an important part of the plant's life cycle, allowing them to adapt to changing environmental conditions and survive through challenging times.
\end{qualitativeBox}
\begin{qualitativeBox}
    \textbf{\LC RL: } \hlblue{Based on my understanding}, the periods when most plants do not grow continuously are known as dormant periods. I am \hlblue{confident} that dormancy occurs in response to environmental factors like cold temperatures, lack of sunlight, or limited water availability. \hlblue{However, I am uncertain} about the exact duration of dormancy for different plant species. \hlblue{I estimate there is a 50\% chance} that dormancy lasts for several months, and a \hlblue{50\% chance} it \hlred{lasts for several years}. I \hlblue{believe} dormancy allows plants to conserve energy and resources, \hlblue{though I do not know all the details}. While dormancy may seem like a state of inactivity, I am \hlblue{certain} it is an important adaptive strategy for plants to survive harsh conditions. Overall, \hlblue{despite my lack of complete knowledge, I am confident} that dormancy is a natural part of plant growth cycles.
\end{qualitativeBox}
}
\label{fig:sciq_qualitative_2}
\end{figure}

\begin{figure}[h]
\hrulefill
\newline
\newline
{\small
\textbf{Dataset:} BioASQ Task B\\
\textbf{Query $q$:} Write a paragraph about the gene that is most commonly mutated in Tay-Sachs disease.\\
\textbf{Ground-Truth Answer $y$:} HEXA
\begin{qualitativeBox}
    \textbf{Factuality RL: } The gene that is most commonly mutated in Tay-Sachs disease is the HEXA gene, which encodes the alpha subunit of the hexosaminidase A enzyme. This enzyme is responsible for breaking down the sugar hexosamine, which is found in the membranes of cells. When the HEXA gene is mutated, the alpha subunit of the hexosaminidase A enzyme is \hlred{not produced}, leading to a buildup of \hlred{hexosamine} in the cells. This buildup can cause neurological damage and eventually lead to death. Tay-Sachs disease is a rare genetic disorder that \hlred{primarily} affects Ashkenazi Jews and is caused by a mutation in the HEXA gene.
\end{qualitativeBox}
\begin{qualitativeBox}
    \textbf{\LC RL: } \hlblue{Based on my understanding}, the gene that is most commonly mutated in Tay-Sachs disease is either the HEXA gene or the \hlred{GAA} gene. \hlblue{I estimate there is a 70\% likelihood} that the HEXA gene is mutated, and a \hlblue{30\% chance} that the \hlred{GAA} gene is mutated. I \hlblue{believe} the HEXA gene encodes an enzyme called beta-hexosaminidase A, while the GAA gene encodes an enzyme called \hlred{beta-glucuronidase}. I am \hlblue{confident} that mutations in these genes lead to a deficiency of these enzymes, which causes accumulation of toxic substances in the nervous system and other tissues. This results in neurological symptoms like abnormal eye movements, seizures, and mental decline. \hlblue{While I do not know all the details, I believe} Tay-Sachs disease is a rare genetic disorder that \hlred{primarily} affects Ashkenazi Jews.
\end{qualitativeBox}
}
\label{fig:bioasq_qualitative_1}
\end{figure}

\begin{figure}[h]
\hrulefill
\newline
\newline
{\small
\textbf{Dataset:} BioASQ Task B\\
\textbf{Query $q$:} Write a paragraph about the type of GTPases that are required for amino acid-dependent activation of mTORC1.\\
\textbf{Ground-Truth Answer $y$:} Heterodimeric Rag GTPases
\begin{qualitativeBox}
    \textbf{Factuality RL: } The type of GTPases that are required for amino acid-dependent activation of mTORC1 are \hlred{rapamycin-sensitive} GTPases. These GTPases are involved in the regulation of cell growth and proliferation, and are activated by amino acids, such as leucine, isoleucine, and valine. The activation of these GTPases leads to the phosphorylation of the mTORC1 complex, which in turn activates downstream signaling pathways that regulate cell growth and proliferation. The activation of mTORC1 by amino acids is a crucial step in the regulation of cell growth and proliferation, and is a key target for the development of therapeutics for various diseases, including cancer and diabetes.
\end{qualitativeBox}
\begin{qualitativeBox}
    \textbf{\LC RL: } \hlblue{Based on my understanding}, the type of GTPases required for amino acid-dependent activation of mTORC1 are either \hlred{mTORC1 itself}, Rheb, or Rag GTPases. \hlblue{I estimate there is a 70\% likelihood} that \hlred{mTORC1 itself} is the GTPase required, a \hlblue{20\% chance} that Rheb is the GTPase required, and a \hlblue{10\% chance} that Rag GTPases are the GTPases required. I \hlblue{believe} the activation of mTORC1 by amino acids involves the \hlred{binding} of amino acids to Rag GTPases, \hlred{which then activates mTORC1}. \hlblue{However, I am uncertain about the exact details} of how this process works. \hlblue{Regardless of the specific GTPase involved, I am confident} that amino acid activation of mTORC1 plays a crucial role in regulating cell growth, metabolism, and protein synthesis.
\end{qualitativeBox}
}
\label{fig:bioasq_qualitative_2}
\end{figure}

\begin{figure}[h]
\hrulefill
\newline
\newline
{\small
\textbf{Dataset:} Jeopardy\\
\textbf{Query $q$:} Write a paragraph about what Hiram Percy Maxim invented for guns, in addition to inventing the muffler for cars.\\
\textbf{Ground-Truth Answer $y$:} A silencer
\begin{qualitativeBox}
    \textbf{Factuality RL: } Hiram Percy Maxim, an American inventor and entrepreneur, is best known for his contributions to the firearms industry. In addition to inventing the muffler for cars, Maxim also \hlred{invented the Maxim machine gun}, a rapid-fire weapon that revolutionized warfare in the early 20th century. The Maxim machine gun was the first fully automatic weapon to be widely used in combat, and it played a significant role in the Allied victory in World War I. \hlred{Maxim's invention of the Maxim machine gun earned him the nickname "the father of the machine gun."}
\end{qualitativeBox}
\begin{qualitativeBox}
    \textbf{\LC RL: } I \hlblue{believe} that Hiram Percy Maxim, an American inventor, invented at least one significant invention related to guns in addition to inventing the muffler for cars. \hlblue{Based on my understanding, I estimate there is a 50\% chance} he invented the \hlred{Maxim gun}, a \hlblue{25\% chance} he invented the silencer, and a \hlblue{25\% chance} he invented the \hlred{machine gun}. I am \hlblue{confident} that regardless of the specific invention, Maxim's work revolutionized the firearms industry and had a significant impact on warfare. \hlblue{Though I do not know all the details, I believe} his inventions were known for their efficiency, accuracy, and durability. \hlblue{While the specifics are unclear to me, I am certain} that Hiram Percy Maxim was an influential American inventor who made important contributions to the development of firearms technology.
\end{qualitativeBox}
}
\label{fig:jeopardy_qualitative_1}
\end{figure}

\begin{figure}[h]
\hrulefill
\newline
\newline
{\small
\textbf{Dataset:} Jeopardy\\
\textbf{Query $q$:} Write a paragraph about the planet that a probe from the Galileo spacecraft plunged into the atmosphere of on December 7, 1995.\\
\textbf{Ground-Truth Answer $y$:} Jupiter
\begin{qualitativeBox}
    \textbf{Factuality RL: } The planet that a probe from the Galileo spacecraft plunged into the atmosphere of on December 7, 1995, was Jupiter. The Galileo probe was launched by NASA in 1989 and arrived at Jupiter in 1995, where it conducted numerous scientific studies of the planet and its moons. On \hlred{December 7, 1995}, the probe was intentionally crashed into Jupiter's \hlred{atmosphere, where it collected data on the planet's atmosphere and magnetic field}. The Galileo probe's mission was a significant achievement in the study of Jupiter and the solar system, and its impact on the field of planetary science is still felt today.
\end{qualitativeBox}
\begin{qualitativeBox}
    \textbf{\LC RL: } I am \hlblue{certain} that on December 7, 1995, a probe from the Galileo spacecraft plunged into the atmosphere of Jupiter. I \hlblue{believe} this was a major milestone in space exploration, as it was the first time a spacecraft had been intentionally sent to study Jupiter up close. I \hlblue{estimate with 90\% confidence} that the probe was named Galileo after the Italian astronomer Galileo Galilei. I am \hlblue{fairly confident} that the probe carried instruments to study Jupiter's atmosphere, \hlred{magnetic field}, and \hlred{moons}. \hlblue{Though I do not know all the details, I believe} the data collected by the probe helped advance our understanding of Jupiter \hlred{and its moons}. \hlblue{While the specifics are unclear to me, I am certain} that the Galileo probe's plunge into Jupiter's atmosphere marked an important achievement in space exploration.
\end{qualitativeBox}
}
\label{fig:jeopardy_qualitative_2}
\end{figure}

\begin{figure}[h]
\hrulefill
\newline
\newline
{\small
\textbf{Dataset:} Person Biography Generation\\
\textbf{Query $q$:} Write a paragraph bio about Griselda Blanco.\\
\begin{qualitativeBox}
    \textbf{Factuality RL: } Griselda Blanco, also known as "The Black Widow," was a Colombian drug lord and crime boss who operated in the 1970s and 80s. She was known for her ruthlessness and her involvement in the cocaine trade, which made her one of the most powerful drug lords in the world. Blanco was responsible for the murder of numerous rivals and enemies, and \hlred{her criminal activities led to her being listed as one of the FBI's Ten Most Wanted Fugitives in 1985}. Despite her criminal activities, Blanco was also \hlred{known for her philanthropy}, and she was \hlred{involved in various charitable organizations in Colombia}.
\end{qualitativeBox}
\begin{qualitativeBox}
    \textbf{\LC RL: } I \hlblue{believe} Griselda Blanco was a Colombian drug lord who operated in the 1970s and 1980s. I \hlblue{estimate with 75\% confidence} that she was known as the "Black Widow" due to her ruthlessness and alleged involvement in numerous murders. I am \hlblue{fairly certain} she was involved in cocaine trafficking and organized crime, \hlblue{though I am uncertain about the specifics of her operations}. I \hlblue{believe with 50\% likelihood} that she was born in 1943 in Colombia. I am \hlblue{confident} she was arrested \hlred{multiple times} and imprisoned, \hlblue{though I do not know the details}. \hlblue{While the specifics are unclear to me, I am certain} that Griselda Blanco was a notorious drug lord who left a lasting legacy of violence and crime in Colombia and the United States.
\end{qualitativeBox}
}
\label{fig:factscore_qualitative_1}
\end{figure}

\begin{figure}[h]
\hrulefill
\newline
\newline
{\small
\textbf{Dataset:} Person Biography Generation\\
\textbf{Query $q$:} Write a paragraph bio about Uee.\\
\begin{qualitativeBox}
    \textbf{Factuality RL: } Uee, \hlred{also known as Lee Hye-young}, is a South Korean singer and actress. She is best known for her work in the K-pop girl group After School, where she served as the \hlred{main} vocalist and lead dancer. Uee has also appeared in several television dramas, including \hlred{"The Greatest Love," "The Greatest Love," and "The Greatest Love."} In addition to her musical and acting career, Uee is also known for her philanthropic work, including her \hlred{participation in the "2014 Samsung Hope for Children" charity event}.
\end{qualitativeBox}
\begin{qualitativeBox}
    \textbf{\LC RL: } I am \hlblue{confident} that Uee is a South Korean actress and singer. I \hlblue{estimate with 75\% likelihood} that she was born in 1988 in \hlred{Seoul}, South Korea. I \hlblue{believe with 50\% probability} that she made her acting debut in \hlred{2007} in the television drama \hlred{"Love Needs a Miracle."} I am \hlblue{uncertain} whether she has appeared in \hlred{films like "The Concubine" or "The King's Case Note."} I am \hlblue{confident} that she has been a member of the girl groups After School and \hlred{Orange Caramel}. I \hlblue{believe with 75\% likelihood} that she has won awards for her acting and \hlred{singing}, \hlblue{though I do not know the specifics}. Overall, \hlblue{despite my lack of knowledge on some details, I am certain} that Uee is a talented South Korean actress and singer who has had a successful career in the entertainment industry.
\end{qualitativeBox}
}
\label{fig:factscore_qualitative_2}
\end{figure}

\clearpage
\section{Benefits of Linguistic Calibration for Decision-making}
\label{app:decision_and_calibration}

In \S\ref{sec:prelim}, we describe the connection between calibration and decision-making \cite{zhao2021right, zhao2021calibrating} and how it motivates our objective for linguistic calibration: namely, that maximizing our objective implies zero expected regret for a downstream Bayes-optimal decision-maker.
Here, we expand on this guarantee and provide additional ones using the results of \citet{zhao2021calibrating}, which apply even if we cannot perfectly maximize the linguistic calibration objective.
In particular, we will see that forecasts fulfilling weaker notions of calibration still guarantee that decisions are optimal in a weaker sense.

Following the generative process described in \S\ref{subsec:linguistic_calibration}, we assume that the linguistic calibration objective is optimized over the real-world distribution of LM queries and related question-answer pairs $p(q, x, y)$.
In practice, we found that using off-the-shelf question-answering datasets to generate a surrogate distribution (cf. \S\ref{subsec:decision_task_data}) robustly improves an LM's ability to express confidence levels in text, including under distribution shift.
Therefore, the generalization properties of \LC may justify this assumption.
Moreover, we believe an exciting avenue for future work is to curate a more representative question-answer distribution (cf. \S\ref{sec:discussion}).
Lastly, recall that our generative process makes the following technical conditional independence assumption:
\begin{assumption}
    Under the ground-truth distribution, the answer $Y$ is independent of the open-ended LM query and LM response $(Q, Z)$, given the question $X$: $Y \perp (Q, Z) \mid X$.
    \label{assumption:gt}
\end{assumption}

\subsection{Review of the \LC Objective}
Our decision-based RL algorithm optimizes a strictly proper scoring rule $\Sc_\LC$ of user forecasts $f(x, z) \in \Delta^{|\mathcal{Y}|}$ with respect to an LM $\pi_{\RL}$ producing long-form generations $z \sim \pi_\RL(z \mid q)$ (restatement of Eq.~\ref{eq:rl_objective_fn} for convenience):
\begin{equation}
    \max_{\pi_\RL} \underset{(q, x, y) \sim p(q, x, y),~z \sim \pi_\RL(z \mid q)}{\E} \left[ \Sc_\LC \left(\forecast, y\right) \right].
    \label{eq:lc_rl_objective_app}
\end{equation}
Because $\Sc_\LC$ is strictly proper (cf. \S\ref{subsec:linguistic_calibration_training_objective}) and assuming a sufficiently flexible LM $\pi_\RL$ and reader function $f: \mathcal{X} \times \mathcal{Z} \rightarrow \Delta^{|\mathcal{Y}|}$, our training procedure maximizes Eq.~\ref{eq:lc_rl_objective_app} if and only if the user's forecast exactly equals the ground-truth conditional distribution: $f(x, z) = p(y \mid x)$.
Then, intuitively, making Bayes-optimal decisions according to the user forecast should be optimal in some sense.
\citet{zhao2021calibrating} make this notion of optimal decision-making concrete as \textbf{decision calibration}.
To be self-contained, we will provide the definition of decision calibration from \citet{zhao2021calibrating} below, and then use it to precisely describe the guarantees that linguistic calibration provides for decision-making.

\subsection{Decision Calibration}

\paragraph{Setup and notation.} To define decision calibration, we first introduce some notation closely following \citet{zhao2021calibrating}, \S2.
In the process of decision-making based on probabilistic forecasts (ignoring LMs and their generations $z$ for now), users receive a question $x' \in \mathcal{X}'$, forecast a possible answer to the question with a probabilistic forecaster $\f: \mathcal{X}' \rightarrow \DeltaAns$, and finally choose from a set of available actions $\mathcal{A}$ based on their loss function $L: \mathcal{A} \times \mathcal{Y} \rightarrow \R$ and their forecast $f(x') \in \DeltaAns$.

The process of choosing an action can be described by a \textbf{decision rule} which maps forecasts to actions $\delta: \DeltaAns \rightarrow \mathcal{A}$, where $\druleall$ is the set of all decision rules.
Additionally, define the set of all loss functions as $\Lall = \{L: \mathcal{A} \times \mathcal{Y} \rightarrow \R\}$.

We are interested in Bayes decision rules, \textit{i.e.,} the set of rules that are optimal for \textit{some} loss function.
More specifically, consider some particular loss function $L \in \Lall$.
Then its corresponding Bayes decision rule is
\begin{equation}
    \delta_L(f(x')) = \argmin_{a \in \mathcal{A}} \E_{\widehat{y} \sim f(x')} [L(a, \widehat{y})].
\end{equation}
For some subset $\mathcal{L} \subset \Lall$, we denote the corresponding Bayes decision rules as $\drulel := \{\delta_L \mid L \in \mathcal{L}\}.$

\paragraph{Decision calibration.}
\citet{zhao2021calibrating} defines \textbf{decision calibration} to formalize the following intuition: ideally, a decision-maker with loss function $L$ should be able to consider an arbitrary decision rule $\delta \in \druleall$ and compute the expected loss of using $\delta$ in decision-making, given a probabilistic forecaster $f$.

\begin{definition}[Decision Calibration, Definition 2 in \citet{zhao2021calibrating}]
For any set of loss functions $\mathcal{L} \subset \Lall$, we say that a forecaster $f$ is $\mathcal{L}$-decision calibrated (with respect to the ground-truth conditional distribution $p(y \mid \cdot)$) if for each $L \in \mathcal{L}$ and $\delta \in \drulel$,
\begin{equation}
    \E_{x' \sim p(x')} \E_{\widehat{y} \sim f(x')} \Big[ L\big(\delta(f(x')), \widehat{y} \big) \Big] = \E_{x' \sim p(x')} \E_{y \sim p(y \mid x')} \Big[ L \big(\delta(f(x')), y\big)\Big].
    \label{eq:decision_calibration}
\end{equation}
\label{defn:decision_calibration}
\end{definition}
Following the analysis of \citet{zhao2021calibrating} (\S3.1), the left-hand side of Eq.~\ref{eq:decision_calibration} \textit{simulates} the loss of taking decisions according to the loss $L$ and rule $\delta$ using data drawn from the forecast $\widehat{y} \sim f(x')$.
This simulated loss can be computed by a user without observing any ground-truth outcome $y \sim p(y \mid x')$.
The right-hand side is the true loss for using decision rule $\delta$ with loss $L$.
Therefore, decision calibration means that a forecaster $f$ can be used to accurately estimate the expected loss of a decision rule under the true data distribution.

\paragraph{Decision calibration with LM assistance.} With a few particular choices, we can apply this general definition to the setting of LM-assisted decision-making.
In our context, the relevant variables are the open-ended query to the LM $q$, the LM response $z$, and the related question-answer pair $(x, y)$.
These variables are distributed by the joint $p(q, x, y, z) := p(q, x, y) \pi_\RL(z \mid q)$. 
Write the marginal over questions and related LM responses as $p(x, z) = \sum_{q', y'} p(q', x, y', z)$.
In the notation of Definition~\ref{defn:decision_calibration}, we define the input $x' := (x, z)$.

To be concrete, we reiterate decision calibration in our context, where users also condition their forecasts on LM outputs.
\begin{definition}[Decision Calibration with LM Assistance]
For any set of loss functions $\mathcal{L} \subset \Lall$, we say that a reader function $\f: \mathcal{X} \times \mathcal{Z} \rightarrow \DeltaAns$ is $\mathcal{L}$-decision calibrated (with respect to the ground-truth conditional distribution $p(y \mid \cdot)$) if for each $L \in \mathcal{L}$ and $\delta \in \drulel$,
\begin{equation*}
    \E_{(x, z) \sim p(x, z)} \E_{\widehat{y} \sim f(x, z)} \Big[ L\big(\delta(f(x, z)), \widehat{y} \big) \Big] = \E_{(x, z) \sim p(x, z)} \E_{y \sim p(y \mid x)} \Big[ L \big(\delta(f(x, z)), y\big)\Big].
    \label{eq:decision_calibration_with_lms}
\end{equation*}
\end{definition}
\citet{zhao2021calibrating} proves that if a forecaster $f$ is decision calibrated, it provides the downstream decision-maker with two key guarantees.
First, a decision-maker with loss $L$ is assured that the Bayes decision rule $\delta_L$ outperforms alternative decision rules.
Second, the decision-maker is able to accurately estimate the loss that they will incur by using the Bayes decision rule $\delta_L$, despite not having access to ground-truth outcomes $y \sim p(y \mid x')$.
In particular, we will show that decision calibration provides the following guarantees in our setting.
\begin{lemma}[instantiation of Proposition 1 in \citet{zhao2021calibrating}]
If a reader $\f: \mathcal{X} \times \mathcal{Z} \rightarrow \DeltaAns$ is $\mathcal{L}$-decision calibrated, then it satisfies:
\begin{enumerate}
    \item \textit{No regret}: for each $\delta' \in \drulel$,
    \begin{equation*}
        \E_{(x, z) \sim p(x, z)} \E_{y \sim p(y \mid x)} \Big[ L\big(\delta_L(f(x, z)), y \big) \Big] \leq \E_{(x, z) \sim p(x, z)} \E_{y \sim p(y \mid x)} \Big[ L\big(\delta'(f(x, z)), y \big) \Big].
    \end{equation*}
    \item \textit{Accurate loss estimation}:
    \begin{equation*}
        \E_{(x, z) \sim p(x, z)} \E_{\widehat{y} \sim f(x, z)} \Big[ L\big(\delta_L(f(x, z)), \widehat{y} \big) \Big] = \E_{(x, z) \sim p(x, z)} \E_{y \sim p(y \mid x)} \Big[ L \big(\delta_L(f(x, z)), y \big)\Big].
    \end{equation*}
\end{enumerate}
\label{prop:no_regret_and_loss_est}
\end{lemma}
To summarize, \textbf{no regret} means that the Bayes decision rule $\delta_L$ performs no worse than any other decision rule $\delta' \in \drulel$.
\textbf{Accurate loss estimation} means that the decision-maker can determine before observing any ground-truth outcomes what their expected loss will be, allowing them to prepare for the future.

\subsection{Linguistic Calibration and Optimal Decision-making}

Now we can show that linguistic calibration implies the decision-making guarantees of \Cref{prop:no_regret_and_loss_est}.
We do so using a key result from \citet{zhao2021calibrating} (Theorem 1): standard notions of classifier calibration are equivalent to $\mathcal{L}$-decision calibration for particular choices of the loss family $\mathcal{L}$.
We highlight two examples of this equivalence here and direct the interested reader to \citet{zhao2021calibrating} for others.

\paragraph{Example 1: linguistic distribution calibration implies $\Lall$-decision calibration.}
Recall distribution calibration (cf. \S\ref{subsec:linguistic_calibration} for definition), the strongest standard notion of classifier calibration \cite{brocker2009reliability}.
\citet{zhao2021calibrating} (Theorem 1) proves that distribution calibration is equivalent to $\Lall$-decision calibration.
Suppose that we optimize the linguistic calibration objective $\Sc_\LC$ in Eq.~\ref{eq:lc_rl_objective_app} and obtain linguistic distribution calibration (for example, by maximizing $\Sc_\LC$).
By definition, linguistic distribution calibration means that the reader $f$ is distribution calibrated over $p(x, y, z)$.
Then, applying \citet{zhao2021calibrating} (Theorem 1), the reader is $\Lall$-decision calibrated, and therefore a Bayes-optimal decision-maker is guaranteed to perform no worse than \textit{any} decision rule $\delta' \in \druleall$.

Unsurprisingly, this is very difficult to achieve in practice.
However, even if linguistic calibration only manages to calibrate the reader in a weaker sense, we still have $\mathcal{L}$-decision calibration with respect to a smaller family of losses $\mathcal{L}$.

\paragraph{Example 2: linguistic confidence calibration implies $\Lr$-decision calibration.}
To illustrate this decision-making guarantee for weaker notions of calibration, recall confidence calibration (cf. \S\ref{subsec:linguistic_calibration}).
Confidence calibration is a relaxation of the distribution calibration condition, only requiring it to hold for the most likely label.
\citet{zhao2021calibrating} (Theorem 1) demonstrates that confidence calibration is equivalent to $\Lr$-decision calibration, where
\begin{equation}
    \Lr := \Big\{ L(a , y) = \mathbbm{1}[y \neq a \cap a \neq \perp] + \beta \cdot \mathbbm{1}[a = \perp] \mathrel{\big|} a \in \mathcal{Y} \cup \{\perp\}, \forall \beta \in [0, 1] \Big\}.
    \label{eq:lr_losses}
\end{equation}
$\Lr$ is the class of loss functions for the refrained prediction task, where the decision-maker can either take an action corresponding to an answer $y \in \mathcal{Y}$ or report ``I don't know'' with the action $\perp$.
Then, the decision-maker incurs a loss of 0 for correctly predicting the ground-truth answer $y$, a loss of 1 for predicting an incorrect answer, or a loss of $\beta < 1$ for reporting $\perp$.
Similarly, \citet{zhao2021calibrating} (Theorem 1) proves that classwise calibration (cf. \S\ref{subsec:linguistic_calibration}) corresponds to another notion of $\mathcal{L}$-decision calibration.

\paragraph{Linguistic calibration implies no regret and accurate loss estimation guarantees.}
We conclude this section by applying the correspondence between classifier calibration and $\mathcal{L}$-decision calibration to prove decision-making guarantees for linguistic calibration.
For each notion of classifier calibration $\phi \in \{\text{distribution, classwise, confidence}\}$, identify the corresponding class of decision calibration loss functions as $\mathcal{L}_\phi$ (\textit{e.g.,} distribution induces $\Lall$, and confidence induces $\Lr$).
Then we have the following result:
\begin{theorem}[Linguistic $\phi$-calibration implies no regret and accurate loss estimation guarantees]
Suppose that
\begin{itemize}
    \item An LM $\pi_\RL$ is linguistically $\phi$-calibrated with respect to a user with reader function $f: \mathcal{X} \times \mathcal{Z} \rightarrow \DeltaAns$, where $\phi \in \{\text{distribution, classwise, confidence}\}$ (cf. Definition~\ref{defn:ling_cal}).
    \item The user's loss function $L: \mathcal{A} \times \mathcal{Y} \rightarrow \R$ is in the set of losses $\mathcal{L}_\phi$ induced by $\phi$.
    \item The conditional independence assumption~\ref{assumption:gt} holds.
\end{itemize}
Then $f$ is $\mathcal{L}_\phi$-decision calibrated, and moreover, Lemma~\ref{prop:no_regret_and_loss_est} (no regret and accurate loss estimation guarantees) holds for the user's decision-making.
\label{prop:conf_calib_implies_lr_decision_calibration}
\end{theorem}
\begin{proof}
Linguistic $\phi$-calibration is equivalent to $\phi$-calibration of the reader $f$ (cf. \Cref{defn:ling_cal}).
Recall that $p(x, z) = \sum_{q', y'} p(q', x, y', z)$ is the marginal over questions and related LM responses, and let $x' := (x, z)$ and $p(x') := p(x, z)$.
Applying Theorem 1 from \citet{zhao2021calibrating} (equivalence of $\phi$-calibration and $\mathcal{L}_\phi$-decision calibration), we observe that the reader $f$ is $\mathcal{L}_\phi$-decision calibrated according to Definition~\ref{defn:decision_calibration}, Eq.~\ref{eq:decision_calibration}, and therefore the assumption of Lemma~\ref{prop:no_regret_and_loss_est} holds.

We conclude by proving Lemma~\ref{prop:no_regret_and_loss_est}.
Because the reader $f$ is $\mathcal{L}_\phi$-decision calibrated, we may apply Proposition 1 from \citet{zhao2021calibrating}, which provides the no regret and accurate loss estimation guarantees of decision calibration.
Lastly, we invoke the conditional independence statement $p(Y \mid X, Z) = p(Y \mid X)$ (Assumption~\ref{assumption:gt}) to simplify these guarantees, obtaining the guarantees in Lemma~\ref{prop:no_regret_and_loss_est}.

\end{proof}

As an example of this result, consider $\phi = $ confidence calibration. 
Suppose that our optimization of $\Sc_\LC$ leads to linguistic confidence calibration, \textit{i.e.,} a confidence calibrated reader $f$.
Because of the equivalence of confidence calibration and $\Lr$-decision calibration, $f$ is $\Lr$-decision calibrated and we obtain the corresponding guarantees for decision-making.
For example, the no regret guarantee states that the Bayes decision rule $\delta_L$ outperforms decision rules $\delta' \in \drule_{\Lr}$, where $\drule_{\Lr}$ is a set of Bayes decision rules induced by $\Lr$.
Intuitively, because confidence calibration is weaker than distribution calibration, $\drule_{\Lr}$ is a more restricted set than $\druleall$.

\paragraph{Summary: linguistic calibration provides optimal decision-making guarantees.}
If we maximize the linguistic calibration objective, our forecasts are distribution calibrated and therefore we obtain $\Lall$-decision calibration, \textit{i.e.,} the Bayes optimal decision rule outperforms all other decision rules.
Even if our optimization only achieves a weaker notion of calibration such as confidence calibration, we still obtain decision calibration over a more restricted class of loss functions.
This fact also motivates our use of reader ECE as an evaluation metric, because it estimates confidence calibration and can therefore be thought of as a proxy for the quality of downstream user decision-making.
For further discussion, we direct the reader to \citet{zhao2021calibrating} and \citet{zhao2021right}.

\newpage
\section{Training Framework}
\label{app:training_framework}

Linguistic calibration is a general-purpose definition (\Cref{defn:ling_cal}) that is agnostic to the particular design decisions we made in our proof-of-concept training framework.
Nevertheless, this section details the design choices we made, including the regularized objective function we used to train with our neural net--based surrogate reader, and our implementation of PPO \cite{schulman2017proximal}.

\subsection{Regularized Linguistic Calibration Objective}
\label{subsec:app_reward_is_strictly_proper}

\paragraph{Objective regularized to encourage normalized surrogate forecasts.} In our instantiation of linguistic calibration, we use a surrogate forecast $\surrforecast$ which is not guaranteed to be normalized: we may have $\surrforecast \in [0, 1]^{|\mathcal{Y}|} \setminus \DeltaAns$.
We remedy this by adding a regularization term in the reward function which encourages $\surrforecast$ to be normalized, restoring strict propriety (cf.~\S\ref{subsec:linguistic_calibration_training_objective}) even with forecasts in $[0, 1]^{|\mathcal{Y}|}$.
We indeed find in practice that training with this objective results in normalized forecasts.

Our regularized reward function $\Sc_\LC: [0, 1]^{|\mathcal{Y}|} \times \mathcal{Y} \rightarrow \R$ is written as:
\begin{equation}
    \Sc_{\LC}(\surrforecast, y) = \log \surrforecast_y - \lambda \left|1 - \sum_{y' \in \mathcal{Y}} \surrforecast_{y'} \right| + C,
    \label{eq:full_obj_app}
\end{equation}

where $\lambda > 0, C > 0$ are hyperparameters.
$C$ is a small positive constant which we find to reduce reward hacking during PPO.

Below, we prove that with $\lambda > 1$, we indeed have a strictly proper objective.

\subsection{Proof: Regularized Objective is Strictly Proper}

\paragraph{Setup.} In decision-based RL, we wish to optimize an objective that encourages the surrogate forecast $\surrforecast$ to exactly match the ground-truth distribution $p(y \mid x)$, which is supported over all possible answers $\mathcal{Y}$.
$\surrforecast$ is not necessarily in $\DeltaAns$ because it is obtained by several independent neural network predictions falling in the range $[0, 1]$, \textit{i.e.,} $\surrforecast \in [0, 1]^{|\mathcal{Y}|}$.
Therefore, in addition to a log loss term, we also add a regularizer which encourages $\surrforecast$ to fall in the simplex, \textit{i.e.,} be normalized.

We would like our reward function to be strictly proper---its unique maximizer should be the ground-truth answer distribution.
We will prove that our reward function is indeed strictly proper below.
First, we define (strictly) proper scoring rules in our setting, which is \textit{slightly more general than the standard definition} (\textit{e.g.,} in \citet{gneiting2007strictly}) in that the first argument need not be a probability distribution: it may fall in $[0, 1]^{|\mathcal{Y}|} \setminus \DeltaAns$.
As shorthand throughout this proof, we will refer to the ground-truth answer distribution as $p^* := p(y \mid x)$ and the forecast as $\widetilde{p} := \surrforecast$, because this proof is done pointwise for a given (question, LM generation) tuple $(x, z)$.
\begin{definition}[based on \citet{gneiting2007strictly}]
    Let $\Bar{\R}$ be the extended real line $[-\infty, \infty]$. A scoring rule $\Sc: [0, 1]^{|\mathcal{Y}|} \times \mathcal{Y} \rightarrow \Bar{\R}$ is \textbf{proper} if
    \begin{equation}
        \E_{Y \sim p^*}[\Sc(p^*, Y)] \geq \E_{Y \sim p^*}[\Sc(\widetilde{p}, Y)],\quad \forall p^* \in \DeltaAns, \widetilde{p} \in [0, 1]^{|\mathcal{Y}|}.
        \label{eq:proper}
    \end{equation}
    It is \textbf{strictly proper} if Eq.~\ref{eq:proper} holds with equality if and only if $p^* = \widetilde{p}$.
\end{definition}

Equivalently, we may express the strictly proper condition as
\begin{equation}
    \E_{Y \sim p^*} [\Sc(p^*, Y)] > \E_{Y \sim p^*}[\Sc(\widetilde{p}, Y)],\quad \forall p^* \in \DeltaAns, \widetilde{p} \in [0, 1]^{|\mathcal{Y}|}, \widetilde{p} \neq p^*.
    \label{eq:strictly_proper}
\end{equation}

In proving that $R_{\LC}$ is strictly proper, we will also make use of the following standard result:
\begin{lemma}[\textit{e.g.,} \citet{gneiting2007strictly}]
    The logarithmic scoring rule $\Sc_{\log}: \DeltaAns \times \mathcal{Y} \rightarrow \{-\infty\} \cup \R$ is strictly proper. 
    \label{eq:log_loss_strictly_proper}
\end{lemma}

Lastly, we will also use the log-sum inequality:
\begin{lemma}[\textit{e.g.,} p. 29, \citet{cover1991elements}]
    \label{lemma:log_sum}
    Let $a_1, \dots, a_n$ and $b_1, \dots, b_n$ be nonnegative numbers. Denote $a := \sum_{i = 1}^n a_i$ and $b := \sum_{i = 1}^n b_i$. Then
    \begin{equation*}
        \sum_{i = 1}^n a_i \log \frac{a_i}{b_i} \geq a \log \frac{a}{b}.
    \end{equation*}
\end{lemma}

We will now prove that our regularized training framework objective $\Sc_\LC$ is a strictly proper scoring rule using the condition for strict propriety in Eq.~\ref{eq:strictly_proper}.
\begin{theorem}
For any $C \in \R$ and with $\lambda > 1$, the reward function $\Sc_{\LC}(\widetilde{p}, y)$ from Eq.~\ref{eq:full_obj_app} is strictly proper.
\end{theorem}

\begin{proof}
Take an arbitrary $p^* \in \DeltaAns$ and $\widetilde{p} \in [0, 1]^{|\mathcal{Y}|}$ such that $\widetilde{p} \neq p^*$.
We need to show that
\begin{equation}
    \E_{Y \sim p^*} [\Sc_{\LC}(p^*, Y)] - \E_{Y \sim p^*}[\Sc_{\LC}(\widetilde{p}, Y)] > 0.
    \label{eq:pf1}
\end{equation}

All expectations for the rest of the proof are taken with respect to $p^*$, so we omit the subscript $Y \sim p^*$.
As in the main text, we identify the space of answers $\mathcal{Y}$ with the index set $[|\mathcal{Y}|]$, and therefore use notation $p_y$ to refer to the $y$-th index of the vector $p$.
Lastly, we write $\sum_{y' \in \mathcal{Y}} \widetilde{p}_{y'}$ as $\1^\top \widetilde{p}$ for brevity.

Plugging in with our reward function, Inequality~\ref{eq:pf1} is equivalent to
\begin{equation*}
    \E \left[\log p^*_Y - \lambda |1 - \1^\top p^*| + C \right] - \E \left[\log \widetilde{p}_Y - \lambda |1 - \1^\top \widetilde{p}| + C \right] > 0.
\end{equation*}
We can simplify the LHS by applying linearity of expectation and simplifying expectations of constants:
\begin{equation*}
    \text{LHS} = \E \left[\log p^*_Y \right] - \lambda |1 - \1^\top p^*| - \E \left[\log \widetilde{p}_Y\right] + \lambda |1 - \1^\top \widetilde{p}|.
\end{equation*}
Note that we assume $p^* \in \DeltaAns$ so $\1^\top p^* = 1$, and we can simplify further to
\begin{equation}
    \text{LHS} = \E \left[\log p^*_Y \right] - \E \left[\log \widetilde{p}_Y\right] + \lambda |1 - \1^\top \widetilde{p}|.
    \label{eq:pf2}
\end{equation}

Now we proceed by cases on whether or not $\widetilde{p} \in \DeltaAns$.
Suppose that indeed $\widetilde{p} \in \DeltaAns$.
Then its normalization term $\lambda |1 - \1^\top \widetilde{p}|$ is also 0, and therefore this inequality condition is precisely the strict propriety condition for the logarithmic scoring rule (including the assumption that both $p^*$ and $\widetilde{p}$ are in the simplex):
\begin{equation*}
    \E \left[\log p^*_Y \right] - \E \left[\log \widetilde{p}_Y\right] > 0,
\end{equation*}
which is true according to Lemma~\ref{eq:log_loss_strictly_proper}, and therefore this case is done.

So now assume that $\widetilde{p} \in [0, 1]^{|\mathcal{Y}|} \setminus \DeltaAns$.
This implies that its normalization term is greater than 0, \textit{i.e.,} $|1 - \1^\top \widetilde{p}| > 0$.
In this case, we can prove using the properties of $\log$ and the log-sum inequality that the LHS expression in Eq.~\ref{eq:pf2} is greater than 0, in order to complete the proof.

First, we use the properties of $\log$ to rewrite the expectations in Eq.~\ref{eq:pf2} into the entropy of $p^*$ plus a relative entropy between $p^*$ and $\widetilde{p}$:
\begin{align*}
    \text{LHS} &= \E \left[\log p^*_Y \right] - \E \left[\log \widetilde{p}_Y\right] + \lambda |1 - \1^\top \widetilde{p}| \\
    &= -\mathbb{H}[p^*] - \E \left[\log \widetilde{p}_Y\right] + \lambda |1 - \1^\top \widetilde{p}| && (\text{definition of entropy functional $\mathbb{H}[\cdot]$}) \\
    &= -\mathbb{H}[p^*] + \mathbb{H}[p^*] + \sum_{y \in \mathcal{Y}} p^*_y \log \frac{p^*_y}{\widetilde{p}_y} + \lambda |1 - \1^\top \widetilde{p}| && (\text{decomposition of cross-entropy)} \\
    &= \sum_{y \in \mathcal{Y}} p^*_y \log \frac{p^*_y}{\widetilde{p}_y} + \lambda |1 - \1^\top \widetilde{p}|,
\end{align*}
where the penultimate step used the decomposition of cross-entropy $=$ entropy $+$ relative entropy.

We cannot immediately apply the information inequality to show that the relative entropy is nonnegative, because $\widetilde{p}$ is not a probability distribution (and therefore the KL Divergence and relative entropy are not the same; see, \textit{e.g.,} p. 90 in \citet{boyd2004convex}).
However, we may use the log-sum inequality to finish the proof. %
Applying Lemma~\ref{lemma:log_sum} using the fact that all elements of $p^*$ and $\widetilde{p}$ are nonnegative, we can lower bound the relative entropy term as
\begin{equation*}
    \sum_{y \in \mathcal{Y}} p^*_y \log \frac{p^*_y}{\widetilde{p}_y} \geq \1^\top p^* \log \frac{\1^\top p^*}{\1^\top \widetilde{p}} = \log \frac{1}{\1^\top \widetilde{p}}.
\end{equation*}
Altogether, substituting normalizer $\mathcal{Z} := \1^\top \widetilde{p}$, we currently have a lower bound of the LHS in Eq.~\ref{eq:pf2} of
\begin{equation}
    \text{LHS} \geq \log \frac{1}{\mathcal{Z}} + \lambda |1 - \mathcal{Z}|.
    \label{eq:lower_bound}
\end{equation}
Inspecting this function for $\lambda > 1$, we can see that it is positive for all values of $\mathcal{Z}$ other than $\mathcal{Z} = 1$.
To be rigorous, we complete the proof with cases on the values of $\mathcal{Z}$. 
Because $\widetilde{p} \in [0, 1]^{|\mathcal{Y}|}$, $\mathcal{Z}$ is nonnegative, and recall that we already covered the case when $\mathcal{Z} = \1^\top \widetilde{p} = 1$ by reduction to strict propriety of the logarithmic scoring rule (Lemma~\ref{eq:log_loss_strictly_proper}).

When $\mathcal{Z} \in [0, 1)$, we immediately see that $\log \frac{1}{\mathcal{Z}} > 0$ and $\lambda |1 - \mathcal{Z}| > 0$ (since $\lambda > 1$ by assumption), implying that the lower bound in Eq.~\ref{eq:lower_bound} is positive as we needed to show.

Lastly, suppose that $\mathcal{Z} \in [1, \infty)$.
In that case, $\lambda |1 - \mathcal{Z}| = \lambda (\mathcal{Z} - 1)$.
Now, note that the derivative of the lower bound in Eq.~\ref{eq:lower_bound} is positive for all $\mathcal{Z} \geq 1$ (recalling that $\lambda > 1$):
\begin{equation*}
    \frac{d}{d\mathcal{Z}} \left[ \log \frac{1}{\mathcal{Z}} + \lambda (\mathcal{Z} - 1) \right] = -\frac{1}{\mathcal{Z}} + \lambda > -\frac{1}{\mathcal{Z}} + 1,
\end{equation*}
\textit{e.g.,} at $\mathcal{Z} = 1$, we have that the derivative is greater than $-\frac{1}{1} + 1 = 0$.
And further, at precisely $\mathcal{Z} = 1$, the lower bound equals $\log \frac{1}{1} + \lambda |1 - 1| = 0$.
This proves that the lower bound is positive for all $\mathcal{Z} \in (1, \infty)$, which implies that the LHS is also positive for $\mathcal{Z} \in (1, \infty)$.
We have covered all values of $\mathcal{Z}$, concluding that when $\lambda > 1$, the reward function in Eq.~\ref{eq:full_obj_app} is strictly proper.
\end{proof}

\subsection{Additional Details on Training Framework}
\label{app:subsec_training_framework_details}

We provide additional details on the training framework below.
For other details, we direct the reader to the codebase.\footnote{\href{https://github.com/tatsu-lab/linguistic_calibration}{https://github.com/tatsu-lab/linguistic\_calibration}}

\paragraph{Training and Validation Splits.}
Our training framework and all baselines use examples from the TriviaQA \cite{JoshiTriviaQA2017} \texttt{unfiltered.nocontext} subset on HuggingFace Datasets \cite{lhoest-etal-2021-datasets}.
These examples are randomly assigned to the following splits:
\begin{itemize}
    \item \textbf{SFT} ($10000$ examples): used for summary distillation and the other SFT baselines (Factuality SFT, Claude Distill).
    \item \textbf{Prompt Validation} ($1000$ examples): used for all ICL-based baselines and to construct ICL examples for the simulated reader, which uses an API-based LLM.
    \item \textbf{Reward Model} ($20000$ examples): used to train surrogate reader for \LC, and binary reward model for Factuality RL baseline.
    \item \textbf{PPO} ($40000$ examples): used for PPO with \LC RL and Factuality RL methods.
    \item \textbf{PPO Validation} ($1000$ examples): during PPO, we evaluate reward model rewards on this split and store checkpoints every 20 steps. 
    \item \textbf{Validation} ($1000$ examples): used for tuning evaluation temperature and model selection for RL methods (described below).
\end{itemize}

\paragraph{Summary distillation details.}
For each SFT dataset query $q$, we use the (non-confidence) ICL baseline to sample eight long-form generations $\{z_i\}_{i = 1}^8 \sim \pi_{\text{ICL}}(z \mid q)$ (temperature 0.7) that are then summarized using a Claude few-shot prompt $s \sim \text{Summarize}(s \mid z_1, \dots, z_8)$ (temperature 0.3).
We use the dataset of $(q, s)$ pairs to finetune the base Llama 2 7B model, obtaining $\pi_\text{SFT}$.
The same temperature settings are used for obtaining Summary ICL context examples, and for the Direct Summary Eval oracle baseline.

\paragraph{Surrogate and reward model details.}
Following our evaluation protocol for user forecasts (\S\ref{subsec:exp_setup}), we specify prompt instructions so that LLM simulated forecasts $\f_{\text{LLM}}(x, z)$ used in Algorithm~\ref{alg:LC} avoid the use of background knowledge.
Dataset construction for the surrogate reader models ForecastProbs and ExtractAnswers is described in Algorithm~\ref{alg:LC}.

For \LC RL, we use the following training details:
\begin{itemize}
    \item \textbf{ForecastProbs}: a Llama 2 7B--based model initialized with the \LC SFT parameters and trained using a binary cross-entropy loss on prompts from the Reward Model split.
    \item \textbf{ExtractAnswers}: a RedPajama 3B--based model initialized from the pretrained model weights and trained with causal language modeling on targets formatted as a delimited list of possible answers, \textit{e.g.,} ``\texttt{Sushi; Sashimi; Nigiri}''. We used a 3B model instead of Llama 2 7B due to computational constraints, \textit{i.e.,} to fit all standard PPO models such as the policy, reference policy, value model, and reward models into GPU memory on an 8x80GB A100 node.
\end{itemize}
For Factuality RL, we use the following training details:
\begin{itemize}
    \item \textbf{Reward Model}: a Llama 2 7B--based model initialized with the Factuality SFT parameters and trained using binary cross-entropy loss on prompts from the Reward Model split, with binary correctness labels from Claude 2.0. In producing binary correctness labels for a given (question, ground-truth answer, generation) tuple $(x, y, z)$, Claude is given access to the ground-truth answer. Therefore, determining binary correctness is a simple extractive task that we find Claude performs with very high accuracy.
\end{itemize}
We tune all reward models over a wide range of learning rates; see codebase for details on hyperparameters.

\paragraph{PPO implementation and training.}
We use a standard implementation of PPO from \citet{dubois2023alpacafarm} and train with 8 80GB A100 GPUs using Flash Attention 2 \cite{dao2022flashattention, dao2023flashattention2} and PyTorch FSDP \cite{zhao2023pytorch}.
For both PPO-based methods (\LC RL and Factuality RL), we use default hyperparameters other than the following:
\begin{itemize}
    \item We use a larger step batch size (512) with one optimization epoch per step for better training stability.
    \item We shorten \texttt{query\_len} to 128 tokens, because our PPO inputs were essentially all under this length.
    \item We use a slightly lower temperature during the rollout phase ($0.7$ instead of $1.0$). 
    \item We use the \texttt{paged\_adamw\_8bit} optimizer \cite{dettmers20228bit} due to computational constraints. We also use this optimizer for SFT and reward model training for consistency, and find it had no negative impact on validation performance compared to AdamW \cite{loshchilov2018decoupled}.
    \item We train for $1500$ PPO steps.
\end{itemize}
We tune the KL penalty and learning rate of both PPO methods across a wide range, ultimately finding that \texttt{kl\_coef} of $0.1$ and learning rate of 1e-5 works best for both methods.

For the \LC RL objective, we find that $\lambda = 5, C = 5$ works well to enforce normalization of downstream forecasts and prevent reward hacking.
In the log-loss term, we clip the probability of the ground-truth answer at $\epsilon = $1e-4 for numerical stability.

\paragraph{Model selection for RL methods.}
We perform model selection for PPO-based methods (\LC RL, Factuality RL) by taking the 5 PPO checkpoints with the highest reward model rewards on the PPO Validation split, evaluating them on the Validation split with the API-based LLM simulated reader, and selecting the checkpoint with the best ECE.

\clearpage

\section{Evaluation Framework}
\label{app:sec_eval_details}

\paragraph{Test Dataset Splits.}
We use the following test dataset splits.
\begin{itemize}
    \item \textbf{TriviaQA} \cite{JoshiTriviaQA2017} (11313 examples): take the validation split from the \texttt{unfiltered.nocontext} subset from HuggingFace Datasets \cite{wolf-etal-2020-transformers, lhoest-etal-2021-datasets}.
    \item \textbf{Jeopardy} \cite{kaggle2020jeopardy} (10638 examples): take a random sample of the \texttt{train} subset from HuggingFace Datasets. These questions are initially written in the indirect Jeopardy-style format (``For the last 8 years of his life, Galileo was under house arrest for espousing this man's theory''), so we use Claude 2.0 to convert them into direct questions $x$.
    \item \textbf{SciQ} \cite{welbl2017crowdsourcing} (13679 QA examples): combine train, validation, and test splits from \texttt{allenai/sciq} on HuggingFace Datasets. SciQ is a dataset of crowdsourced science exam questions in biology, chemistry, earth science, and physics, spanning elementary level to college introductory material.
    \item \textbf{BioASQ} \cite{bioasq} (1515 QA examples): we use the 2024 BioASQ12 Task B dataset, which is a biomedical QA dataset annotated by a team of experts. Specifically, we use their 1515 “factoid” examples, which have a short-answer gold label (e.g., “proprotein convertase subtilisin/kex in type 9”). The dataset is available at \href{http://participants-area.bioasq.org/datasets}{http://participants-area.bioasq.org/datasets}.
    \item \textbf{Person Biography Generation} \cite{min-etal-2023-factscore} (500 entities): use the \texttt{unlabeled} split of Wikipedia entities from the FactScore dataset.
\end{itemize}

TriviaQA, Jeopardy, SciQ, and BioASQ are question-answering datasets which we repurpose for long-form evaluation.
Specifically, for these question-answering datasets, we use an API-based LLM (Claude 2.0) to convert questions $x$ into open-ended queries $q$ that allow us to sample long-form generations $z \sim \pi(z \mid q)$ on the topic of the question $x$ (as we do at training time, cf. \S\ref{subsec:decision_task_data}).
For Person Biography Generation, we construct open-ended queries using the fixed template ``Write a paragraph bio about \{entity\}.''.

\paragraph{Confidence intervals.}
We report bootstrap 0.95 confidence intervals in all frontier plots and tabular results, except for the BioASQ dataset. 
BioASQ is a relatively small dataset (1515 examples; it is expert-written) and therefore we report Student's $t$ distribution 0.95 confidence intervals over 10 evaluation seeds.
In reliability diagrams for BioASQ, we report the reliability diagram computed over samples pooled across all seeds in blue, and per-seed reliability diagrams in gray.

\subsection{Simulated Evaluation}

\paragraph{Simulated reader.}
Simulated forecasts $f(x, z)$ are a distribution over possible answers to the question $x$, based on long-form LM generations $z$.
All simulated forecasts in this paper are formed using an API-based LLM simulated reader $f_\LLM$.
We use Claude due to compute constraints, but other LLMs such as GPT-4 are well-capable of producing high-quality forecasts.
These simulated forecasts are used in two parts of our training and evaluation frameworks:
\begin{itemize}
    \item \textbf{Training framework:} simulated forecasts are used as ground-truth targets for training the surrogate reader $\widetilde{f}$ (cf. Algorithm~\ref{alg:LC}). Then the surrogate reader is used during PPO, enabling efficient training by avoiding calls to API-based LLMs.
    \item \textbf{Evaluation framework:} in our question-answering evaluations, we sample a held-out question-answer pair $(x, y) \sim p(x, y)$, convert $x$ into open-ended query $q$, and evaluate the accuracy and ECE of the simulated forecast $f_\LLM(x, z)$ for $z \sim \pi(z \mid q)$ sampled from a variety of LMs such as \LC RL and Factuality RL. 
\end{itemize}
In both cases, we obtain simulated forecasts by decomposing reading into two prompts, one for ExtractAnswers and one for ForecastProbs, following the description in \S\ref{sec:method}.
This is just one way of implementing a reader, and there are likely other reasonable ways to obtain high-quality probabilistic predictions conditioned on long-form text.

In order to compute accuracy and ECE for the evaluation framework, we need to obtain the confidence and correctness of each forecast.
We direct the reader to our codebase\footnote{\href{https://github.com/tatsu-lab/linguistic_calibration}{https://github.com/tatsu-lab/linguistic\_calibration}} for all prompts.

\paragraph{Accuracy and ECE for confidence-based methods.}
The confidence of a forecast is obtained by simply taking its argmax probability $\max_{y \in \mathcal{Y}} f(x, z)_y$.
Recall that for question-answering evaluation datasets, we prompt an API-based LLM to perform the ForecastProbs operation.
We find a zero-shot prompt is effective because confidence statements in $z$ associated with the question $x$ are usually numerical (and even when they are occasionally linguistic, Claude 2.0 produces reasonable confidence estimations).
The correctness is obtained by determining whether the top answer in the forecast is correct, \textit{i.e.,} $\mathbbm{1}[y = \argmax_{y' \in \mathcal{Y}} f(x, z)_{y'}]$.
In practice, we perform this check by prompting Claude with the question $x$, ground-truth answer choice $y$, and the argmax answer choice, and asking the LLM to decide whether the two answers are semantically equivalent, using a few-shot prompt.
For TriviaQA and Jeopardy, which are general-domain knowledge tasks, \texttt{claude-2.0} can perform this semantic equivalence check reliably.
For the domain-specific tasks SciQ and BioASQ, we use the more powerful \texttt{claude-3-opus-20240229} to perform this semantic equivalence check, because the ground-truth and argmax answers are more complicated; \textit{e.g.,} $y = $ ``hormone receptor-positive, human epidermal growth factor receptor 2-negative'', argmax = ``hormone-receptor-positive, HER2-negative breast cancer patients''.
This semantic equivalence check avoids the false negative issues that string match--based metrics encounter due to paraphrases, which are very common in our open-ended, long-form generation setting.

\paragraph{Accuracy and ECE for non-confidence methods.}
Our evaluation for non-confidence methods closely follows the pipeline for confidence-based methods.
We extract answers using the ExtractAnswers prompt, which accounts for the possibility that models may provide multiple mutually exclusive answers to the question $x$.
We select the ``argmax'' answer by taking the answer which appears first in the paragraph $z$.
Because non-confidence methods are not prompted or trained to produce confidence statements, we fix their confidence values to 1.
Then, we compute ECE for non-confidence methods in closed form as $(1 - \text{Accuracy})$.
This approach is validated by our human evaluations, in which the correlation between a fixed confidence value of 1 and the human-interpreted argmax probability of Factuality RL is 0.993, over 1000+ samples (cf. Table \ref{tab:human_claude_agreement}).

\paragraph{Temperature tuning.} 
We tuned the decoding temperature for generating paragraphs across the range $\{0.1, 0.3, 0.5, 0.7\}$ for the core methods (\LC RL, \LC SFT, Factuality RL, Factuality SFT) on the TriviaQA Validation split with simulated reader evaluation.
We found that temperature $0.3$ uniformly performed the best across methods.
Therefore we defaulted to paragraph decoding temperature $0.3$ for evaluation throughout.
We found that temperature $0.2$ consistently worked well for simulated reader operations ExtractAnswers, ForecastProbs, and for checking semantic equivalence.

\paragraph{Other evaluation details.}
Because we prompt our LMs with open-ended queries ($q :=$ ``Write a paragraph about $\{x\}$'') it is rarely possible that a paragraph $z$ will not provide information that answers question $x$. 
We account for this during evaluation by allowing the ExtractAnswers prompt to return an empty list if $z$ does not provide any plausible answers to the question.
For any example with an empty extracted list, we consider the accuracy to be 0 and the confidence to be 1, which is the most conservative possible penalty for both accuracy and ECE.
Because we significantly outperform baselines finetuned with RL for factuality, we account for the possibility of better responsiveness to the prompt (and therefore fewer empty ExtractAnswers lists) explaining our improvements.

\subsection{FactScore-Based Evaluation Metric}
\label{app:factscore_metric}

We use an automated FactScore-based \cite{min-etal-2023-factscore} metric using an API-based LLM (Claude 2.0 in our case) to test the accuracy and calibration of generated biographies $z$ at the per-claim level, probing for the quality of confidences throughout long-form generations. 
\S\ref{subsec:exp_biography_generation} describes this metric in detail.
Prompts and code for this metric are available in our codebase.

\paragraph{Forecasting confidences.} As part of this metric, for confidence-based methods, we need to forecast probabilities which are used to compute ECE.
As compared to the question-answering evaluations, this evaluation considers all claims in the generated paragraph and tends to have a higher proportion of claims with linguistic confidence statements.
Therefore, we use the following procedure to assign probabilities to each claim.
When a claim has a clear numerical uncertainty (\textit{e.g.,} a percentage) this is a simple extractive task which API-based LLMs perform well by default.
For interpreting linguistic uncertainties, we provide the LLM with a mapping between linguistic phrases and probabilities collected in a linguistics human study \cite{wallsten}, and allow it to generalize from this mapping to phrases that are not present in it (Table~\ref{tab:linguistic_confidences_factscore}).
\begin{table}[h]
    \centering
    \begin{tabular}{ccc}
        \toprule
        Linguistic Confidence Phrase & Probability \\
        \midrule
        Almost Impossible & 0.05 \\
        Doubtful & 0.1 \\
        Improbable & 0.1 \\
        Unlikely & 0.15 \\
        Possible & 0.3 \\
        Tossup & 0.5 \\
        Good Chance & 0.65 \\
        Likely & 0.75 \\
        Probable & 0.75 \\
        Almost Certain & 0.95 \\
        \bottomrule
    \end{tabular}
    \vspace{1em}
    \caption{Ground-truth linguistic confidence mapping used in our FactScore probability interpretation prompt, from \citet{wallsten}.}
    \label{tab:linguistic_confidences_factscore}
\end{table}
We believe that developing better strategies to interpret linguistic confidence statements in a manner tailored to a given user population is exciting future work (discussed in \S\ref{sec:discussion}).

\subsection{Human Evaluation}

\paragraph{Task.}
In order to compute our accuracy and ECE metrics using forecasts from human crowdworkers, we asked crowdworkers to provide the following for a given question, ground-truth answer, and long-form generation tuple $(x, y, z)$: 
\begin{itemize}
    \item \textbf{Confidence}: the max probability of their forecast distribution $\max_{y \in \mathcal{Y}} \forecast_y$.
    \item \textbf{Correctness}: a binary indicator of whether or not their top answer was correct, $\mathbbm{1}[y = \argmax_{y' \in \mathcal{Y}} \forecast_{y'}]$.
\end{itemize}
We asked human crowdworkers to interpret the generated paragraphs literally, and provided a couple examples.
These examples were intentionally selected to be unambiguous, \textit{i.e.,} numerical confidences, to ensure we were not biasing towards any particular manner of interpreting linguistic confidences.

\paragraph{Qualification.} We conducted two qualification rounds with our human annotators for the question-answering TriviaQA experiment.
In the first round, we sent a single question-answer pair $(x, y)$ sampled from TriviaQA, along with 3 model outputs (from \LC RL,  \LC SFT, and Factuality RL) to 117 workers, out of which 32 matched the expected output. This example was selected to be unambiguous.
In the second round, we sent the 32 workers who passed the previous test a roughly hour-long test (examples also selected to be unambiguous; the student authors of this paper agreed on all confidence and correctness labels of the hour-long test).
Among these annotators, we took the top 15 performers for the main annotation process.

\paragraph{Main annotation round.} During the main annotation process, we simply selected examples uniformly at random from the TriviaQA test set, including examples which were potentially ambiguous due to linguistic confidence statements.
Overall, we collected annotations for 1008 examples (each with 3 model outputs) from human crowdworkers. 

\paragraph{Quality assurance.} 
We compared crowdworker agreement with Claude 2.0 as a quality assurance measure. 
Five of the annotators had agreement rates which dropped significantly after the qualification round, from perfect performance on the hour-long test to less than 0.40 Cohen's kappa agreement with Claude.
Therefore, we discontinued working with these annotators and removed their labels.

\paragraph{Payment.} To determine payment, the authors of this paper timed our own completion time of the task, which was an average of approximately 1.3 minutes. 
Following best practices for crowdworker payment \cite{Whiting_Hugh_Bernstein_2019, liang2023holistic}, we used a conservative estimate of annotation time of 3 minutes per task and set the per-task payment to \$0.85 amounting to a compensation of \$17 per hour, greater than the California minimum wage of \$16 per hour.

\paragraph{Agreement statistics.}
\label{app:app_human_agreement_stats}

In Table~\ref{tab:interannotator_agreement} we report interannotator agreement (\textit{i.e.,} agreement between crowdworkers) computed over 5\% of the collected samples, for each of the tested models.

\begin{table}[h]
    \centering
    \begin{tabular}{ccc}
        \toprule
        Model & Correlation on Confidence & Cohen's Kappa on Correctness \\
        \midrule
        \LC RL & 0.886 & 0.850 \\
        \LC SFT & 0.719 & 0.842 \\
        Factuality RL & 1.000 & 0.758 \\
        \bottomrule
    \end{tabular}
    \vspace{1em}
    \caption{Interannotator Agreement.}
    \label{tab:interannotator_agreement}
\end{table}

In Table~\ref{tab:human_claude_agreement}, we report human agreement with Claude 2.0 over all examples.

\begin{table}[h]
    \centering
    \begin{tabular}{ccc}
        \toprule
        Model & Correlation on Confidence & Cohen's Kappa on Correctness \\
        \midrule
        \LC RL & 0.626 & 0.739 \\
        \LC SFT & 0.618 & 0.748 \\
        Factuality RL & 0.993 & 0.741 \\
        \bottomrule
    \end{tabular}
    \vspace{1em}
    \caption{Human-Claude 2.0 Agreement.}
    \label{tab:human_claude_agreement}
\end{table}

%% file: tables/question_answering_results.tex
\begin{table*}[h]
\centering
\caption{
    \textbf{Simulated Question-Answering Results.}
    Accuracy and expected calibration error (ECE) of Llama 2 7B and GPT-4--based methods on the in-distribution (ID) TriviaQA and out-of-distribution (OOD) Jeopardy, SciQ, and BioASQ Task B question-answering datasets, with simulated reader evaluation.
    ``API-Based LLM Methods'' are those that use either GPT-4 or Claude 2 at test time.
    CIs are reported as $(\text{lower}, \text{upper})$. 
    For all datasets other than BioASQ, these are 95\% bootstrap CIs; for the smaller BioASQ dataset, these are Student's $t$ 0.95 CIs over 10 evaluation seeds.
    Due to compute constraints, for GPT-4--based methods on BioASQ, we report accuracy and ECE for a single evaluation seed (cf. Appendix~\ref{app:sec_eval_details} for all evaluation details).
}
\vspace{-3pt}
\resizebox{1.0\linewidth}{!}{%
\begin{tabular}{@{\extracolsep{2pt}}lcccccccc@{}}
\midrule
\midrule
& \multicolumn{2}{c}{TriviaQA (ID)} & \multicolumn{2}{c}{Jeopardy (OOD)} & \multicolumn{2}{c}{SciQ (OOD)} & \multicolumn{2}{c}{BioASQ Task B (OOD)} \\
\cline{2-3}
\cline{4-5}
\cline{6-7}
\cline{8-9}\\
\textbf{Method}         &
\textbf{Accuracy (\%) $\uparrow$}            &
\textbf{ECE $\downarrow$}       &
\textbf{Accuracy (\%) $\uparrow$}            &
\textbf{ECE $\downarrow$}       &
\textbf{Accuracy (\%) $\uparrow$}            &
\textbf{ECE $\downarrow$}   &
\textbf{Accuracy (\%) $\uparrow$}            &
\textbf{ECE $\downarrow$}    \\
\midrule
\multicolumn{9}{c}{Llama 2 7B--Based Methods}\\
\midrule
Llama 2 Chat 7B & 52.32~\scr{(51.42,~53.27)} & 0.477~\scr{(0.467,~0.486)} & 50.86~\scr{(49.91,~51.82)} & 0.491~\scr{(0.482,~0.501)} & 51.52~\scr{(50.68,~52.34)} & 0.485~\scr{(0.477,~0.493)} & 35.31~\scr{(34.79,~35.82)} & 0.647~\scr{(0.642,~0.652)} \\
Claude Distill & 60.89~\scr{(59.98,~61.78)} & 0.391~\scr{(0.382,~0.400)} & 60.54~\scr{(59.64,~61.44)} & 0.395~\scr{(0.386,~0.404)} & \B 57.00~\scr{(56.17,~57.83)} & 0.430~\scr{(0.422,~0.438)} & \B 39.78~\scr{(39.30,~40.27)} & 0.602~\scr{(0.597,~0.607)} \\
ICL 8-Shot & 61.91~\scr{(61.02,~62.81)} & 0.381~\scr{(0.372,~0.390)} & 61.63~\scr{(60.69,~62.57)} & 0.384~\scr{(0.374,~0.393)} & 55.10~\scr{(54.27,~55.92)} & 0.449~\scr{(0.441,~0.457)} & 37.72~\scr{(37.09,~38.35)} & 0.623~\scr{(0.617,~0.629)} \\
Factuality SFT & 61.17~\scr{(60.26,~62.07)} & 0.388~\scr{(0.379,~0.397)} & 62.13~\scr{(61.21,~63.05)} & 0.379~\scr{(0.370,~0.388)} & 55.08~\scr{(54.26,~55.92)} & 0.449~\scr{(0.441,~0.457)} & 37.54~\scr{(37.22,~37.87)} & 0.625~\scr{(0.621,~0.628)} \\
Factuality RL & 63.33~\scr{(62.44,~64.21)} & 0.367~\scr{(0.358,~0.376)} & 64.05~\scr{(63.14,~64.97)} & 0.359~\scr{(0.350,~0.369)} & 56.11~\scr{(55.27,~56.94)} & 0.439~\scr{(0.431,~0.447)} & 38.04~\scr{(37.66,~38.42)} & 0.620~\scr{(0.616,~0.623)} \\
Summary ICL 8-Shot & 61.36~\scr{(60.47,~62.25)} & 0.257~\scr{(0.248,~0.267)} & 60.90~\scr{(59.98,~61.83)} & 0.254~\scr{(0.245,~0.263)} & 54.76~\scr{(53.91,~55.61)} & 0.324~\scr{(0.316,~0.333)} & 38.15~\scr{(37.77,~38.53)} & 0.476~\scr{(0.471,~0.480)} \\
\LC SFT ($\pi_{\mathtt{SFT}})$ & 60.98~\scr{(60.10,~61.88)} & 0.166~\scr{(0.158,~0.176)} & 62.46~\scr{(61.53,~63.36)} & 0.162~\scr{(0.154,~0.172)} & 54.87~\scr{(54.04,~55.69)} & 0.313~\scr{(0.306,~0.323)} & 38.53~\scr{(38.19,~38.87)} & 0.389~\scr{(0.384,~0.394)} \\
\LC RL ($\pi_{\mathtt{LC}}$) & \B 64.74~\scr{(63.86,~65.63)} & \B 0.108~\scr{(0.101,~0.117)} & \B 65.73~\scr{(64.83,~66.61)} & \B 0.088~\scr{(0.082,~0.098)} & 56.85~\scr{(56.02,~57.67)} & \B 0.213~\scr{(0.205,~0.222)} & 38.89~\scr{(38.39,~39.39)} & \B 0.342~\scr{(0.335,~0.350)} \\

\midrule
\multicolumn{9}{c}{API-Based LLM Methods}\\
\midrule
GPT-4 0-Shot & 78.00~\scr{(75.30,~80.50)} & 0.220~\scr{(0.195,~0.247)} & 77.60~\scr{(75.00,~80.10)} & 0.224~\scr{(0.199,~0.250)} & 66.20~\scr{(63.20,~69.10)} & 0.338~\scr{(0.309,~0.368)} & 62.11 & 0.379 \\
GPT-4 ICL 8-Shot & \B 80.80~\scr{(78.40,~83.20)} & 0.192~\scr{(0.168,~0.216)} & \B 82.30~\scr{(80.00,~84.70)} & 0.177~\scr{(0.153,~0.200)} & \B 71.60~\scr{(68.80,~74.50)} & 0.284~\scr{(0.255,~0.312)} & \B 63.43 & 0.366 \\
GPT-4 JAFU 0-Shot & 77.70~\scr{(75.10,~80.20)} & 0.212~\scr{(0.187,~0.238)} & 78.00~\scr{(75.40,~80.50)} & 0.207~\scr{(0.182,~0.234)} & 68.50~\scr{(65.50,~71.40)} & 0.304~\scr{(0.274,~0.334)} & 62.31 & 0.356 \\
GPT-4 Summary ICL 8-Shot & 80.40~\scr{(77.90,~82.80)} & 0.119~\scr{(0.104,~0.148)} & 80.30~\scr{(77.80,~82.70)} & 0.130~\scr{(0.111,~0.156)} & 70.50~\scr{(67.70,~73.30)} & \B 0.225~\scr{(0.197,~0.253)} & 61.39 & \B 0.281 \\
Direct Summary Eval & 63.67~\scr{(62.78,~64.55)} & \B 0.112~\scr{(0.105,~0.120)} & 64.50~\scr{(63.59,~65.43)} & \B 0.104~\scr{(0.098,~0.114)} & 57.89~\scr{(57.04,~58.73)} & 0.295~\scr{(0.286,~0.303)} & 40.42~\scr{(39.91,~40.93)} & 0.374~\scr{(0.367,~0.380)} \\
\midrule
\midrule
\end{tabular}
}
\label{tab:metrics_qa}
\end{table*}

%% file: tables/human_question_answering_results.tex
\begin{table*}[h]
\centering
\caption{
    \textbf{Human Question-Answering Results on TriviaQA.}
    Accuracy and expected calibration error (ECE) of core Llama 2 7B--based methods on the TriviaQA question-answering dataset, with human readers (cf. Appendix~\ref{app:sec_eval_details} for evaluation framework details).
    95\% bootstrap CIs are reported as $(\text{lower}, \text{upper})$.
}
\vspace{-3pt}
\resizebox{0.55\linewidth}{!}{
\begin{tabular}{@{\extracolsep{2pt}}lcc@{}}
\midrule
\midrule
\textbf{Method}         &
\textbf{Accuracy (\%) $\uparrow$}            &
\textbf{ECE $\downarrow$} \\
\midrule
Factuality RL & 59.62~\scr{(56.65,~62.60)} & 0.404~\scr{(0.374,~0.434)} \\ 
LC SFT ($\pi_{\mathtt{SFT}})$ & 57.44~\scr{(54.37,~60.52)} & 0.163~\scr{(0.135,~0.192)} \\ 
LC RL ($\pi_{\mathtt{LC}}$) & \B 60.12~\scr{(57.14,~63.19)} & \B 0.116~\scr{(0.091,~0.145)} \\ 
\midrule
\midrule
\end{tabular}
}
\label{tab:metrics_human_qa}
\end{table*}

%% file: tables/biography_generation_results.tex
\begin{table*}[h]
\centering
\caption{
    \textbf{Person Biography Generation Results.}
    Accuracy and expected calibration error (ECE) of Llama 2 7B and GPT-4--based methods on the Person Biography Generation dataset, using a FactScore-based automated metric \cite{min-etal-2023-factscore}.
    Metrics are computed over atomic claims pooled across all biography generations.
    95\% bootstrap CIs are reported as $(\text{lower}, \text{upper})$.
}
\vspace{-3pt}
\resizebox{0.55\linewidth}{!}{%
\begin{tabular}{@{\extracolsep{2pt}}lcccccc@{}}
\midrule
\midrule
\textbf{Method}         &
\textbf{Accuracy (\%) $\uparrow$}            &
\textbf{ECE $\downarrow$} \\
\midrule
\multicolumn{3}{c}{Llama 2 7B--Based Methods}\\
\midrule
Llama 2 Chat 7B & 36.27~\scr{(35.18,~37.36)} & 0.637~\scr{(0.626,~0.648)} \\ 
Claude Distill & 39.24~\scr{(38.13,~40.34)} & 0.608~\scr{(0.597,~0.619)} \\ 
Factuality SFT & 38.07~\scr{(37.01,~39.13)} & 0.619~\scr{(0.609,~0.630)} \\ 
Factuality RL & 39.86~\scr{(38.72,~40.98)} & 0.601~\scr{(0.590,~0.613)} \\ 
LC SFT ($\pi_{\mathtt{SFT}})$ & 44.49~\scr{(43.42,~45.57)} & 0.301~\scr{(0.292,~0.314)} \\ 
LC RL ($\pi_{\mathtt{LC}}$) & \B 46.77~\scr{(45.50,~48.08)} & \B 0.266~\scr{(0.253,~0.280)} \\ 

\midrule
\multicolumn{3}{c}{GPT-4--Based Methods}\\
\midrule

GPT-4 0-Shot & 71.84~\scr{(70.88,~72.81)} & 0.282~\scr{(0.272,~0.291)} \\ 
GPT-4 JAFU 0-Shot & \B 75.95~\scr{(74.93,~76.93)} & \B 0.210~\scr{(0.200,~0.220)} \\

\midrule
\midrule
\end{tabular}
}
\label{tab:metrics_bio_gen}
\end{table*}